\documentclass[10pt]{article} 
\usepackage{natbib}
\usepackage{tikz}
\usetikzlibrary{positioning}
\usetikzlibrary{arrows.meta}
\usepackage[accepted]{tmlr}

\usepackage{amsmath,amsfonts,bm}

\def\eqref#1{equation~\ref{#1}}

\def\1{\bm{1}}

\DeclareMathAlphabet{\mathsfit}{\encodingdefault}{\sfdefault}{m}{sl}
\SetMathAlphabet{\mathsfit}{bold}{\encodingdefault}{\sfdefault}{bx}{n}

\DeclareMathOperator*{\argmax}{arg\,max}
\DeclareMathOperator*{\argmin}{arg\,min}

\usepackage{array,tabularx}
\newcolumntype{M}[1]{>{\centering\arraybackslash}m{#1}} 
\newcolumntype{Y}{>{\centering\arraybackslash}X}
\usepackage{titletoc}
\usepackage{hyperref}
\usepackage{url}
\usepackage{changepage}
\usepackage{amsthm}
\newtheorem{Definition}{Definition}

\newtheorem{lemma}{Lemma}
\usepackage{algorithm}
\usepackage{placeins}
\usepackage{algpseudocode}
\usepackage{graphicx}
\usepackage{multirow}
\usepackage{booktabs}
\usepackage{amsmath,amssymb}
\usepackage{enumitem}
\usepackage{xcolor}

\newcommand{\fslot}[1]{\text{\footnotesize\itshape (#1)}\ \ }
\newenvironment{formalsummary}[2]{
  \par\medskip\noindent\textbf{Formal summary (#1).}%
  \small
  \begin{subequations}\label{#2}%
}{\end{subequations}\par\medskip}

\definecolor{darkblue}{RGB}{0,0,100}
\hypersetup{
  colorlinks=true,
  linkcolor=darkblue,
  citecolor=darkblue,
  urlcolor=blue
}

\newcommand{\printappendixtoc}{%
\begingroup
\section*{Table of Contents}

\begin{adjustwidth}{0.06\textwidth}{0.008\textwidth}
\hrule

\setcounter{tocdepth}{3}

\titlecontents{section}
  [0pt]
  {\addvspace{0.45em}\bfseries}
  {\contentslabel{2.8em}}
  {}
  {\titlerule*[0.6pc]{.}\contentspage}

\titlecontents{subsection}
  [1.5em]
  {}
  {\contentslabel{3.8em}}
  {}
  {\titlerule*[0.6pc]{.}\contentspage}

\titlecontents{subsubsection}
  [3.5em]
  {}
  {\contentslabel{4.8em}}
  {}
  {\titlerule*[0.6pc]{.}\contentspage}

\printcontents[appendix]{}{1}{\setcounter{tocdepth}{3}}
\end{adjustwidth}

\endgroup
}

\title{Causal Bayesian Optimization:\\ Foundations, Methods, and Applications}

\author{\name Chenfeng Huang \email chenfenghuang@ucla.edu\\
      \addr Department of Statistics \& Data Science \\
      University of California, Los Angeles
      \AND
      \name Thuy T. Le  \email Thuy.Le@csulb.edu  \\
      \addr Department of Mathematics and Statistics \\
      California State University, Long Beach
      \AND
      \name Zixuan Ma   \email zixuanma@ucla.edu \\
      \addr Department of Statistics \& Data Science \\
      University of California, Los Angeles
      \AND
      \name Hien Tran \email tran@math.ncsu.edu
      \\
      \addr Department of Mathematics \\
      North Carolina State University
 }

\def\month{08}  
\def\year{2026} 
\def\openreview{\url{https://openreview.net/forum?id=XT6DC37m5I}} 

\begin{document}

\maketitle

\begin{abstract}
Causal Bayesian Optimization (CBO) integrates causal inference with Bayesian optimization to enable sample-efficient intervention selection in systems governed by causal structure. This survey provides a comprehensive and systematic review of the CBO landscape, organizing the growing literature through a unified BO-loop perspective that reveals how causal assumptions shape four core components: intervention search spaces, surrogate construction, acquisition design, and decision policies. We organize methods along recurring design axes, including graph and system-knowledge assumptions, environmental assumptions, intervention representation, surrogate architecture, and decision rules, and we clarify conceptual and notational connections between CBO and adjacent fields, including causal bandits, Bayesian experimental design, safe optimization, policy search, and causal abstraction. To address the lack of standardized evaluation in the field, we introduce a reproducibility-oriented benchmark that covers hard- and soft-intervention settings, implements both the standard GAP metric and a new trajectory-aware Path-Aware GAP (PA-GAP) metric, and evaluates seven CBO methods alongside a non-causal BO baseline under a common scoring protocol. Within this benchmark-specified regime, where the benchmark SCM, intervention domains, and reference optima are fixed and known-graph methods receive the benchmark graph unless explicitly specified otherwise, our empirical study across thirteen datasets, three budget levels, and two metrics reveals that no single method dominates uniformly: rankings depend critically on the dataset, budget, metric, and method-specific use of causal information, and strong non-causal baselines remain competitive in several settings. We further add controlled graph-misspecification and omitted-variable stress tests, which show that rankings can change substantially when the learner-side causal information is perturbed. We therefore identify robustness to causal-assumption violations, scalable unknown-graph optimization, mixed intervention types, realistic cost models, tighter theoretical guarantees, and integration with modern representation learning and causal abstractions as open challenges for moving CBO from controlled benchmarks toward reliable deployment.
\end{abstract}

\section{Introduction}

Many scientific and engineering problems require the selection of interventions in systems whose variables are causally coupled. A clinician chooses a drug dosage that propagates through metabolic pathways to affect a clinical endpoint; an engineer tunes process parameters that cascade through a manufacturing pipeline; a policymaker adjusts economic levers whose effects depend on the causal structure of the economy. In each case, the key challenge is the same: the system's response to an intervention is not governed by statistical association but by a causal mechanism, and purely correlational optimization can recommend actions that appear promising in observational data, yet fail or even cause harm when deployed \citep{pearl2009causality,peters2017elements}.

Bayesian Optimization (BO) offers a basic framework for sample-efficient optimization when evaluations are expensive, noisy, or constrained \citep{Jones1998,shahriari2016taking}. By iteratively fitting a probabilistic surrogate and selecting queries via an acquisition function that balances exploitation against exploration, BO can identify near-optimal solutions with far fewer evaluations than grid or random search. However, standard BO treats the objective as a black box: it is agnostic to the causal relationships among the variables it manipulates. Causal Bayesian Optimization (CBO) closes this gap by embedding structural causal knowledge into the BO loop, enabling the optimizer to reason about \emph{which} variables to intervene on, \emph{how} observational data relate to interventional effects, and \emph{where} the experimental budget is best spent \citep{aglietti2020CBO}.

Since its introduction, CBO has rapidly expanded into a diverse family of methods, spanning dynamic environments \citep{aglietti2021}, safety constraints \citep{aglietti2023cCBO}, mechanism-level surrogates with regret guarantees \citep{Sussex2022MCBO}, functional and contextual policies \citep{gultchin2023FCBO,ContextualCBO}, adversarial non-stationarity \citep{ACBOSussex}, high-dimensional scaling \citep{Wu2024HighDimensionalCBO}, unknown-graph optimization \citep{branchini2023CEO,mukherjee2024}, noise-robust acquisition with learned priors \citep{Li2023}, and abstraction-based causal decision making \citep{zennaro2024camab,dyer2025atucb,zeitler2025mfacbo}. Section~\ref{Methods} organizes these developments along five recurring design axes---graph and system-knowledge assumptions, environment, intervention representation, surrogate architecture, and decision rule (Figure~\ref{fig:cbo_categories}, Table~\ref{tab:cbo-variants}).

Despite this progress, the field lacks two critical ingredients for maturation. First, a \textbf{unified conceptual framework} that reveals the shared structure across CBO variants and connects them to adjacent fields, such as causal bandits, Bayesian experimental design, active causal discovery, and safe reinforcement learning, is missing. Second, \textbf{standardized evaluation infrastructure} is absent. Existing papers use different datasets, incompatible codebases, inconsistent intervention conventions, and different metrics, making cross-paper comparisons unreliable. A method that excels when the graph is known and interventions are low-dimensional may behave very differently under soft interventions, unknown graphs, or context-dependent policies, yet these distinctions are obscured when each paper evaluates in isolation.

This survey addresses both gaps. We provide a systematic, technically detailed review of the CBO landscape organized through a unified BO-loop perspective, and we introduce a reproducibility-oriented benchmark that standardizes and repackages inherited CBO tasks together with newly added scenarios under a common execution and scoring protocol. In terms of concreteness, our contributions are as follows.
\begin{itemize}[leftmargin=*,itemsep=2pt]
    \item \textbf{A unified design-space perspective.} We organize CBO methods by graph and system-knowledge assumptions, environmental assumptions, intervention representation, surrogate architecture, and decision rule (Section~\ref{Methods}). This taxonomy is intended as a survey-level organization of recurring design patterns and couplings, rather than as an ablation study proving that components are interchangeable.
    \item \textbf{Connections to adjacent fields.} We clarify conceptual and notational links between CBO and causal bandits, Bayesian experimental design, active causal discovery, safe optimization, policy search, and causal abstraction / multi-scale decision making, while distinguishing such connections from formal equivalences or reductions (Section~\ref{subsec:adjacent_fields}).
    \item \textbf{A reproducibility-oriented benchmark.} We release a standardized benchmark harness covering eight hard-intervention and five soft-intervention scenarios, combining inherited tasks from prior CBO/function-network work with newly added or reconstructed scenarios. The benchmark provides standardized GAP, a new trajectory-aware PA-GAP metric, and a common best-so-far trajectory interface that enables fair re-scoring of outputs from heterogeneous implementations (Section~\ref{sec:experiments}).
    \item \textbf{A unified empirical comparison.} We evaluate seven CBO methods and a non-causal BO baseline across thirteen datasets, three budget levels, and two metrics. Because the compared methods do not all solve the same optimization problem, we do not pool them into cross-method rank plots: each method--task pair carries a formulation label distinguishing matched comparisons, stress-test comparisons, and descriptive single-method or single-dataset settings, the main text reports per-dataset comparisons, and budget-stratified descriptive rank-reliability checks are reported in the appendix. The results show that no method is statistically separable as uniformly best, rankings depend strongly on dataset--budget-metric interactions, and strong non-causal baselines remain competitive in several settings (Section~\ref{subsec:hard_benchmark}--\ref{subsec:soft_benchmark}).
    \item \textbf{Structured open problems.} We identify six concrete research directions (robustness to causal assumptions, scalability, richer intervention models, realistic evaluation, theoretical foundations, and integration with representation learning) that must be addressed for CBO to achieve reliable real-world deployment (Section~\ref{sec:open_problems}).
\end{itemize}

\paragraph{Paper outline.}
Section~\ref{sec:related_surveys} places this survey relative to existing reviews. Section~\ref{sec:foundations} establishes the theoretical foundations in causal inference and Bayesian optimization. Section~\ref{Methods} presents the methodological frameworks organized by design axes. Section~\ref{sec:experiments} describes the benchmark, metrics, datasets, and empirical results. Section~\ref{sec:open_problems} discusses open problems, and Section~\ref{sec:conclusion} concludes.

\subsection*{Related Surveys and Scope}
\label{sec:related_surveys}

Several surveys cover topics adjacent to CBO, but none provides the unified treatment of causal optimization methods and standardized empirical evaluation that we pursue here.

\paragraph{Bayesian optimization surveys.}
Comprehensive reviews of BO cover surrogate modeling, acquisition functions, and scalability \citep{shahriari2016taking,frazier2018tutorial,garnett2023bayesian}, but treat the objective as a black box and do not address causal structure, intervention design, or the observation-intervention trade-off central to CBO.

\paragraph{Causal inference and causal discovery.}
Textbooks and surveys on causal inference \citep{pearl2009causality,peters2017elements,hernan2020whatif,glymour2019review} provide the framework of the structural causal model that underlies CBO, but focus on the identification and estimation of causal effects rather than on sequential optimization of interventions under budget constraints. Causal discovery surveys \citep{spirtes2000causation,chickering2002optimal,glymour2019review} address the learning of the graph from data, which is a component of unknown-graph CBO but not its main goal.

\paragraph{Causal bandits and Bayesian experimental design.}
Causal bandits \citep{lattimore2016causal,lu2021causal,nair2021budgeted} share CBO's goal of selecting interventions using causal structure but typically assume discrete action sets with parametric reward models, whereas CBO handles continuous domains with nonparametric surrogates; Bayesian optimal experimental design \citep{chaloner1995bayesian,foster2021deep} selects experiments to maximize information about model parameters rather than to optimize an outcome. Section~\ref{subsec:adjacent_fields} details both relationships.

\paragraph{Scope of this survey.}
We focus on methods that explicitly combine (i)~a structural causal model or causal graph with (ii)~a sequential Bayesian optimization loop to select interventions. We exclude pure causal effect estimation without an optimization loop, causal bandits without GP-based surrogates (except when contrasting with CBO), and reinforcement learning methods that learn policies through environment interaction without an explicit causal model. We include both known- and unknown-graph settings, hard and soft interventions, and single- and multi-objective formulations.

\paragraph{Survey protocol and inclusion criteria.} To make the construction of the review corpus explicit, we followed a targeted search-and-screening protocol with a literature cutoff date of June 28, 2026. We searched OpenReview, PMLR, arXiv, Google Scholar, and the proceedings of major machine learning and causality venues using keyword combinations such as “causal Bayesian optimization,” “causal global optimization,” “Bayesian optimization with interventions,” “causal optimization,” and “optimization under causal structure.” Candidate papers were first screened by title and abstract, and then by their methodological formulation. A work was included if its central contribution couples an explicit causal model, such as an SCM, causal graph, or distribution over graphs, with a sequential Bayesian-optimization-style procedure for intervention selection. We further used backward and forward citation tracing from the original CBO paper and subsequent representative extensions to identify closely related work. Papers whose primary objective is only causal effect estimation, graph recovery, bandit learning, experimental design, or reinforcement learning were excluded from the main corpus unless they directly clarify the boundary between CBO and adjacent fields. This protocol yields a focused corpus of methods in which causal assumptions actively modify the Bayesian optimization loop, rather than a broad survey of causal decision-making methods.

\section{Theoretical Foundations}
\label{sec:foundations}
This section establishes the conceptual and mathematical foundations underlying Causal Bayesian Optimization.
We first review structural causal models, interventions, and identifiability (Section~\ref{subsec:causal_inference}).
We then summarize classical Bayesian optimization with an emphasis on components that CBO modifies (Section~\ref{subsec:classical_bo}).
Finally, we formulate the CBO problem and highlight the key ways in which causality reshapes the BO loop (Section~\ref{subsec:cbo_formulation}).
Table~\ref{tab:notation} consolidates the notation used throughout the paper.

\begin{table}[t]
\caption{Summary of notation used throughout this survey.}
\label{tab:notation}
\centering
\small
\setlength{\tabcolsep}{5pt}
\renewcommand{\arraystretch}{1.12}
\begin{tabular}{@{}ll@{}}
\toprule
Symbol & Meaning \\
\midrule
$M = \langle U, V, F, P(U)\rangle$ & Structural Causal Model \\
$V = \{V_1,\dots,V_n\}$ & Endogenous (observed) variables \\
$U$ & Exogenous variables (determined outside the system) \\
$\mathrm{Pa}_i$ & Parents of $V_i$ in the causal graph $\mathcal{G}$ \\
$\mathcal{G}$ & Causal directed acyclic graph (DAG) \\
$do(X{=}\mathbf{x})$ & Hard intervention setting $X$ to value $\mathbf{x}$ \\
$X, Y$ & Manipulable variables, target outcome \\
$C$ & Context variables (used only by contextual extensions, Section~\ref{subsec:intervention_developments}) \\
$X_s \subseteq X$ & Intervention scope (subset of manipulable variables) \\
$\mathbf{x}_s \in \mathcal{D}(X_s)$ & Intervention value in the feasible domain \\
$f(X_s,\mathbf{x}_s)$ & Interventional objective: $\mathbb{E}[Y \mid do(X_s{=}\mathbf{x}_s)]$ ($C$ enters only in contextual extensions) \\
$\mathcal{GP}(m, k)$ & Gaussian process with mean $m$ and kernel $k$ \\
$\mu_t(\cdot),\;\sigma_t^2(\cdot)$ & GP posterior mean and variance after $t$ observations \\
$\alpha(\cdot)$ & Acquisition function \\
$T$ & Total number of interventional trials (budget) \\
$y^*$ & Reference optimum (for metric computation) \\
$R_t$ & Best-so-far improvement ratio at trial $t$ \\
\bottomrule
\end{tabular}
\end{table}

\subsection{Causal Inference}
\label{subsec:causal_inference}

Most CBO methods adopt the framework of Structural Causal Models (SCMs) \citep{pearl2009causality}, which provide a compact, modular representation of how variables are generated and how they respond to interventions.

\begin{Definition}[Structural Causal Model (SCM)]
An SCM is a tuple $M=\langle U,V,F,P(U)\rangle$ where $U$ is a set of exogenous variables, $V=\{V_1,\dots,V_n\}$ is a set of endogenous variables, and
$F=\{f_1,\dots,f_n\}$ is a collection of structural assignments
\begin{equation}
V_i := f_i(\mathrm{Pa}_i, U_i), \qquad i=1,\dots,n,
\end{equation}
where $\mathrm{Pa}_i \subseteq V\setminus\{V_i\}$ denotes the direct causes (parents) of $V_i$ in the causal graph and $U_i\subseteq U$ is the exogenous noise that affects $V_i$.
The distribution $P(U)$ completes the model and, together with $F$, induces a unique joint distribution on $V$. The exogenous variables $U$ are variables determined outside the modeled causal system; their probability distribution $P(U)$ is explicitly part of the SCM specification, so they are not ``unmodeled'' but rather externally determined.
\end{Definition}

\paragraph{Causal graphs, assumptions, and the causal Markov property.}
An SCM induces a causal graph $\mathcal{G}$, commonly represented as a directed acyclic graph (DAG), with a node for each endogenous variable and an edge $V_j \to V_i$ whenever $V_j \in \mathrm{Pa}_i$. The graph encodes a qualitative causal structure in which variables are direct causes of others but does not, by itself, specify functional forms or effect magnitudes. That quantitative information resides in the structural equations $F$.

Two standard assumptions link the graph to the distribution. The \emph{causal Markov condition} states that each variable is conditionally independent of its non-descendants given its parents \citep{pearl2009causality,spirtes2000causation}. Under the additional assumption that exogenous variables are mutually independent (\emph{causal sufficiency}), the observational distribution factorizes according to the graph.
\begin{equation}
P(V_1,\dots,V_n) \;=\; \prod_{i=1}^n P(V_i \mid \mathrm{Pa}_i).
\label{eq:markov_factorization}
\end{equation}
The \emph{faithfulness} assumption further requires that all conditional independencies in $P$ are entailed by the Markov condition applied to $\mathcal{G}$, ruling out cancelations that could hide edges. Together, these assumptions enable structure learning from observational data and underpin the identifiability results that CBO exploits.

\paragraph{Hard interventions.}
In Pearl's framework, a \emph{hard} (or \emph{perfect}) intervention $do(X{=}\mathbf{x})$ replaces the structural equation for $X$ with the constant $\mathbf{x}$, severing all incoming edges to $X$ while leaving other mechanisms unchanged. This produces a modified SCM $M_{do(X=\mathbf{x})}$ and an interventional distribution that admits a \emph{truncated factorization}:
\begin{equation}
P(V_1,\dots,V_n \mid do(X{=}\mathbf{x})) \;=\; \prod_{V_i \notin X} P(V_i \mid \mathrm{Pa}_i)\;\bigg|_{X=\mathbf{x}}.
\label{eq:truncated_factorization}
\end{equation}
Hard interventions are the default in the original CBO formulation and in most Hard intervention benchmarks.

\paragraph{Soft interventions.}
A \emph{soft} (or \emph{imperfect}) intervention modifies the structural mechanism of a variable without completely severing its incoming edges. Formally, a soft intervention on $V_i$ replaces $f_i$ with a new mechanism $\tilde{f}_i(\mathrm{Pa}_i, U_i; \theta)$ parameterized by an intervention parameter $\theta$. In the broader literature, soft interventions can also encompass structural changes such as dropping or adding parent dependencies, or policy-based functional interventions where $\theta$ parameterizes a function mapping contexts to mechanism parameters; the formal summaries in Section~\ref{Methods} state what $\theta$ represents for each method. This is important in several CBO variants, notably MCBO \citep{Sussex2022MCBO} and fCBO \citep{gultchin2023FCBO}, that model mechanism modifications rather than variable clamping. Soft interventions yield a richer action space but require modeling how the intervention parameter interacts with the existing mechanism.

\paragraph{Counterfactuals.}
SCMs also define counterfactual quantities that describe outcomes under hypothetical actions in a specific latent world (indexed by $u\in U$), typically written as $Y_\mathbf{x}(u)$ \citep{pearl2009causality}.
The methods surveyed here optimize interventional expectations rather than counterfactuals; counterfactual CBO (for example, personalized optimization after observing unit-specific histories, or counterfactual safety constraints) remains a future direction.

\paragraph{Identifiability and the $do$-calculus.}
In practice, the full SCM is rarely known, so CBO relies on \emph{identifiability}: whether an interventional quantity such as $\mathbb{E}[Y\mid do(X=\mathbf{x})]$ can be expressed purely in terms of the observational distribution given assumptions about the graph.
Pearl's $do$-calculus provides three general inference rules for transforming interventional expressions into observational quantities when identification is possible \citep{pearl2009causality,pearl2012docalculus}. A common special case is the \emph{backdoor adjustment}: if a set $Z$ satisfies the backdoor criterion relative to $(X, Y)$, then
\begin{equation}
\mathbb{E}[Y \mid do(X{=}\mathbf{x})] \;=\; \sum_z \mathbb{E}[Y \mid X{=}\mathbf{x}, Z{=}z]\, P(Z{=}z).
\label{eq:backdoor}
\end{equation}
When identification holds but closed-form evaluation is difficult, one can approximate the required expectations via Monte Carlo integration of the identified estimand.

\paragraph{Identifiability as an algorithmic assumption in CBO.}
In CBO, identifiability is not merely a causal inference condition; it directly shapes the optimizer's behavior. CBO benefits from identifiability to construct observationally informed priors, which accelerates sample efficiency, but the framework does not strictly rely on it to function. If a causal effect is not identifiable under the given graph, a sound algorithm can recognize this and fall back to a standard uninformative GP prior while still leveraging the causal graph for search-space reduction via POMIS. When an intervention effect is identifiable from observational data, the observational sample can be used to construct informative GP priors, initialize surrogates, or eliminate dominated intervention scopes \emph{before} any costly intervention is performed.
Two failure modes should be distinguished: if identifiability fails on a correctly specified graph, a sound adjustment procedure defaults to an uninformative prior, avoiding bias; overconfident, biased priors instead arise when the graph or adjustment assumptions themselves are wrong---a failure of the graph assumption rather than of identifiability. Empirical comparisons should therefore distinguish methods that assume correct identifiable priors from those that rely primarily on interventional data or explicitly model graph and identification uncertainty. Throughout this survey, we follow the common CBO assumption of causal sufficiency and perfect interventions unless otherwise stated.

\subsection{Classical Bayesian Optimization}
\label{subsec:classical_bo}

\textbf{Bayesian Optimization (BO)} is a sequential strategy for optimizing an expensive black-box objective $f:\mathcal{X}\to\mathbb{R}$ on a bounded domain $\mathcal{X}\subset\mathbb{R}^d$ \citep{Kushner1964, Zhilinskas1975, Mockus1978, Mockus1989, Jones1998, shahriari2016taking}.
At each iteration $t$, BO selects a query $\mathbf{x}_t$, observes $y_t = f(\mathbf{x}_t) + \epsilon_t$ with noise $\epsilon_t$, and updates a probabilistic model to guide future queries. BO is characterized by two interacting components.

\paragraph{Probabilistic surrogate.}
The most common surrogate is a Gaussian Process (GP) \citep{rasmussen2006gaussian}, specified by a prior mean function $m_0(\cdot)$ and a kernel (covariance function) $k(\cdot,\cdot)$. Given $t$ observations $\mathcal{D}_t = \{(\mathbf{x}_i, y_i)\}_{i=1}^t$, the GP posterior of any candidate input $\mathbf{x}$ is available in closed form:
\begin{align}
\mu_t(\mathbf{x}) &= m_0(\mathbf{x}) + \mathbf{k}_t(\mathbf{x})^\top (\mathbf{K}_t + \sigma_\epsilon^2 \mathbf{I})^{-1}(\mathbf{y}_t - \mathbf{m}_0), \label{eq:gp_mean}\\
\sigma_t^2(\mathbf{x}) &= k(\mathbf{x},\mathbf{x}) - \mathbf{k}_t(\mathbf{x})^\top (\mathbf{K}_t + \sigma_\epsilon^2 \mathbf{I})^{-1} \mathbf{k}_t(\mathbf{x}), \label{eq:gp_var}
\end{align}
where $\mathbf{k}_t(\mathbf{x})$ is the vector of covariances between $\mathbf{x}$ and the observed inputs, $\mathbf{K}_t$ is the kernel matrix over the observed inputs, and $\sigma_\epsilon^2$ is the variance of the observation noise. The posterior mean $\mu_t$ provides a point estimate, while $\sigma_t^2$ quantifies the predictive uncertainty, enabling the acquisition function to distinguish well-explored from poorly-understood regions.

\paragraph{Acquisition functions.}
The acquisition function $\alpha(\mathbf{x}; \mathcal{D}_t)$ scores candidate inputs by balancing exploitation (questioning where $\mu_t$ is favorable) against exploration (questioning where predictive uncertainty or information gain is large). Two widely used choices are Expected Improvement (EI) and the Upper Confidence Bound (UCB):
\begin{align}
\alpha_{\mathrm{EI}}(\mathbf{x}) &= \mathbb{E}\big[\max(f(\mathbf{x}) - f^+, 0) \mid \mathcal{D}_t\big], \label{eq:ei}\\
\alpha_{\mathrm{UCB}}(\mathbf{x}) &= \mu_t(\mathbf{x}) + \beta_t^{1/2}\, \sigma_t(\mathbf{x}), \label{eq:ucb}
\end{align}
where $f^+$ is the incumbent value (in noisy settings, the best noisy observation $y^+$ or a posterior estimate of the latent incumbent) and $\beta_t > 0$ is a schedule parameter that controls the exploration--exploitation trade-off \citep{Jones1998, Srinivas2010GP-UCB}. The BO loop iterates: maximize $\alpha$ to select $\mathbf{x}_{t+1}$, evaluate $f$, and update the GP posterior.

\paragraph{Regret and sample efficiency.}
Under regularity conditions in the kernel and the appropriate choice of $\beta_t$, GP-UCB achieves a sublinear cumulative regret $\sum_{t=1}^T [f(\mathbf{x}^*) - f(\mathbf{x}_t)] = \mathcal{O}(\sqrt{T\gamma_T})$, where $\gamma_T$ is the maximum information gain of the kernel \citep{Srinivas2010GP-UCB,chowdhury2017kernelized}. This theoretical grounding is relevant to CBO because MCBO \citep{Sussex2022MCBO} and ACBO \citep{ACBOSussex} extend GP-UCB-style regret analysis to causal settings, obtaining bounds that can improve over the non-causal case when the causal graph induces a factored reward structure.

\paragraph{Beyond standard GPs.}
Although GPs remain the dominant surrogate in CBO, the broader BO literature also uses random forests, Bayesian neural networks, neural processes, and deep kernel learning \citep{shahriari2016taking,garnett2023bayesian}. These alternatives are not causal by themselves. They become CBO components only when the modeled target is a causal estimand, when observational data enter through an identified causal prior, or when the architecture respects graph or mechanism structure. Thus, surrogate choice encodes both a statistical model class and a causal modeling target: effect-level surrogates model $\theta\mapsto \mathbb{E}[Y\mid do(X_s=\theta)]$, whereas mechanism-level surrogates model structural equations and propagate uncertainty through the graph.

\subsection{Causal Bayesian Optimization}
\label{subsec:cbo_formulation}

In CBO, the manipulable variables are linked to the target through structural mechanisms rather than through a black-box input--output map. This causal structure requires and enables algorithmic choices that standard BO does not face: which variables to intervene on (the intervention scope), at what level, and how abundant observational data relate to interventional effects through causal identification. This coupling of combinatorial scope selection with continuous value optimization, guided by a causal graph, is what distinguishes CBO from black-box BO.

Section~\ref{subsec:core-cbo} formalizes the CBO problem, defines the key concepts of minimal and possibly-optimal intervention sets, presents the surrogate and acquisition design, and identifies the limitations that motivate the extensions reviewed in the remainder of Section~\ref{Methods}.

\section{Methodological Frameworks}
\label{Methods}

Causal Bayesian Optimization (CBO) is best viewed as a design space rather than a single algorithm.
All CBO methods instantiate a common loop: maintain a probabilistic model of candidate interventional effects, score candidate interventions with an acquisition or decision rule, select the next intervention or observation action, and update beliefs using newly obtained observational and/or interventional data. What differentiates methods is where the causal structure is injected into this loop and which uncertainty sources are modeled.

\paragraph{A unifying lens: design axes and assumptions.}
In the literature, most variants can be understood through five interacting axes (Figure~\ref{fig:cbo_categories}):
\begin{enumerate}
    \item \textbf{Graph and system-knowledge assumptions.} Known graph, unknown graph, graph uncertainty, causal sufficiency, hidden confounding, and identifiability assumptions.
    \item \textbf{Environmental assumptions.} Static, dynamic, contextual, constrained, adversarial or non-stationary, and multi-objective settings.
    \item \textbf{Intervention representation.} Hard interventions, soft mechanism modifications, intervention scopes, functional policies, context-dependent actions, and mixed discrete--continuous actions.
    \item \textbf{Surrogate architecture.} Direct effect-level surrogates, mechanism-level surrogates, graph-uncertain surrogates, multi-output models, and neural or deep-kernel variants.
    \item \textbf{Decision rule and budget allocation.} EI-/UCB-style acquisition, information-gain search, optimism over plausible models, bandit allocation across scopes, and online-learning rules.
\end{enumerate}
We use these axes to highlight what assumptions each method relaxes or introduces, what mathematical object changes relative to baseline CBO, and what statistical efficiency or computational cost follows.

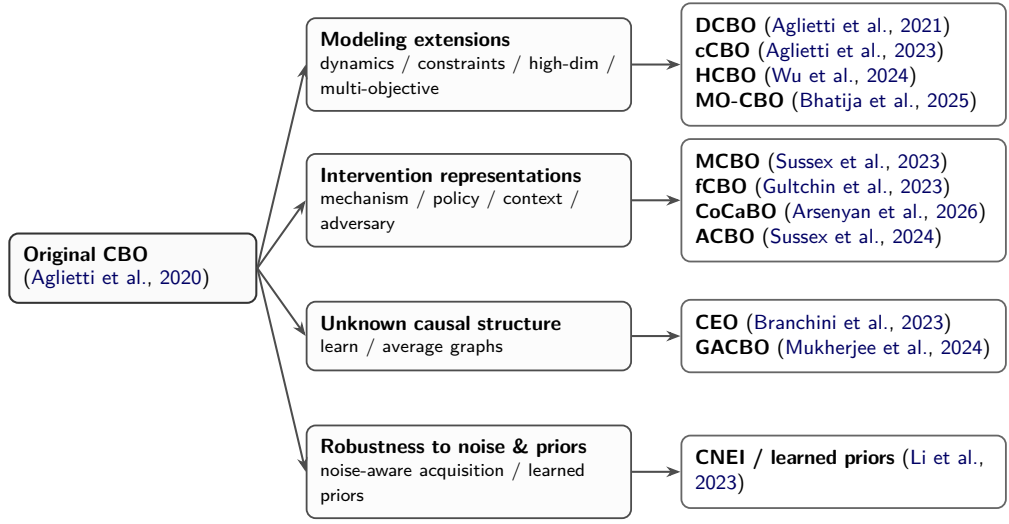
\begin{figure}[!ht]
\centering
\resizebox{0.80\linewidth}{!}{%
\begin{tikzpicture}[
  font=\small\sffamily,
  arr/.style={-{Stealth[length=2.2mm]}, line width=0.9pt, draw=black!70},
  root/.style={draw=black!80, rounded corners=4pt, line width=0.9pt,
               fill=black!2, inner sep=6pt, align=left, text width=3.4cm},
  cat/.style={draw=black!70, rounded corners=4pt, line width=0.8pt,
              fill=black!1, inner sep=6pt, align=left, text width=4.6cm},
  lst/.style={draw=black!60, rounded corners=4pt, line width=0.8pt,
              fill=white, inner sep=6pt, align=left, text width=4.6cm},
  tight/.style={}
]

\node[root, tight] (cbo) at (0,0) {\textbf{Original CBO}\\[-1pt]\citep{aglietti2020CBO}};

\node[cat,  tight] (mod) at (5.2, 3.15) {\textbf{Modeling extensions}\\[-1pt]\footnotesize dynamics / constraints / high-dim / multi-objective};
\node[cat,  tight] (int) at (5.2, 1.05) {\textbf{Intervention representations}\\[-1pt]\footnotesize mechanism / policy / context / adversary};
\node[cat,  tight] (unk) at (5.2,-1.05) {\textbf{Unknown causal structure}\\[-1pt]\footnotesize learn / average graphs};
\node[cat,  tight] (rob) at (5.2,-3.15) {\textbf{Robustness to noise \& priors}\\[-1pt]\footnotesize noise-aware acquisition / learned priors};

\node[lst, tight] (modex) at (11.0, 3.15) {%
\textbf{DCBO} \citep{aglietti2021}\\
\textbf{cCBO} \citep{aglietti2023cCBO}\\
\textbf{HCBO} \citep{Wu2024HighDimensionalCBO}\\
\textbf{MO-CBO} \citep{MOCBOBhatija}
};

\node[lst, tight] (intex) at (11.0, 1.05) {%
\textbf{MCBO} \citep{Sussex2022MCBO}\\
\textbf{fCBO} \citep{gultchin2023FCBO}\\
\textbf{CoCaBO} \citep{ContextualCBO}\\
\textbf{ACBO} \citep{ACBOSussex}
};

\node[lst, tight] (unkex) at (11.0,-1.05) {%
\textbf{CEO} \citep{branchini2023CEO}\\
\textbf{GACBO} \citep{mukherjee2024}
};

\node[lst, tight] (robex) at (11.0,-3.15) {%
\textbf{CNEI / learned priors} \citep{Li2023}
};

\draw[arr] (cbo.east) -- (mod.west);
\draw[arr] (cbo.east) -- (int.west);
\draw[arr] (cbo.east) -- (unk.west);
\draw[arr] (cbo.east) -- (rob.west);

\draw[arr] (mod.east) -- (modex.west);
\draw[arr] (int.east) -- (intex.west);
\draw[arr] (unk.east) -- (unkex.west);
\draw[arr] (rob.east) -- (robex.west);

\end{tikzpicture}%
}
\caption{Causal Bayesian Optimization development map organized by the primary bottleneck addressed.}
\label{fig:cbo_categories}
\end{figure}

\paragraph{How we summarize methods.}
For each approach, we emphasize: (i) assumptions (graph knowledge, intervention type, observed variables); (ii) the main algorithmic lever (search-space reduction, surrogate design, acquisition/decision policy); and (iii) failure modes and scaling constraints.
This avoids re-stating paper narratives verbatim and instead clarifies how the methods relate at the level of design choices in the CBO loop.
To make this comparison fully systematic, every surveyed method closes with a \emph{formal summary} block whose four slots restate, as mathematics, the four columns of Table~\ref{tab:cbo-variants}: (i)~assumptions, (ii)~objective, (iii)~surrogate, and (iv)~acquisition. Slots whose defining equation already appears in the text point to it rather than repeating it.
Unless a statement is explicitly tied to benchmark evidence in Section~\ref{sec:experiments}, the ``main challenge'' and limitation remarks in this section are \emph{analytical}: they follow from each method's stated assumptions, surrogate level, and decision rule (see also the identification-strategy audit in Table~\ref{tab:ident-audit}), and should not be read as new empirical findings of this survey.

\paragraph{Common notation for method summaries.}
To make the method summaries precise and comparable, we use a single notation throughout this section: $\mathcal{G}$ denotes the causal graph, $X_s \subseteq X$ an intervention scope, $\boldsymbol{\theta}$ an intervention value or policy parameter, $Y$ the target outcome, and $\mathcal{D}_{\mathrm{obs}}$ and $\mathcal{D}_{\mathrm{int}}$ the observational and interventional datasets. Effect-level surrogates model the map $g_s(\boldsymbol{\theta}) = \mathbb{E}[Y \mid do(X_s{=}\boldsymbol{\theta})]$ directly, whereas mechanism-level surrogates place a model $\hat{f}_j$ on each structural equation $f_j(\mathrm{Pa}_j, U_j)$ and obtain the effect on $Y$ by propagating samples or uncertainty through $\mathcal{G}$. Actions are written in one of three equivalent forms depending on the method family: hard scope--value interventions $(X_s,\mathbf{x}_s)$ (CBO, DCBO, cCBO, HCBO, MO-CBO, CEO, GACBO, CNEI), mechanism-shift actions $\boldsymbol{\theta}\in\Theta$ that enter the structural equations as inputs (MCBO, ACBO), and policies $\pi$ (fCBO, CoCaBO). Objectives follow the minimization convention of Definition~\ref{def:cbo_problem} unless a method is inherently a maximization or regret formulation (cCBO, CoCaBO, GACBO, ACBO), in which case the conversion is immediate.

\subsection{Causal Bayesian Optimization: Core Framework}
\label{subsec:core-cbo}

The original CBO framework \citep{aglietti2020CBO} formalizes causal optimization as a search over interventions in a known causal graph. Let $X$ denote manipulable variables, $C$ non-manipulable context variables, and $Y$ the target outcome. An action consists of an intervention scope $X_s \subseteq X$ and an assignment $\mathbf{x}_s$ in the feasible intervention domain $\mathcal{D}(X_s)$.

\begin{Definition}[CBO objective and action space]\label{def:cbo_problem}
A CBO action is a pair $(X_s,\mathbf{x}_s)$ where $X_s \subseteq X$
is an intervention scope and $\mathbf{x}_s \in \mathcal{D}(X_s)$ is an assignment in the feasible intervention domain. The interventional objective is
\begin{equation}
f(X_s,\mathbf{x}_s) \;=\; \mathbb{E}\!\left[\,Y \mid do(X_s=\mathbf{x}_s)\,\right].
\end{equation}
The global causal optimization problem is
\begin{equation}
(X_s^\star,\mathbf{x}_s^\star)
\in
\argmin_{X_s \subseteq X,\; \mathbf{x}_s \in \mathcal{D}(X_s)}
\; f(X_s,\mathbf{x}_s),
\label{eq:cgo}
\end{equation}
with $\argmin$ replaced by $\argmax$ for maximization tasks.
\end{Definition}

Definition~\ref{def:cbo_problem} states the baseline CBO problem strictly according to its original formulation \citep{aglietti2020CBO}, which does not involve context variables. Context variables $C$ are introduced only in Section~\ref{subsec:intervention_developments} when discussing contextual extensions such as CoCaBO (Equation~\ref{eq:cocabo}). The objective in Definition~\ref{def:cbo_problem} is an interventional expectation: none of the methods surveyed in this section optimizes counterfactual quantities $Y_{\mathbf{x}}(u)$, which we treat only as a future direction (Section~2.1).

The key point is that the optimizer is free to choose the scope of the intervention, not only the
level. This yields a mixed discrete--continuous search problem and motivates two
core ideas in CBO: graph-based pruning of candidate scopes and observationally informed priors over
interventional effects.

\paragraph{Why causality changes the BO problem.}
A naive approach would flatten $(X_s,\mathbf{x}_s)$ into a single input vector and run the standard BO on the interventional objective.
This ignores the structural information encoded in the causal graph. Causal knowledge enables two advantages that standard BO cannot exploit: (i)~\emph{scope reduction}: identifying which variables need to be manipulated at all, thereby shrinking the combinatorial search; and (ii)~\emph{informative priors}: converting plentiful observational data into prior beliefs about interventional effects, which is possible only when identifiability holds (Section~\ref{subsec:causal_inference}).

\paragraph{Scope reduction: minimal intervention sets.}
A central insight of CBO is that the causal graph can prune intervention scopes that are structurally redundant. Two key notions formalize this:

\begin{Definition}[Minimal Intervention Set (MIS) \citep{lee2018structural,lee2019structural}]\label{def:mis}
An intervention scope $X_s \subseteq X$ is a \emph{minimal intervention set} for the target $Y$ if (i)~there exists a directed path from at least one variable in $X_s$ to $Y$ in the causal graph $\mathcal{G}$ that is not blocked by the intervention on the remaining variables in $X_s$, and (ii)~no proper subset $X_s' \subsetneq X_s$ satisfies condition~(i). Equivalently, $X_s$ is a minimal set of manipulable variables whose joint intervention can affect $Y$ through the graph.
\end{Definition}

\begin{Definition}[Possibly-Optimal MIS (POMIS) \citep{lee2018structural,lee2019structural}]\label{def:pomis}
A MIS $X_s$ is \emph{possibly optimal} if, given the observational distribution $P(V)$ and the causal graph $\mathcal{G}$, one cannot determine from observational evidence alone that $\max_{\mathbf{x}_s} \mathbb{E}[Y \mid do(X_s = \mathbf{x}_s)]$ is dominated by the optimal objective achievable under some other MIS $X_{s'}$. Formally, the set of POMIS is the subset of MISs that remains after eliminating scopes whose optimal interventional value is provably dominated under the available graph and distributional information.
\end{Definition}

By restricting the search to POMIS, CBO avoids wasting budget on scopes that are causally irrelevant (no path to $Y$) or observationally dominated (another scope provably achieves a better objective). This reduction is not merely computational; it changes the statistical problem by concentrating evaluations on scopes that the graph predicts can meaningfully move the outcome, and it is one of the main reasons CBO can be statistically more efficient than standard BO. It is also complementary to causal abstraction, which relates decisions across graphs at different granularities (Section~\ref{subsec:adjacent_fields}).

\paragraph{Surrogate modeling: effect-level GPs with causal priors.}
The CBO models the interventional objective for each candidate scope with an effect-level GP. The key innovation is the prior construction: when the interventional effect is identifiable, observational data are used to build an initial estimate of the interventional mean and its uncertainty:
\begin{align}
g_s(\mathbf{x}_s) &:= \mathbb{E}[Y \mid do(X_s=\mathbf{x}_s)],\\
g_s(\cdot) &\sim \mathcal{GP}\!\left(m_s(\cdot),\, k_s(\cdot,\cdot)\right), \label{eq:cbo_gp_prior}\\
m_s(\mathbf{x}_s) &\approx \widehat{\mathbb{E}}\!\left[Y \mid do(X_s=\mathbf{x}_s)\right],\\
k_s(\mathbf{x}_s,\mathbf{x}'_s)
&=
k_{\mathrm{RBF}}(\mathbf{x}_s,\mathbf{x}'_s)
+
\widehat{\sigma}_s(\mathbf{x}_s)\,\widehat{\sigma}_s(\mathbf{x}'_s),
\end{align}
where $k_{\mathrm{RBF}}(\mathbf{x},\mathbf{x}')=\exp\!\big(-\|\mathbf{x}-\mathbf{x}'\|^2/(2\ell^2)\big)$ and $\widehat{\sigma}_s(\mathbf{x}_s)$ is the uncertainty in the observationally-derived effect estimate. The observational data thus serve as prior information, not as additional labels, giving the surrogate a reasonable global shape before costly interventions begin.

\paragraph{Acquisition and the observation--intervention trade-off.}
CBO uses a cost-aware expected improvement to choose between pairs (scope, value):
\begin{equation}
\mathrm{CEI}_s(\mathbf{x}_s)
=
\frac{\mathbb{E}\big[(y^{+} - g_s(\mathbf{x}_s))_+\big]}
{\mathrm{Co}(X_s,\mathbf{x}_s)},
\label{eq:cei}
\end{equation}
where $y^{+}$ is the incumbent (the best value observed so far; we reserve $y^*$ for the reference optimum used in the evaluation metrics of Section~\ref{sec:experiments}) and $\mathrm{Co}(\cdot)$ is an intervention cost.
Beyond the standard exploration--exploitation trade-off, CBO introduces an \emph{observation--intervention trade-off}: the algorithm can decide whether to allocate budget to additional observational samples (cheap but potentially biased for interventional decisions) or to new interventions (expensive but causally informative). This framing is foundational: many later variants can be interpreted as changing \emph{what is uncertain} and therefore changing \emph{which data is worth acquiring}.

\begin{formalsummary}{CBO}{eq:cbo_summary}
\begin{align}
\fslot{assumptions}
  & \mathcal{G}\ \text{known DAG};\quad
    X_s\in\mathrm{POMIS}(\mathcal{G})\ \text{(Defs.~\ref{def:mis}--\ref{def:pomis})};\quad
    m_s\ \text{identifiable from } \mathcal{D}_{\mathrm{obs}}
    \label{eq:cbo_assum}\\
\fslot{objective}
  & \text{global causal optimum over hard interventions } (X_s,\mathbf{x}_s)\text{, Equation~\ref{eq:cgo}}
    \label{eq:cbo_objective}\\
\fslot{surrogate}
  & \text{effect-level GP with causal prior mean and inflated kernel, Equation~\ref{eq:cbo_gp_prior}}
    \label{eq:cbo_surrogate}\\
\fslot{acquisition}
  & \text{cost-aware causal EI with observe-or-intervene choice, Equation~\ref{eq:cei}}
\end{align}
\end{formalsummary}

\paragraph{The generic CBO loop.}
At each iteration, the CBO loop proceeds in three stages: (i) the surrogate predicts the causal effect for each candidate intervention (scope, value), (ii) the acquisition function scores candidates using the posterior predictions, and (iii) the highest-scoring intervention is selected and evaluated. Algorithm~\ref{alg:cbo_loop} summarizes the generic CBO procedure, highlighting where causal knowledge enters. Steps~1--2 exploit the graph \emph{before} the optimization loop; Steps~4--5 use a causal structure \emph{within} each iteration. This template applies, with variations, to all methods reviewed in the remainder of this section.

\begin{algorithm}[t]
\caption{Generic Causal Bayesian Optimization Loop}\label{alg:cbo_loop}
\begin{algorithmic}[1]
\Require Causal graph $\mathcal{G}$ (or prior over graphs), observational data $\mathcal{D}_{\mathrm{obs}}$, budget $T$
\Ensure Best intervention $(X_s^\star, \mathbf{x}_s^\star)$
\State \textbf{Scope reduction} {\small\color{gray}[causal]}: Use $\mathcal{G}$ to identify candidate scopes (MIS/POMIS)
\State \textbf{Prior construction} {\small\color{gray}[causal]}: Use $\mathcal{D}_{\mathrm{obs}}$ and identifiability to set GP prior means $\{m_s\}$
\For{$t = 1, \dots, T$}
    \State \textbf{Surrogate update}: Update GP posteriors $\{\mu_t^{(s)}, \sigma_t^{(s)}\}$ for each scope $s$
    \State \textbf{Acquisition} {\small\color{gray}[causal]}: Select $(X_s, \mathbf{x}_s) = \argmax_{s,\mathbf{x}} \alpha_s(\mathbf{x}; \mathcal{D}_t)$ across scopes
    \State \textbf{Evaluate}: Observe $y_t = f(X_s, \mathbf{x}_s) + \epsilon_t$ via intervention
    \State \textbf{Update}: Add $(X_s, \mathbf{x}_s, y_t)$ to interventional dataset $\mathcal{D}_t$
\EndFor
\State \Return $(X_s^\star, \mathbf{x}_s^\star)$ with best observed objective
\end{algorithmic}
\end{algorithm}

\paragraph{Limitations that motivate extensions.}
The modularity of CBO also exposes its bottlenecks: (i)~reliance on a correctly known graph and identifiable observational priors; (ii)~limited robustness to noise and prior misspecification; and (iii)~scaling challenges due to exponentially many candidate scopes and separate surrogates per scope. The remainder of this section can be read as a sequence of responses to these bottlenecks, organized by the design axis that each method primarily addresses.

\subsection{Modeling Extensions}
\label{subsec:modeling_extensions}

The original CBO formulation assumes a static causal system, a single objective, and an
intervention space that remains tractable after graph-based pruning. Modeling extensions relax these assumptions by changing what the optimizer must represent: time variation that renders the objective non-stationary, feasibility that restricts attention to a safe set, dimensionality that enlarges the scope of the space, or multiple competing goals that shift the target from a single optimum to a Pareto set, in each case keeping the same surrogate--acquisition loop while adding structure that decomposes the problem into smaller subproblems that share statistical strength.

\paragraph{Time-evolving causal systems (DCBO).}
Dynamic Causal Bayesian Optimization (DCBO) addresses environments where the best intervention is non-stationary, typically because mechanisms drift or the system state evolves \citep{aglietti2021}.
The main contribution is not a new acquisition rule, but a transfer mechanism. Under invariance assumptions such as a fixed graph and additive-noise structure, interventional knowledge from earlier time steps can be mapped into an informative prior for the current step. This treats time as a source of reusable evidence so that exploration can be amortized rather than restarted. The trade-off is the sensitivity of the assumption. If the graph or key mechanisms shift in ways not captured by the transfer recursion, the prior can become overconfident and bias the search, so DCBO is strongest when temporal variation is structured and not adversarial.

\begin{formalsummary}{DCBO}{eq:dcbo_summary}
\begin{align}
\fslot{assumptions}
  & \mathcal{G}_t=\mathcal{G}_0\ \ \forall t;\quad
    Y_t=f_Y^{\mathrm{p}}\big(\mathrm{Pa}^{\mathrm{T}}(Y_t)\big)
       +f_Y^{\mathrm{np}}\big(\mathrm{Pa}^{\mathrm{NT}}(Y_t)\big)+\epsilon_t;\quad
    \text{no hidden confounders in } \mathcal{G}_{0:T}
    \label{eq:dcbo_assum}\\
\fslot{objective}
  & (X_{s,t}^{\star},\mathbf{x}_{s,t}^{\star})\in
    \argmin_{X_{s,t}\subseteq X_t,\ \mathbf{x}\in\mathcal{D}(X_{s,t})}
    \mathbb{E}\big[Y_t\mid do(X_{s,t}{=}\mathbf{x}),\,I_{0:t-1}\big],
    \quad
    I_{0:t-1}=\textstyle\bigcup_{i=0}^{t-1}do\big(X_{s,i}^{\star}{=}\mathbf{x}_{s,i}^{\star}\big)
    \label{eq:dcbo_obj}\\
\fslot{surrogate}
  & g_{s,t}(\mathbf{x})=f_Y^{\mathrm{p}}\big(\mathbf{f}_{t-1}^{\star}\big)
    +\mathbb{E}_{p(\mathbf{w}\mid do(X_{s,t}=\mathbf{x}),\,I_{0:t-1})}
     \big[f_Y^{\mathrm{np}}(\mathbf{x},\mathbf{w})\big],
    \quad
    g_{s,t}\sim\mathcal{GP}\big(m_{s,t},\,k_{s,t}\big)
    \label{eq:dcbo_surr}\\
\fslot{acquisition}
  & \mathrm{CEI}_{s,t}(\mathbf{x})=
    \frac{\mathbb{E}\big[(y_t^{+}-g_{s,t}(\mathbf{x}))_{+}\big]}
         {\mathrm{Co}(X_{s,t},\mathbf{x})}
    \label{eq:dcbo_acq}
\end{align}
\end{formalsummary}
Here $\mathrm{Pa}^{\mathrm{T}}(Y_t)\subseteq\{Y_0,\dots,Y_{t-1}\}$ and $\mathrm{Pa}^{\mathrm{NT}}(Y_t)$ partition the parents of $Y_t$ into past targets and non-target variables, $\mathbf{f}_{t-1}^{\star}=\{\mathbb{E}[Y_i\mid do(X_{s,i}^{\star}{=}\mathbf{x}_{s,i}^{\star}),I_{0:i-1}]\}_{Y_i\in\mathrm{Pa}^{\mathrm{T}}(Y_t)}$ collects the optimal targets achieved at earlier steps, $\mathbf{w}$ denotes the remaining non-intervened variables, and $y_t^{+}$ is the best target value observed at time $t$ across scopes. The additive split in Equation~\ref{eq:dcbo_assum} is exactly what licenses the prior-transfer recursion in Equation~\ref{eq:dcbo_surr}: the prior mean $m_{s,t}$ is obtained by plugging $\mathbf{f}_{t-1}^{\star}$ and observational estimates into Equation~\ref{eq:dcbo_surr}, and the kernel $k_{s,t}$ mirrors the static causal prior of Equation~\ref{eq:cbo_gp_prior}.

\paragraph{Constraints and safe causal optimization (cCBO).}
Constrained Causal Bayesian Optimization (cCBO) targets settings where interventions must satisfy safety or feasibility constraints \citep{aglietti2023cCBO}. The key insight from the modeling is that constraints change the definition of a useful intervention scope. A scope that can move the objective but cannot maintain feasibility is effectively irrelevant; cCBO therefore extends graph-based pruning to constraint-relevant pathways through constrained minimal intervention sets, and it couples objective and constraints through multi-output surrogates so evidence about feasibility shapes where the objective surrogate should be trusted.
Formally, cCBO solves
\begin{equation}
(X_s^\star,\mathbf{x}_s^\star) \in \argmax_{X_s,\mathbf{x}_s} \; \mathbb{E}[Y \mid do(X_s{=}\mathbf{x}_s)] \quad \text{s.t.}\quad \Pr\!\big(c_j(X_s,\mathbf{x}_s)\le 0 \;\ \forall j\big)\ge 1{-}\delta,
\label{eq:ccbo}
\end{equation}
where each $c_j$ is a constraint function and $\delta$ controls the tolerated feasibility violation. In the original formulation, the constraints are stated as thresholds on interventional constraint means, $\mathbb{E}[C_j\mid do(X_s{=}\mathbf{x}_s)]\ge\lambda_j$; Equation~\ref{eq:ccbo} recovers this form with $c_j(X_s,\mathbf{x}_s):=\lambda_j-\mathbb{E}[C_j\mid do(X_s{=}\mathbf{x}_s)]$.
In practice, this often acts as a regularizer that prevents spending the budget in regions that look promising under the objective surrogate but are unlikely to be deployable. The main limitation is scaling with the number of constraints and candidate scopes, although structural pruning mitigates this in graphs with sparse causal connectivity.

\begin{formalsummary}{cCBO}{eq:ccbo_summary}
\begin{align}
\fslot{assumptions}
  & \mathcal{G}\ \text{known DAG};\quad
    \text{constraint variables } C_1,\dots,C_J \text{ embedded in } \mathcal{G};
    \notag\\
  & \quad \text{scopes pruned to constraint-aware minimal intervention sets}
    \\
\fslot{objective}
  & \text{constrained interventional optimum, Equation~\ref{eq:ccbo}}
    \\
\fslot{surrogate}
  & \big(g_s,\,c_{s,1},\dots,c_{s,J}\big)\sim\mathcal{GP}\big(\mathbf{m}_s,\,\mathbf{K}_s\big),
    \quad
    [\mathbf{K}_s]_{ab}=\operatorname{Cov}\!\big(h_a(\mathbf{x}_s),\,h_b(\mathbf{x}'_s)\big)
    \label{eq:ccbo_surrogate}\\
\fslot{acquisition}
  & \mathrm{cEI}_s(\mathbf{x}_s)=
    \frac{\mathbb{E}\Big[\big(g_s(\mathbf{x}_s)-y^{+}_{\mathrm{feas}}\big)_{+}\,
    \textstyle\prod_{j=1}^{J}\mathbb{I}\{c_{s,j}(\mathbf{x}_s)\le 0\}\Big]}
         {\mathrm{Co}(X_s,\mathbf{x}_s)}
    \label{eq:ccbo_cei}
\end{align}
\end{formalsummary}
In Equation~\ref{eq:ccbo_surrogate}, $h_0{=}g_s$ and $h_j{=}c_{s,j}$ index the outputs of the multi-output GP. The expectation in Equation~\ref{eq:ccbo_cei} is taken under this joint posterior, whose cross-output covariances are either learned (multi-task GP) or derived from the shared structural equations under intervention ($\mathcal{G}$-informed multi-task GP); $y^{+}_{\mathrm{feas}}$ is the best \emph{feasible} incumbent in $\mathcal{D}_{\mathrm{int}}$, and the cost normalization $\mathrm{Co}(\cdot)$ replaces the scope-size normalization $|X_s|$ used in the original paper, for uniformity with Equation~\ref{eq:cei}.

\paragraph{High-dimensional intervention spaces (HCBO).}
High-Dimensional Causal Bayesian Optimization (HCBO) addresses the combinatorial bottleneck of scope selection in large graphs \citep{Wu2024HighDimensionalCBO}. Its central move is to replace exact enumeration of minimal or possibly-optimal scopes with an efficient coverage-based approximation.
Specifically, HCBO selects a scope family $\mathcal{S}^\star \subseteq 2^X$ by solving a coverage problem that approximately maximizes the fraction of causal pathways to $Y$ that are ``covered'' by at least one scope in $\mathcal{S}^\star$. Let $\mathrm{Paths}(\mathcal{G}, Y)$ denote the set of directed paths from manipulable variables to $Y$. HCBO seeks
\begin{equation}
\mathcal{S}^\star \in \argmax_{\mathcal{S} \subseteq 2^X,\, |\mathcal{S}|\le K} \;\;\bigl|\{p \in \mathrm{Paths}(\mathcal{G},Y) : \exists\, X_s \in \mathcal{S},\; X_s \cap p \ne \varnothing\}\bigr|,
\label{eq:hcbo_coverage}
\end{equation}
where $K$ is a user-specified scope budget.
This reframes high-dimensional CBO as a hierarchical problem: select a promising scope family, then run continuous BO within the chosen scopes. HCBO also introduces normalized cross-scope scoring so that surrogates trained on different scopes can be compared in a meaningful way despite scale and data imbalance.
The cost of scalability is that any coverage proxy can miss narrow, high-impact pathways, so HCBO is most reliable when the system exhibits causal sparsity or a low intrinsic causal dimension.

\begin{formalsummary}{HCBO}{eq:hcbo_summary}
\begin{align}
\fslot{assumptions}
  & \mathcal{G}\ \text{known DAG};\quad
    \text{causal sparsity: the intrinsic causal dimension is small}
    \notag\\
  & \quad \text{(the minimal covering } K \text{ in Equation~\ref{eq:hcbo_coverage})}
   \\
\fslot{objective}
  & \text{Equation~\ref{eq:cgo} restricted to the scope family } \mathcal{S}^\star \text{ of Equation~\ref{eq:hcbo_coverage}}
    \\
\fslot{surrogate}
  & \text{per-scope effect-level GPs (Equation~\ref{eq:cbo_gp_prior}) with posteriors } \big(\mu^{(s)}_t,\sigma^{(s)}_t\big)
    \\
\fslot{acquisition}
  & (X_s,\mathbf{x}_s)_{t+1}\in
    \argmax_{X_s\in\mathcal{S}^\star,\ \mathbf{x}_s\in\mathcal{D}(X_s)}\;
    \frac{\mu^{(s)}_t(\mathbf{x}_s)+\beta_t^{1/2}\,\sigma^{(s)}_t(\mathbf{x}_s)-b^{(s)}_t}{r^{(s)}_t}
    \label{eq:hcbo_score}
\end{align}
\end{formalsummary}
Equation~\ref{eq:hcbo_score} states the cross-scope scoring rule in schematic form: $b^{(s)}_t$ and $r^{(s)}_t$ are per-scope location and scale statistics that place UCB scores from surrogates fitted on different scopes and data volumes on a common scale, following the normalized scope-comparison scheme of \citet{Wu2024HighDimensionalCBO}.

\paragraph{Multiple outcomes and Pareto structure (MO-CBO).}
Multi-Objective Causal Bayesian Optimization (MO-CBO) generalizes CBO to settings with multiple targets, where the goal is to approximate a Pareto front rather than a single optimum
\citep{MOCBOBhatija}. The key observation is that causal structure is especially valuable when
objectives compete. Intervening on everything can obscure trade-offs by coupling variables that need not be coupled. MO-CBO uses the graph to reduce the search to a minimal collection of local multi-objective subproblems (Equation~\ref{eq:mocbo_pareto}) and then allocates evaluation budget across them using a cross-scope criterion based on relative hypervolume improvement (Equation~\ref{eq:mocbo_acq}). Within each local problem, standard multi-objective BO components can be used, for example, DGEMO \citep{DGEMO:Lukovic2020}. The main challenge remaining is combinatorial. The number of local subproblems can still grow quickly with dense graphs or many objectives, motivating stronger cross-scope sharing and joint surrogates that exploit causal factorization across objectives.

\begin{formalsummary}{MO-CBO}{eq:mocbo_summary}
\begin{align}
\fslot{assumptions}
  & \mathcal{G}\ \text{known multi-target DAG};\quad
    \mathbf{Y}=(Y_1,\dots,Y_m);\quad
    \boldsymbol{\mu}(X_s,\mathbf{x}_s)=\big(\mathbb{E}[Y_i\mid do(X_s{=}\mathbf{x}_s)]\big)_{i=1}^{m}
    \label{eq:mocbo_assum}\\
\fslot{objective}
  & \mathcal{P}^{\mathrm{c}}(\mathcal{S})=\big\{\boldsymbol{\mu}(X_s,\mathbf{x}_s):
    \nexists\,(X_{s'},\mathbf{x}_{s'})\ \text{with}\
    \boldsymbol{\mu}(X_{s'},\mathbf{x}_{s'})\prec\boldsymbol{\mu}(X_s,\mathbf{x}_s)\big\},
    \quad
    \mathcal{P}^{\mathrm{c}}(\mathcal{S})\subseteq\textstyle\bigcup_{X_s\in\mathcal{S}}\mathcal{P}^{\ell}(X_s)
    \label{eq:mocbo_pareto}\\
\fslot{surrogate}
  & \hat{\mu}_i^{(s)}\sim\mathcal{GP}\big(m_i^{(s)},\,k_i^{(s)}\big)
    \quad\text{independently for each objective } i \text{ and each scope } X_s\in\mathcal{S}
    \label{eq:mocbo_surr}\\
\fslot{acquisition}
  & \mathrm{RHVI}(X_s,B_s)=
    \frac{\mathcal{H}\big(\mathcal{P}^{\ell}(X_s)\cup\boldsymbol{\mu}(X_s,B_s)\big)
         -\mathcal{H}\big(\mathcal{P}^{\ell}(X_s)\big)}
         {\mathcal{H}\big(\mathcal{P}^{\ell}(X_s)\big)},
    \quad
    B^{\star}\in\argmax_{X_s\in\mathcal{S},\,B_s}\mathrm{RHVI}(X_s,B_s)
    \label{eq:mocbo_acq}
\end{align}
\end{formalsummary}
Here $\prec$ denotes Pareto dominance under minimization (weakly smaller in every coordinate of $\boldsymbol{\mu}$, strictly smaller in at least one), $\mathcal{S}$ is the family of possibly-Pareto-optimal minimal intervention sets (the multi-target analogue of POMIS, characterized graphically in \citealp{MOCBOBhatija}), $\mathcal{P}^{\ell}(X_s)$ is the local Pareto front of the subproblem restricted to scope $X_s$, $B_s$ is a candidate evaluation batch, and $\mathcal{H}$ is the hypervolume indicator with respect to a fixed reference point.

\subsection{Intervention Strategy Developments}
\label{subsec:intervention_developments}

Beyond modeling extensions, another major line of work changes what an intervention looks like and how decisions are made during optimization. The common motivation is that treating each intervention as a single black-box query can waste information. These methods modify the BO loop so it can reuse causal structure inside the system, express richer actions such as policies, and stay meaningful under contextual variation or non-stationarity.

Two orthogonal design choices appear in this section and should not be conflated: expanding the \emph{intervention representation} (fCBO, CoCaBO) and changing the \emph{surrogate architecture} (MCBO, ACBO) are independent axes---mechanism-level surrogates can be paired with only hard interventions, and effect-level surrogates with soft interventions. Relatedly, which uncertainty dominates a method is not a design choice but a property \emph{induced} by the system-knowledge assumptions (known versus unknown graph, causal sufficiency, known feasible set); the taxonomy therefore lists system-knowledge assumptions as the primary axis, and the discussion below identifies the induced uncertainty for each method.

Two further clarifications: CoCaBO and ACBO also differ on the \emph{environmental} axis (contextual versus adversarial or non-stationary), so grouping them here reflects shared machinery, not identical assumptions. And a closely related formulation outside the SCM tradition is Bayesian optimization of function networks \citep{astudillo2021functionnetworks}, which optimizes over a known computational DAG with observed intermediate values; it anticipates the mechanism-level, soft-intervention view of MCBO and ACBO, its benchmarks are inherited by our soft-intervention benchmark (Section~\ref{subsec:soft_benchmark}), and its key difference from SCM-based CBO is that function-network inputs are causally independent decision variables rather than variables embedded in a causal system.

\paragraph{Mechanism-level modeling with guarantees (MCBO).}
Model-based Causal Bayesian Optimization (MCBO) models the causal system at the level of structural mechanisms rather than only the intervention-to-outcome effect \citep{Sussex2022MCBO}. This changes data reuse: observations of intermediate variables update upstream mechanisms and therefore tighten uncertainty about downstream outcomes, which can reduce the number of costly interventions needed to identify good actions. The main limitations are practical. Mechanism posteriors must be maintained for multiple equations, uncertainty propagation increases computational cost, and acquisition optimization can be challenging. MCBO also assumes a known graph and a mechanism model class that is expressive enough for the system at hand.

Formally, each mechanism receives an independent GP prior, $\hat{f}_j \sim \mathcal{GP}(m_j, k_j)$, and the outcome of an intervention $\boldsymbol{\theta}$ (a hard clamp or soft mechanism shift) is obtained by propagating through the graph in topological order:
\begin{equation}
V_j = \hat{f}_j\big(\mathrm{Pa}_j(\boldsymbol{\theta}),\, U_j\big) \;\;\text{for each non-intervened node,}
\qquad
\hat{Y}(\boldsymbol{\theta}) = \hat{f}_Y\big(\mathrm{Pa}_Y(\boldsymbol{\theta}),\, U_Y\big),
\label{eq:mcbo_propagation}
\end{equation}
where each parent value $\mathrm{Pa}_j(\boldsymbol{\theta})$ is itself generated by upstream mechanism models under the intervention. The decision rule is optimism-based over a calibrated set of plausible mechanism vectors:
\begin{equation}
\mathcal{F}_t=\Big\{\tilde F=(\tilde f_j)_j \;:\; \tilde f_j\in\mathcal{H}_{k_j},\ \
\|\tilde f_j\|_{k_j}\le B_j,\ \
\big|\tilde f_j(\cdot)-\mu_{j,t-1}(\cdot)\big|\le \beta_t\,\sigma_{j,t-1}(\cdot)\ \ \forall j\Big\},
\label{eq:mcbo_modelset}
\end{equation}
\begin{equation}
\boldsymbol{\theta}_t\in\argmax_{\boldsymbol{\theta}\in\Theta}\;\mathrm{ucb}_t(\boldsymbol{\theta}),
\qquad
\mathrm{ucb}_t(\boldsymbol{\theta})=\max_{\tilde F\in\mathcal{F}_t}\;
\mathbb{E}_{U}\big[\hat{Y}(\boldsymbol{\theta})\mid \tilde F\big],
\label{eq:mcbo_ucb}
\end{equation}
where $(\mu_{j,t-1},\sigma_{j,t-1})$ are the mechanism-GP posteriors, $\beta_t$ is calibrated so that the true mechanism vector lies in $\mathcal{F}_t$ with probability at least $1-\delta$, and the inner maximization is made tractable by the reparameterization $\tilde f_j=\mu_{j,t-1}+\beta_t\,\sigma_{j,t-1}\,\eta_j$ with $\eta_j(\cdot)\in[-1,1]$. This optimism-based rule, applied to the propagated outcome of Equation~\ref{eq:mcbo_propagation}, yields the cumulative-regret guarantees established in \citet{Sussex2022MCBO}.

Compared to effect-level CBO, MCBO has two additional trade-offs that deserve explicit discussion. First, maintaining $|V|$ separate mechanism-level GPs and propagating posterior uncertainty through the full graph at every acquisition step is computationally far more expensive than fitting a single effect-level GP, limiting current implementations to graphs with tens of nodes (see also Section~\ref{subsec:scalability}). Second, mechanism-level models are highly vulnerable to unobserved confounders between intermediate nodes: if a hidden confounder links two internal variables, the GP for the downstream mechanism will be fitted on systematically biased inputs, and this bias propagates through the graph. Effect-level CBO can mitigate the same issue more simply by discarding the biased observational prior and proceeding with an uninformative GP prior.

A related nuance concerns intermediate observations under effect-level models. The claim that effect-level models ``cannot reuse intermediate observations'' is only strictly true for the GP posterior itself: a direct $X_s\to Y$ surrogate does not condition on intermediate node values. However, intermediate observations collected during interventional trials can be appended to the observational dataset and used to re-estimate or refine the observationally-derived prior mean $m_s(\cdot)$ at each step. This indirect reuse does not provide the same per-mechanism uncertainty reduction as MCBO, but it is a practical pathway that existing effect-level implementations could exploit.

\begin{formalsummary}{MCBO}{eq:mcbo_summary}
\begin{align}
\fslot{assumptions}
  & \mathcal{G}\ \text{known DAG; causal sufficiency (independent noises } U_j\text{)};
    \notag\\
  & \quad \text{mechanisms } f_j\in\mathcal{H}_{k_j}\ \text{in the plausible set } \mathcal{F}_t\ \text{of Equation~\ref{eq:mcbo_modelset}}
    \\
\fslot{objective}
  & \max_{\boldsymbol{\theta}\in\Theta}\ \mathbb{E}\big[Y\mid\boldsymbol{\theta}\big]
    \quad\text{(hard clamps or soft mechanism shifts; static environment)}
    \label{eq:mcbo_obj}\\
\fslot{surrogate}
  & \text{mechanism-level GPs propagated through } \mathcal{G}\text{ in topological order, Equation~\ref{eq:mcbo_propagation}}
    \\
\fslot{acquisition}
  & \text{optimistic causal UCB over } \mathcal{F}_t\text{ with regret guarantees, Equation~\ref{eq:mcbo_ucb}}
\end{align}
\end{formalsummary}

\paragraph{Policies as interventions (fCBO).}
Functional Causal Bayesian Optimization (fCBO) expands the action space from selecting a value to selecting a policy that maps the observed context to an action \citep{gultchin2023FCBO}. This is motivated by domains such as personalized treatment, where a single fixed intervention is rarely optimal for all individuals. fCBO makes policy search tractable by defining a kernel over policies using RKHS tools and running Bayesian optimization in this function space with a functional expected improvement rule. The key contribution is that it unifies causal optimization and policy learning within a single sequential procedure, where the surrogate models policy-to-outcome effects and the acquisition balances exploration across qualitatively different decision rules. The main
challenge is representation. Because the policy space is large, empirical performance depends
strongly on the chosen parameterization and on whether the kernel captures a meaningful similarity between policies.

\begin{formalsummary}{fCBO}{eq:fcbo_summary}
\begin{align}
\fslot{assumptions}
  & \mathcal{G}\ \text{known DAG};\quad
    \text{functional interventions } do(X_s{=}\pi(Z)) \text{ replace the mechanism of } X_s;
    \notag\\
  & \quad \pi \text{ in a bounded ball } \Pi_s\subset\mathcal{H}_{\kappa_s}
    \\
\fslot{objective}
  & \pi^{\star}\in\argmin_{\pi\in\Pi_s}\;
    \mathbb{E}\big[\,Y\mid do(X_s=\pi(Z))\,\big]
    \quad\text{jointly with } X_s \text{ over non-redundant scopes}
    \label{eq:fcbo_objective}\\
\fslot{surrogate}
  & g_s(\pi)\sim\mathcal{GP}\big(m_s(\pi),\,k_s(\pi,\pi')\big),
    \qquad
    k_s(\pi,\pi')=k_{\mathrm{RBF}}\big(\|\pi-\pi'\|_{\mathcal{H}_{\kappa_s}}\big)
    \label{eq:fcbo_surrogate}\\
\fslot{acquisition}
  & \mathrm{fEI}_s(\pi)=
    \frac{\mathbb{E}\big[(y^{+}-g_s(\pi))_{+}\big]}{\mathrm{Co}(X_s,\pi)},
    \qquad
    (X_s,\pi)_{t+1}\in\argmax_{X_s,\ \pi\in\Pi_s}\;\mathrm{fEI}_s(\pi)
    \label{eq:fcbo_fei}
\end{align}
\end{formalsummary}
The original formulation optimizes jointly over scope and policy and allows mixed scopes that combine hard components with functional components, $\Pi_s=\mathbb{R}^{|s_{\mathrm{hard}}|}\times\mathcal{B}(\mathcal{H}_{\kappa_s})$; correspondingly, the policy kernel in Equation~\ref{eq:fcbo_surrogate} is in general a composite of a Euclidean distance on hard components and the RKHS distance on functional components \citep{gultchin2023FCBO}.

\paragraph{Context-dependent scope selection (CoCaBO).}
Contextual Causal Bayesian Optimization (CoCaBO) targets settings where the exogenous context affects outcomes and where the best intervention depends on the realized context \citep{ContextualCBO}. The additional difficulty is that one must decide not only which action to take, but also which contextual variables should be used to condition decisions. CoCaBO separates these levels: an outer bandit allocates budget across candidate \emph{mixed policy scopes} $\mathcal{M}_k = (X_{s_k}, C_{s_k})$, each pairing an intervention scope with a context subset, while an inner mixed-variate Bayesian optimization routine such as HEBO \citep{HEBO:2022} optimizes within the selected scope. CoCaBO optimizes a context-dependent objective of the form
\begin{equation}
\max_{\pi: \mathcal{C} \to \mathcal{X}} \;\mathbb{E}_{C}\!\left[\mathbb{E}\!\left[Y \mid do(X_{s_k}=\pi(C_{s_k})),\, C\right]\right],
\label{eq:cocabo}
\end{equation}
where $\pi$ maps observed context values to intervention values within the chosen scope, and the outer expectation is over the context distribution.
This architecture makes explicit that context introduces a discrete model selection problem on top of continuous optimization. The main limitations are derived from the hierarchy. The outer allocation must spend trials to distinguish scopes, and improvements learned within one scope do not necessarily transfer to others.

\begin{formalsummary}{CoCaBO}{eq:cocabo_summary}
\begin{align}
\fslot{assumptions}
  & \mathcal{G}\ \text{known DAG; observed exogenous context } C;
    \notag\\
  & \quad \text{actions are mixed policy scopes } \mathcal{M}_k=(X_{s_k},C_{s_k}) \text{ pruned to possibly-optimal scopes}
    \\
\fslot{objective}
  & \text{contextual policy objective, Equation~\ref{eq:cocabo}}
    \\
\fslot{surrogate}
  & g_k(\mathbf{x},\mathbf{c})\sim
    \mathcal{GP}\big(0,\ k_k\big((\mathbf{x},\mathbf{c}),(\mathbf{x}',\mathbf{c}')\big)\big)
    \quad\text{for each candidate scope } \mathcal{M}_k
    \label{eq:cocabo_gp}\\
\fslot{acquisition}
  & k_t\in\argmax_{k}\;
    \Big[\,\bar{U}_k(t)+\rho_k\big(n_k(t)\big)\big/\sqrt{n_k(t)}\,\Big],
    \quad\text{then mixed-variate BO within } \mathcal{M}_{k_t}
    \label{eq:cocabo_ucb}
\end{align}
\end{formalsummary}
In Equation~\ref{eq:cocabo_ucb}, $n_k(t)$ is the number of rounds allocated to $\mathcal{M}_k$, $\bar{U}_k(t)$ is the running average of the inner optimizer's UCB values under $\mathcal{M}_k$ (the arm statistic averages optimistic scores rather than raw rewards), and $\rho_k$ is an exploration width chosen to obtain $\tilde{\mathcal{O}}(\sqrt{mT})$ cumulative regret over $m$ candidate scopes \citep{ContextualCBO}.

\paragraph{Non-stationarity and adversarial environments (ACBO).}
Adversarial Causal Bayesian Optimization (ACBO) addresses cases where the system response changes over time or is influenced by an adversary that can also intervene in the system \citep{ACBOSussex}.
ACBO combines multiplicative weight updates from adversarial online learning \citep{MWLITTLESTONE:1994,MWFREUND:1997} with causal modeling to construct calibrated optimistic reward estimates via a causal UCB oracle. A central benefit of using the causal graph is that it can reduce regret dependence on the nominal action dimension when rewards factor through causal compositions, improving over non-causal analogs such as GP-MW \citep{Sessa:2019}. The main challenges are computational and statistical. The oracle can be expensive to optimize, and in non-stationary regimes, the method relies on uncertainty sets that must track drift accurately to avoid misleading optimism.

Formally, writing $\boldsymbol{\theta}'_t$ for the adversary's action at round $t$ and $r(\boldsymbol{\theta},\boldsymbol{\theta}')=\mathbb{E}[Y\mid\boldsymbol{\theta},\boldsymbol{\theta}']$ for the expected reward, ACBO targets the adversarial regret
\begin{equation}
R(T)=\max_{\boldsymbol{\theta}\in\Theta}\sum_{t=1}^{T} r(\boldsymbol{\theta},\boldsymbol{\theta}'_t)
\;-\;\sum_{t=1}^{T} r(\boldsymbol{\theta}_t,\boldsymbol{\theta}'_t),
\label{eq:acbo_regret}
\end{equation}
and selects actions from a multiplicative-weights distribution over a discretized action set $\Theta$, fed by optimistic causal-UCB reward estimates:
\begin{equation}
\boldsymbol{\theta}_t\sim w_t,
\qquad
w_{t+1}(\boldsymbol{\theta})\;\propto\;w_t(\boldsymbol{\theta})\,
\exp\!\Big(\tau\,\mathrm{UCB}^{\mathcal{G}}_t(\boldsymbol{\theta},\boldsymbol{\theta}'_t)\Big),
\qquad
\mathrm{UCB}^{\mathcal{G}}_t(\boldsymbol{\theta},\boldsymbol{\theta}')
=\max_{\tilde F\in\mathcal{F}_t}\;\mathbb{E}_{U}\big[Y\mid \tilde F,\boldsymbol{\theta},\boldsymbol{\theta}'\big],
\label{eq:acbo_mw}
\end{equation}
with learning rate $\tau=\sqrt{8\log|\Theta|/T}$ and the plausible mechanism set $\mathcal{F}_t$ of Equation~\ref{eq:mcbo_modelset}, whose mechanisms now take the adversary's action as an additional input; the inner maximization uses the same mechanism-confidence reparameterization as MCBO \citep{ACBOSussex}. No-regret means $R(T)/T\to 0$ for any adversary sequence.

\begin{formalsummary}{ACBO}{eq:acbo_summary}
\begin{align}
\fslot{assumptions}
  & \mathcal{G}\ \text{known DAG; mechanism-level model class and plausible set } \mathcal{F}_t\ \text{(Equation~\ref{eq:mcbo_modelset}),}
    \notag\\
  & \quad \text{extended by the adversary input } \boldsymbol{\theta}'
\\
\fslot{objective}
  & \text{no-regret against the best fixed intervention, Equation~\ref{eq:acbo_regret}}
    \\
\fslot{surrogate}
  & \text{mechanism-level propagation, Equation~\ref{eq:mcbo_propagation}, inputs extended by } (\boldsymbol{\theta},\boldsymbol{\theta}')
    \\
\fslot{acquisition}
  & \text{multiplicative weights with optimistic causal-UCB rewards, Equation~\ref{eq:acbo_mw}}
\end{align}
\end{formalsummary}

\subsection{Unknown causal structure}
\label{subsec:unknown_structure}

A major practical barrier for Causal Bayesian Optimization is that many applications do not provide a trusted causal graph in advance. In this regime, the optimization loop must allocate a limited intervention budget across two coupled goals: improving the outcome and reducing the ambiguity about which causal explanations remain consistent with the data. This introduces an additional uncertainty source, structure uncertainty, on top of uncertainty about causal mechanisms and observation noise.

To keep terminology consistent throughout the paper, we use \emph{unknown graph} for the assumption class in which the learner is not given the causal graph; \emph{graph uncertainty} (used interchangeably with \emph{structure uncertainty}) for the learner's remaining uncertainty over graph structure, however it is represented (a posterior, a confidence set, or a candidate list); and \emph{graph misspecification} for the distinct situation in which the learner is given, and trusts, a single graph that differs from the true one. The methods in this subsection address the first two; graph misspecification is instead probed by the stress tests in Appendix~\ref{appendix:robustness_stress_tests}.
Existing approaches mainly differ in how they represent uncertainty over candidate graphs and in whether they explicitly target the decision-relevant structure, that is, the parts of the graph that materially change which intervention should be selected.

\paragraph{Information-seeking joint learning and optimization (CEO).}
Causal Entropy Optimization, CEO, couples structure learning and optimization using an explicitly information-seeking decision rule \citep{branchini2023CEO}. The method maintains a posterior over candidate graphs and averages interventional predictions across it (Equation~\ref{eq:ceo_surr}). The next intervention is then chosen to maximize the expected information gain about the joint of the optimal outcome value and the graph (Equation~\ref{eq:ceo_acq}), rather than to maximize the predicted reward alone. This focus has an important practical consequence. CEO is not optimized to recover a globally accurate graph. It spends budget on experiments that resolve the structural ambiguities that matter to decide how to act so that it can focus on a decision-relevant subgraph for the target $Y$.

\begin{formalsummary}{CEO}{eq:ceo_summary}
\begin{align}
\fslot{assumptions}
  & \mathcal{G}\ \text{unknown};\quad
    G\in\{\mathcal{G}^{(1)},\dots,\mathcal{G}^{(R)}\},\ \text{prior } P(G);
    \notag\\
  & \quad V_j=f_j\big(\mathrm{Pa}_j^{G}\big)+\epsilon_j,\ \
    f_j\sim\mathcal{GP}(0,k_j),\ \epsilon_j\sim\mathcal{N}(0,\sigma_j^2)
    \label{eq:ceo_assum}\\
\fslot{objective}
  & \text{Equation~\ref{eq:cgo} in the (unknown) true graph;}
    \quad
    P(G\mid\mathcal{D})\,\propto\,p(\mathcal{D}\mid G)\,P(G)
    \label{eq:ceo_posterior}\\
\fslot{surrogate}
  & m_s(\mathbf{x})=\mathbb{E}_{P(G\mid\mathcal{D})}\!\big[\widehat{\mathbb{E}}[Y\mid do(X_s{=}\mathbf{x}),G]\big],
    \notag\\
  & \quad
    \widehat{\mathbb{V}}(\mathbf{x})
    =\mathbb{E}_{P(G\mid\mathcal{D})}\!\big[\widehat{\mathbb{V}}[Y\mid do(X_s{=}\mathbf{x}),G]\big]
    +\mathbb{V}_{P(G\mid\mathcal{D})}\!\big[\widehat{\mathbb{E}}[Y\mid do(X_s{=}\mathbf{x}),G]\big]
    \label{eq:ceo_surr}\\
\fslot{acquisition}
  & \alpha_{\mathrm{CES}}(X_s,\mathbf{x})
    =\tfrac{1}{\mathrm{Co}(X_s,\mathbf{x})}\,
    \mathbb{E}_{p(\mathbf{v}_Y,\,y\,\mid\,do(X_s=\mathbf{x}),\,\mathcal{D})}
    \Big[\mathbb{H}\big[p(y^{\star}\!,G\mid\mathcal{D})\big]
    \notag\\
  & \hspace{11.2em}
        -\mathbb{H}\big[p(y^{\star}\!,G\mid\mathcal{D}\cup\{(\mathbf{x},\mathbf{v}_Y,y)\})\big]\Big]
    \label{eq:ceo_acq}
\end{align}
\end{formalsummary}
The graph posterior in Equation~\ref{eq:ceo_posterior} is available in closed form because the GP marginal likelihoods are tractable. The entropy $\mathbb{H}$ in Equation~\ref{eq:ceo_acq} is over the \emph{joint} of the global optimal value $y^{\star}$ and the graph $G$, so a single experiment is credited both for shrinking structure uncertainty and for localizing the optimum; $\mathbf{v}_Y$ denotes the observed non-target variables of the interventional sample. Because $p(y^{\star}\mid\mathcal{D})$ has no closed form, CEO approximates it by a mixture over which scope hosts the optimum,
\begin{equation}
p(y^{\star}\mid\mathcal{D})\;\approx\;
\sum_{X_s\in\mathrm{ES}} P\big(X_s=X^{\star}\big)\;p\big(y_s^{\star}\mid\mathcal{D}\big),
\label{eq:ceo_mixture}
\end{equation}
with set-level weights maintained by a UCB-style bandit rule over the exploration set $\mathrm{ES}$ of candidate scopes.

The main limitation is computational overhead. Approximating a posterior over graphs typically
requires sampling in a large discrete space, and estimating information gain must be repeated across many candidate interventions (the outer expectation in Equation~\ref{eq:ceo_acq}). As a result, CEO is most practical when the graph is small to medium or when strong structural priors substantially narrow the set of plausible graphs.

\paragraph{Optimism over plausible models (GACBO).}
Graph-Agnostic Causal Bayesian Optimization, GACBO, follows a more algorithmic route that emphasizes scalability \citep{mukherjee2024}. Rather than choosing interventions by entropy reduction, it maintains a confidence set of causal models that remain statistically plausible given the data and selects interventions that could be optimal under at least one model in this set.
Formally, let $\mathcal{C}_t$ denote the confidence set of plausible causal models (graph--mechanism pairs) at round $t$. GACBO selects the next intervention by solving
\begin{equation}
(X_s, \mathbf{x}_s) \in \argmax_{X_s,\mathbf{x}_s}\; \max_{M \in \mathcal{C}_t}\; \hat{f}_M(X_s,\mathbf{x}_s),
\label{eq:gacbo}
\end{equation}
where $\hat{f}_M$ is the predicted interventional objective under model $M$, i.e., $\hat f_M(X_s,\mathbf{x}_s)=\mathbb{E}[Y\mid \tilde F^{\mathcal{G}'},\,do(X_s{=}\mathbf{x}_s)]$ for a graph--mechanism pair $M=(\mathcal{G}',\tilde F^{\mathcal{G}'})$. This ``optimism in the face of uncertainty'' principle selects interventions that look promising under \emph{some} plausible causal explanation, which drives exploration toward experiments that can eliminate incorrect models.
The confidence set itself is constructed from per-node GP posteriors under each candidate graph:
\begin{equation}
\begin{aligned}
\mathcal{C}_t&=\Big\{(\mathcal{G}',\tilde F)\;:\;\mathcal{G}'\in G_t,\ \
\tilde f_j\in\mathcal{H}_{k_j},\ \
\big|\tilde f_j(\cdot)-\mu^{\mathcal{G}'}_{j,t}(\cdot)\big|\le\beta_t\,\sigma^{\mathcal{G}'}_{j,t}(\cdot)\ \ \forall j\Big\},\\
G_t&=\big\{\mathcal{G}'\in G_{t-1}\;:\;\exists\,\tilde F\ \text{satisfying the above bounds under }\mathcal{G}'\big\},
\end{aligned}
\label{eq:gacbo_confset}
\end{equation}
where $(\mu^{\mathcal{G}'}_{j,t},\sigma^{\mathcal{G}'}_{j,t})$ are per-node GP posteriors fitted under the parent sets of $\mathcal{G}'$ and $\beta_t$ is chosen so that the true graph--mechanism pair remains in $\mathcal{C}_t$ with probability at least $1-\delta$; candidate graphs are eliminated once no mechanism vector remains plausible under them \citep{mukherjee2024}.
Compared with CEO, GACBO replaces posterior averaging with best-case evaluation over models consistent with the observations, avoiding explicit entropy computations; this aligns closely with optimistic bandit and Bayesian optimization principles and can make it easier to implement and scale.

The main risk is over-optimism early in the learning process. When the confidence set is wide, the method may conduct trials on interventions that look good only under structural hypotheses that have not yet been ruled out. Performance, therefore, depends on the calibration of the confidence set and on how quickly interventional evidence eliminates incorrect graphs and mechanisms. Recent work extends this direction by exploring alternative uncertainty decompositions, including approaches that emphasize learning exogenous distributions and broader formulations of unknown-graph problems
\citep{ren2025,durand2025}.

\begin{formalsummary}{GACBO}{eq:gacbo_summary}
\begin{align}
\fslot{assumptions}
  & \mathcal{G}\ \text{unknown; causal sufficiency and acyclicity;}
    \notag\\
  & \quad \text{mechanisms in RKHS balls } \mathcal{H}_{k_j}\text{; all nodes observed after each intervention}
\\
\fslot{objective}
  & \text{Equation~\ref{eq:cgo} with the graph unknown}
   \\
\fslot{surrogate}
  & \text{graph--mechanism confidence set } \mathcal{C}_t \text{ and surviving graphs } G_t\text{, Equation~\ref{eq:gacbo_confset}}
  \\
\fslot{acquisition}
  & \text{optimism over plausible models, Equation~\ref{eq:gacbo}}
\end{align}
\end{formalsummary}

\paragraph{Summary.}
Unknown-graph CBO turns intervention design into an adaptive experimental design problem in which structure learning is not a preprocessing step but part of the optimization loop. The CEO and GACBO represent two complementary ways to couple these goals. The CEO prioritizes information gain about the best intervention under a Bayesian posterior, while GACBO prioritizes optimistic improvement under a frequentist-style set of plausible models. The choice between them largely reflects an information-versus-optimism trade-off and the computational budget available for representing and updating uncertainty over graphs.

\begin{table}[t]
\caption{Classification of major CBO methods by key assumptions and modeling choices. Each row lists the graph/system-knowledge assumption, surrogate modeling target, decision strategy, and intervention or environmental feature introduced by the method. Superscripts mark causal-sufficiency assumptions: $^{\ddagger}$ indicates that the method's formulation can accommodate hidden confounders in the graph (e.g., via identification-based adjustment or averaging over structures), whereas $^{\dagger}$ indicates that the method models individual mechanisms or rewards directly and therefore assumes causal sufficiency (fully observed parent--child relations). Unmarked methods follow the common known-graph setting in which causal sufficiency is typically assumed but is not central to the method's design.}
\label{tab:cbo-variants}
\small
\setlength{\tabcolsep}{4pt}
\renewcommand{\arraystretch}{1.15}
\begin{center}

\begin{tabular}{@{} M{0.15\textwidth} | M{0.15\textwidth} M{0.2\textwidth} M{0.2\textwidth} M{0.22\textwidth} @{}}
\toprule
Method & Graph/system assumptions & Surrogate target & Acquisition / strategy & Environment / intervention feature \\
\midrule
CBO (2020)    & Known DAG$^{\ddagger}$ & Effect-level GP for $Y$ with causal prior & Causal EI; cost-aware; observe vs.\ intervene & Scope + value, hard interventions \\
\midrule
DCBO (2021)   & Known DAG$^{\ddagger}$ & Dynamic effect-level GP & Time-aware causal EI & Intervention sequences over time \\
\midrule
MCBO (2023)   & Known DAG$^{\dagger}$ & Mechanism-level GPs, full SCM & UCB / optimism, regret bounds & Hard/soft interventions; mechanism modeling \\
\midrule
cCBO (2023)   & Known DAG$^{\ddagger}$ & Multi-output GP, objective and constraints & Constrained EI, feasibility-weighted & Safety and feasibility constraints \\
\midrule
fCBO (2023)   & Known DAG & GP over policies, RKHS function space & Functional EI & Functional and policy interventions \\
\midrule
CoCaBO (2026) & Known DAG & BO within scope plus bandit over scopes & Two-layer, MAB-UCB then BO within scope & Contextual mixed interventions \\
\midrule
ACBO (2024)   & Known DAG$^{\dagger}$ & Mechanism-level modeling, full SCM & MW-style online learning plus causal UCB oracle & Adversarial and non-stationary environments \\
\midrule
MO-CBO (2025) & Known DAG & GP per objective, local MOBO subproblems & Relative hypervolume improvement & Multi-objective Pareto optimization \\
\midrule
HCBO (2024)   & Known DAG & Multiple effect-level GPs, selected scopes & Normalized UCB-style scope scoring & High-dimensional scope selection \\
\midrule
CEO (2023)    & Unknown DAG$^{\ddagger}$ & Bayesian model average over graphs, mixture & Entropy and information gain & Learns graph structure during optimization \\
\midrule
GACBO (2024)  & Unknown DAG$^{\dagger}$ & Confidence set over graphs and mechanisms & Optimism / UCB over plausible models & Structure discovery via optimistic exploration \\
\midrule
CNEI (2023)   & Known DAG & GP for $Y$ with learned supervised prior & Counter-noise EI & Noise-robust acquisition; learned priors \\
\bottomrule
\end{tabular}
\end{center}
\end{table}

\begin{table}[t]
\setlength{\belowcaptionskip}{6pt}
\caption{Datasets used in empirical evaluations across CBO variants (as reported in the corresponding papers). We add a category column to make the dataset's family structure explicit and use light grid lines to improve readability. The dataset-usage table is restricted to single-objective CBO evaluations.}
\label{tab:dataset_usage}
\centering
\small
\setlength{\tabcolsep}{4pt}
\renewcommand{\arraystretch}{1.05}
\resizebox{\textwidth}{!}{%
\begin{tabular}{c|c|c|c|c|c|c|c|c|c|c|c}
\toprule
Category &  Dataset &
CBO & cCBO & fCBO & MCBO &
HCBO & CEO & GACBO & CNEI &
ACBO & CoCaBO \\
\cline{1-12}

\multirow{5}{*}{\parbox{2.5cm}{\centering Synthetic\\(hard)}} 
& ToyGraph      & $\checkmark$ & $\checkmark$ & $\checkmark$ & $\checkmark$ &  & $\checkmark$ & $\checkmark$ &  &  &  \\
\cline{2-12}
& Synthetic     & $\checkmark$ &  &  &  &  &  &  & $\checkmark$ &  &  \\
\cline{2-12}
& Synthetic-2   &  & $\checkmark$ &  &  &  &  &  &  &  &  \\
\cline{2-12}
& Chain-hard         &  &  & $\checkmark$ &  &  &  &  &  &  &  \\
\cline{1-12}

\multirow{5}{*}{\parbox{2.5cm}{\centering Synthetic\\(soft)}} 
& Ackley        &  &  &  & $\checkmark$ &  &  & $\checkmark$ &  & $\checkmark$ &  \\
\cline{2-12}
& Rosenbrock    &  &  &  & $\checkmark$ &  &  & $\checkmark$ &  & $\checkmark$ &  \\
\cline{2-12}
& Dropwave      &  &  &  & $\checkmark$ &  &  & $\checkmark$ &  & $\checkmark$ &  \\
\cline{2-12}
& Alpine2       &  &  &  & $\checkmark$ &  &  & $\checkmark$ &  & $\checkmark$ &  \\
\cline{2-12}
& Chain-soft         &  &  & $\checkmark$ &  &  &  &  &  &  &  \\
\cline{1-12}

\multirow{5}{*}{\parbox{2.5cm}{\centering Real / fitted\\SCMs\\(hard)}} 
& Ecology       & $\checkmark$ &  &  &  & $\checkmark$ &  &  &  &  &  \\
\cline{2-12}
& Healthcare    & $\checkmark$ & $\checkmark$ & $\checkmark$ & $\checkmark$ & $\checkmark$ & $\checkmark$ &  &  &  & $\checkmark$ \\
\cline{2-12}
& Protein-reconstructed       &  & $\checkmark$ &  &  &  &  &  &  &  &  \\
\cline{2-12}
& Epidemiology  &  &  &  &  &  & $\checkmark$ & $\checkmark$ &  &  &  \\
\bottomrule
\end{tabular}%
}
\end{table}
\subsection{Robustness to noise and priors}
\label{subsec:robustness_noise_priors}

Even with a known graph, practical CBO pipelines face two recurring robustness issues. First, interventional outcomes can be noisy, heteroscedastic, or affected by unmodeled disturbances, which can destabilize improvement-based acquisitions. Second, observationally derived priors may be misspecified, either because identifiability assumptions fail, or because the estimand is hard to estimate accurately from finite data, or because observational correlations leak confounding into a prior mean.

Counter-noise Expected Improvement (CNEI) \citep{Li2023} targets both issues through a noise-aware acquisition and a learned prior construction. CNEI modifies EI-style selection to account for the noise level (Equation~\ref{eq:cnei_acq}) so that the acquisition does not reward an apparent improvement dominated by measurement variance. In parallel, it estimates a surrogate prior mean by fitting predictive models to observational data and using the resulting predictor as a warm-start mean for the GP (Equation~\ref{eq:cnei_prior}). This can improve the early-stage search by giving the surrogate a reasonable global shape before sufficient interventional data accumulate.

\begin{formalsummary}{CNEI}{eq:cnei_summary}
\begin{align}
\fslot{assumptions}
  & \mathcal{G}\ \text{known DAG};\quad
    y_s=g_s(\mathbf{x}_s)+\epsilon,\ \ \epsilon\sim\mathcal{N}(0,\delta^2);\quad
    X_s\in\mathrm{ES}\ \text{(MIS or POMIS of } \mathcal{G}\text{)}
    \label{eq:cnei_assum}\\
\fslot{objective}
  & \text{Equation~\ref{eq:cgo} (baseline CBO problem, unchanged)}
  \\
\fslot{surrogate}
  & g_s\sim\mathcal{GP}(m_s,k_s),\quad
    m_s(\mathbf{x}_s)=\textstyle\sum_{\mathbf{z}}\hat{p}(\mathbf{z})\,\hat{r}(\mathbf{x}_s,\mathbf{z}),\quad
    \hat{r}(\mathbf{x}_s,\mathbf{z})\approx\mathbb{E}[Y\mid X_s{=}\mathbf{x}_s,\,Z{=}\mathbf{z}]
    \ \text{fit on } \mathcal{D}_{\mathrm{obs}}
    \label{eq:cnei_prior}\\
\fslot{acquisition}
  & \alpha_{\mathrm{CNEI}}^{(s)}(\mathbf{x}_s)
    = y_s^{+}-\mathbb{E}\big[\big(g_s(\mathbf{x}_s^{+})-g_s(\mathbf{x}_s)\big)_{+}\,\big|\,\mathcal{D}\big],
    \qquad
    (X_s,\mathbf{x}_s)_{t+1}\in\argmin_{X_s\in\mathrm{ES},\ \mathbf{x}_s}\;
    \alpha_{\mathrm{CNEI}}^{(s)}(\mathbf{x}_s)
    \label{eq:cnei_acq}
\end{align}
\end{formalsummary}
Here $\hat{r}$ is a supervised regressor (support vector regression, ridge regression, or a random forest) plugged into the adjustment formula of Equation~\ref{eq:cnei_prior}, $\mathbf{x}_s^{+}$ is the within-scope incumbent, and $y_s^{+}$ its observed value, which places per-scope improvements on a common scale for cross-scope comparison. The ``counter-noise'' element is that the improvement in Equation~\ref{eq:cnei_acq} is defined on the \emph{latent} objective $g_s$ under the joint posterior of $\big(g_s(\mathbf{x}_s),g_s(\mathbf{x}_s^{+})\big)$ rather than on the noisy observation $y_s^{+}=g_s(\mathbf{x}_s^{+})+\epsilon$; the variance $\mathbb{V}[g_s(\mathbf{x}_s)-g_s(\mathbf{x}_s^{+})]=k_s(\mathbf{x}_s,\mathbf{x}_s)+k_s(\mathbf{x}_s^{+},\mathbf{x}_s^{+})-2k_s(\mathbf{x}_s,\mathbf{x}_s^{+})$ retains the covariance with the incumbent, so apparent improvements driven by measurement variance are discounted.\footnote{We state the noise-robust acquisition in the survey's minimization convention as the expected best value after the candidate intervention; the source alternates sign conventions, and we normalize the improvement variance to the standard variance-of-difference form. The generalization from the source's single adjustment variable to an adjustment set $Z$ in Equation~\ref{eq:cnei_prior} is a survey-level paraphrase.}

A critical distinction concerns the causal status of this prior. CNEI's learned prior is an adjustment-type plug-in (Equation~\ref{eq:cnei_prior}) whose inner regression $\hat{r}$ is a purely supervised, correlational fit; its causal validity therefore rests entirely on the assumed adjustment set. When that set is invalid or hidden confounding is present, the prior degrades to a correlational predictor $\hat{\mathbb{E}}[Y \mid X]$ and can encode spurious associations that a valid causal adjustment would have removed. CNEI therefore trades causal soundness for practical robustness, relying on the noise-aware acquisition to prevent the potentially biased prior from dominating the search as interventional evidence accumulates.

More broadly, observationally learned priors should serve as weak guidance with appropriately inflated uncertainty rather than as a substitute for interventional evidence, motivating robust CBO procedures that explicitly control how strongly observational information can influence acquisition decisions under noise, misspecification, and possible hidden confounding.

\subsection{Comparison tables and unifying observations}
\label{subsec:tables_observations}

This subsection consolidates the cross-method synthesis in one place: rather than closing each preceding subsection with its own summary, we compare the method families once, here, along the design axes introduced at the start of Section~\ref{Methods}. Tables~\ref{tab:cbo-variants} and~\ref{tab:dataset_usage} summarize the field from two
complementary angles. Table~\ref{tab:cbo-variants} organizes methods by key design dimensions:
assumptions about the causal graph, the surrogate modeling target (effect-level, mechanism-level, or graph-uncertainty) and the acquisition or decision strategy. Table~\ref{tab:dataset_usage} maps empirical coverage and reveals substantial fragmentation: beyond a small core of shared benchmarks (ToyGraph, Healthcare), many datasets appear in only one or two papers, which complicates cross-paper performance comparison.

We additionally include an identification-strategy audit table.
\begin{table}[ht]
\caption{Identification-strategy audit for benchmarked pipelines. ``Method-level identification'' describes the causal-effect identification strategy assumed or implemented by the released method. ``Benchmark effect source'' describes how the benchmark wrapper supplies or approximates effects when re-running the released pipeline.}
\label{tab:ident-audit}
\centering
\scriptsize
\setlength{\tabcolsep}{4pt}
\renewcommand{\arraystretch}{1.25}
\begin{tabular}{l p{2.1cm} p{5.4cm} p{4.6cm}}
\toprule
Method & Surrogate level & Method-level identification & Benchmark effect source \\
\midrule
BO     & Target-level GP & None; non-causal baseline & Direct black-box evaluations \\
CoCaBO & Mixed-input BO  & No explicit ID estimand in our wrapper; graph/scope-based policy selection & Benchmark SCM / task oracle evaluations \\
CBO    & Effect-level GP & Observational adjustment when identifiable; not full ID in released code & Benchmark wrapper uses SCM/oracle effect estimates where available \\
cCBO   & Effect-level / constrained GP & Observational adjustment with constraints; not full ID in released code & Benchmark wrapper uses SCM/oracle effect estimates where available \\
DCBO   & Dynamic effect-level / temporal GP & Temporal adjustment under known dynamic-graph assumptions; not full ID & Benchmark dynamic SCM / oracle evaluations \\
HCBO   & Scope-level / high-dimensional CBO & Graph-based scope selection; not full ID & Benchmark SCM / SEM evaluations \\
CEO    & Graph-uncertain / oracle SEM search & Structure-learning over candidate graphs; not full ID over arbitrary latent structures & Candidate graph / SEM evaluations \\
MCBO   & Mechanism-level GP & None required for effect identification; mechanisms are modeled directly & Mechanism-level simulation / forward propagation \\
\bottomrule
\end{tabular}
\end{table}
In particular, none of the benchmarked released pipelines exposes a general front-door or full-ID implementation in the benchmark path used here; where observational adjustment is used, it is method-specific and should not be interpreted as a complete ID-algorithm implementation.

Together, these tables highlight two orthogonal axes of variation. The \emph{modeling axis} ranges from effect-level surrogates (CBO, DCBO, cCBO, HCBO, CNEI) through mechanism-level models with uncertainty propagation (MCBO, ACBO) to explicit graph-uncertainty methods (CEO, GACBO). The \emph{action axis} ranges from scope-and-value selection (CBO, DCBO, HCBO) through policy- or context-dependent decisions (fCBO, CoCaBO) to adversarial interaction (ACBO). The identification-strategy audit (Table~\ref{tab:ident-audit}) additionally clarifies that the benchmark compares released pipelines as complete systems, not isolated acquisition functions. Recognizing these interactions is important for practitioners choosing a method and for researchers identifying gaps in the design space.

\subsection{Connections to Adjacent Fields}
\label{subsec:adjacent_fields}

CBO does not exist in isolation. Several established research areas share overlapping goals, tools, or assumptions. Understanding these connections clarifies what CBO inherits, where it innovates, and what cross-pollination opportunities remain.

\paragraph{Causal abstractions and multi-scale decision making.}
Causal abstraction studies how high-level causal models relate to lower-level mechanistic models through abstraction maps. This perspective is directly relevant to CBO whenever interventions can be chosen at multiple resolutions. Causally Abstracted Multi-Armed Bandits (CAMAB) \citep{zennaro2024camab} formalizes decision making across causally linked abstraction levels, while AT-UCB \citep{dyer2025atucb} shows that exploiting abstraction hierarchies can reduce cumulative regret in complex simulators, including epidemiological models---precisely the kind of complex simulation setting where CBO's aspiration for real-world deployment meets computational tractability limits. For CBO, this suggests a multi-scale analogue of scope reduction: rather than only pruning intervention scopes within one graph, the optimizer may transfer or prune decisions across graphs at different granularities. A first proposal in this direction for CBO itself is MFACBO \citep{zeitler2025mfacbo}, which connects causal abstraction levels to multi-fidelity Bayesian optimization; we note that this proposal is a workshop contribution and not yet peer-reviewed. We therefore treat causal abstraction learning as distinct from causal representation learning from raw observations, and as a promising route for CBO over latent, coarse-grained, or hierarchical causal variables.

\paragraph{Causal bandits.}
The causal bandit literature \citep{lattimore2016causal,lu2021causal,nair2021budgeted} also selects interventions using causal structure, but typically assumes discrete action sets, parametric reward models, and focuses on cumulative regret in stationary environments. CBO extends this setting in three directions: (i)~continuous intervention domains requiring nonparametric surrogates, (ii)~mixed discrete--continuous action spaces (scope value $+$ and (iii)~more robust uncertainty modeling through GP. Conversely, causal bandits offer tighter finite-sample regret bounds and principled treatments of partial graph knowledge that CBO methods have only begun to adopt (e.g., MCBO's regret analysis draws on GP-bandit theory). A promising direction is to unify the two frameworks by developing CBO methods with bandit-style theoretical guaranties for continuous domains.

\paragraph{Bayesian experimental design.}
The Bayesian optimal experimental design (BOED) \citep{chaloner1995bayesian,foster2021deep} selects experiments to maximize information about the model parameters. CEO \citep{branchini2023CEO} can be viewed as a BOED method applied to intervention selection: choose experiments to maximize the information gained about the identity of the best intervention. The key distinction is that standard BOED targets parameter estimation without an optimization objective, whereas CEO jointly optimizes information and reward. This connection suggests that advances in amortized BOED \citep{foster2021deep} could accelerate information-seeking CBO methods.

\paragraph{Active causal discovery.}
Active structure learning \citep{spirtes2000causation,chickering2002optimal,Eberhardt_Scheines_2007} selects interventions to learn the causal graph itself, while CBO selects interventions to optimize an outcome. Unknown-graph CBO methods (CEO, GACBO) bridge these goals by learning graph structure as a means to better optimization. The distinction is that active discovery aims for global graph accuracy, whereas CBO needs only \emph{decision-relevant} structure, the parts of the graph that change which intervention is optimal. This focus on decision-relevant learning is a truly novel aspect of unknown-graph CBO.

\paragraph{Safe and constrained optimization.}
Safe Bayesian optimization \citep{sui2015safe,berkenkamp2019bayesian} maintains feasibility during optimization, which is closely related to cCBO's \citep{aglietti2023cCBO} constrained causal optimization. The causal contribution is that safety constraints can be analyzed through the graph: a scope that cannot affect a constraint variable need not model constraint feasibility, reducing the multi-output modeling burden. Future work could combine safe BO's theoretical safety guaranties with CBO's causal constraint analysis.

\paragraph{Policy search and reinforcement learning.}
Functional CBO (fCBO) \citep{gultchin2023FCBO} and contextual CBO (CoCaBO) \citep{ContextualCBO} can be viewed as policy search methods operating within an explicit causal model. Unlike model-free reinforcement learning, these methods exploit the known graph to decompose the policy-to-outcome mapping and to construct informative priors from observational data. The main limitation relative to RL is the assumption of a single-step (or few-step) decision problem; extending CBO to sequential multi-stage intervention policies remains an open challenge.

\section{Experiments}
\label{sec:experiments}

Although this paper is primarily a survey, we include a benchmark study to support two of the
survey's main claims. First, existing CBO papers are often evaluated in heterogeneous codebases and on only partially overlapping datasets, making cross-paper comparison difficult. Second, commonly reported efficiency metrics can be sensitive to how optimization trajectories are summarized, especially when improvements occur late in the evaluation budget. Our benchmark is therefore designed to separate algorithm execution from evaluation: methods can run in their native implementations, but their outputs are exported into a common best-so-far trajectory format and re-scored with identical metric code.

\subsection{CBO benchmark and evaluation protocol}
\label{subsec:benchmark_release}

We release a CBO benchmark at
\url{https://github.com/chenfeng-huang/CBO-Benchmark-TMLR-2026}. The benchmark is a standardized repackaging and extension of inherited CBO/function-network tasks, rather than an entirely new collection of datasets. It includes eight curated hard-intervention scenarios and a separate set of soft-intervention benchmarks, with dataset provenance and code-authoritative domains documented in Appendix~\ref{appendix:dataset_summary}. It provides standardized implementations of GAP and PA-GAP. Because existing CBO methods often rely on incompatible software environments, the benchmark does not force all algorithms into a single execution stack. Instead, each method is run in its native or vendored implementation, and its trajectory is exported and re-scored using the same evaluation pipeline. This design enables reproducible comparisons across heterogeneous implementations and makes it easy to add new methods, datasets, or metrics.

Because the implemented methods target different problem formulations, we do not pool them into a single cross-method leaderboard: competitive statements are restricted to methods solving the same formulation, and cross-method aggregation appears only as descriptive rank-reliability tables in Appendix~\ref{appendix:ranking_statistical_testing}. The benchmark uses native or vendored implementations with released defaults, common trajectory export, and common scoring; it is therefore a reproducibility harness rather than a fully controlled reimplementation study. Appendix~\ref{appendix:protocol} documents method settings and wrapper choices.

To make the formulation labels of Table~\ref{tab:formulation-labels} usable as fixed reference points rather than user-configured settings, each task in the repository ships with a frozen configuration: the intervention domains, optimization direction, observational sample sizes, constraint functions and thresholds for the constrained tasks, the policy space for Chain-soft, and the reference optimum used for scoring. A new method can therefore be evaluated on a given formulation without any task-design decisions being left to the user.

\paragraph{Execution protocol.}
For each data set, method, and trial limit, we run 20 random seeds. A seed fixes the observational sample, the initial interventional design, and the stochastic components of the optimizer. All methods are evaluated under the same trial limits, intervention domains, optimization direction, and scoring code. An interventional evaluation counts as one trial. Observational data used to construct priors, learn graphs, or initialize surrogates are fixed before the optimization loop and are not counted as interventional trials unless otherwise stated. Each method is run with its released default hyperparameters unless explicitly stated in Appendix~\ref{appendix:protocol}; we do not tune hyperparameters separately for each data set.

\paragraph{Trajectory export and scoring.}
The output of each method is converted to a common trajectory format with two required fields: trial number and best-so-far objective value in the natural optimization direction of the task. For minimization tasks, this is the minimum value observed up to the trial $t$; for maximization tasks, it is the maximum value observed up to the trial $t$. The trajectory includes an initial value followed by interventional trials. Thus, if the budget is $B$ interventions, the exported trajectory contains $B+1$ entries: one initial best value and $B$ post-intervention best-so-far values. GAP and PA-GAP are then computed from these exported trajectories using the same evaluation script for all methods.

\paragraph{Reference optimum.}
Both GAP and PA-GAP require a reference optimum $y^*$. In our benchmark, $y^*$ is computed offline with oracle access to the benchmark SCM, function generator, or precomputed interventional data.
This value is used only for scoring and is never provided to any optimizer. For low-dimensional or discrete intervention domains, we enumerate or densely evaluate admissible intervention scopes and values. For continuous domains, we use high-budget offline search with multi-start optimization or stored oracle intervention sets. The exact offline budgets, grids, restarts, and dataset-specific procedures are reported in the Appendix~\ref{appendix:reference_optima}.

\paragraph{Failure handling.}
When a method proposes an infeasible or out-of-domain intervention, we follow the behavior of the released implementation. If the implementation projects the proposals back to the admissible domain, the projected intervention is evaluated. Otherwise, the trial is marked as a failed proposal, and the trajectory records the current best-so-far value. Runs with numerical errors, missing outputs, or non-finite objective values are logged and are not silently removed. Runtime is not used as a primary metric because implementations differ substantially in language, hardware support, and dependency stacks.

\paragraph{Cost model.}
Unless otherwise stated, each intervention has a unit cost. This choice isolates the efficiency of the sample but does not evaluate the full cost-aware behavior of the CBO acquisition functions. In real applications, the cost of the intervention can depend on the target variable, the number of variables manipulated, the magnitude of the intervention, safety constraints, or whether the action is observational or experimental. We therefore report the current results as unit-cost benchmarks and treat realistic intervention-cost modeling as a separate evaluation dimension.

\subsection{Performance metrics}
\label{subsec:metrics}

We evaluated each run using two numerical efficiency metrics, GAP and PA-GAP, and one visualization metric based on the best-so-far objective trajectory. All scalar metrics are computed from the same exported trajectories described above.

We emphasize that GAP and PA-GAP are auxiliary normalized efficiency summaries that complement, rather than replace, standard BO metrics. Best-so-far trajectories are simple-regret-equivalent (see the trajectory-visualization paragraph below), and normalized average-reward trajectories, the analogue of cumulative-regret reporting, are provided in Appendix~\ref{appendix:average_reward_visualization}. The reason for additionally reporting GAP and PA-GAP is aggregation: raw simple regret is not comparable across datasets with different objective scales, initializations, and optimization directions, whereas GAP and PA-GAP normalize improvement relative to the initial value and reference optimum, enabling the cross-dataset rank summaries used in this benchmark. Readers interested only in per-dataset behavior can rely on the regret-equivalent trajectory plots.

\paragraph{GAP.}
GAP was introduced in DCBO to quantify the efficiency of optimization by combining the final improvement with the speed with which the best value is discovered \citep{aglietti2021}. Let $R_t$ denote the best improvement ratio so far in the trial $t$, computed in the natural optimization direction of the task. For maximization tasks,
\begin{equation}
R_t
=
\min\left\{
\frac{y(\mathbf{x}^*_t)-y(\mathbf{x}_{\mathrm{init}})}
     {y^*-y(\mathbf{x}_{\mathrm{init}})},
1
\right\},
\end{equation}
and for minimization tasks,
\begin{equation}
R_t
=
\min\left\{
\frac{y(\mathbf{x}_{\mathrm{init}})-y(\mathbf{x}^*_t)}
     {y(\mathbf{x}_{\mathrm{init}})-y^*},
1
\right\}.
\end{equation}
Here, $\mathbf{x}^*_t$ denotes the best point so far in the trial $t$. The clipping prevents numerical artifacts when an observed value exceeds the offline reference optimum. Let $R_T$ be the final improvement ratio and let $t^*$ be the first trial in which $R_T$ is achieved. If no improvement occurs, we set $t^*=T$. GAP is then
\begin{equation}
\text{GAP}
=
\left[
\underbrace{R_T}_{\text{Final improvement}}
+
\underbrace{\frac{T-t^*}{T}}_{\text{Discovery efficiency}}
\right]
\Bigg/
\underbrace{\left(1+\frac{T-1}{T}\right)}_{\text{Normalization}}.
\label{eq:gap}
\end{equation}
Under this convention, GAP lies in $[0,1]$, with larger values indicating larger final improvement and earlier discovery.

\paragraph{Path-Aware GAP (PA-GAP).}
GAP depends only on the final best value and the first trial at which it is reached. This can
under-value runs that improve late or make steady progress over the full budget. To better reflect the full optimization trajectory, we use Path-Aware GAP:
\begin{equation}
\text{PA-GAP}
=
\frac{1}{T}\sum_{t=1}^{T}
\left(
\underbrace{R_t}_{\text{Improvement ratio}}
\cdot
\underbrace{\frac{T-(t-1)}{T}}_{\text{Efficiency weight}}
\right).
\label{eq:pa_gap}
\end{equation}
PA-GAP multiplies the best-so-far improvement by a time-dependent efficiency weight, so later
improvements are down-weighted but not discarded, and averages these weighted improvements over the full trajectory. Since $R_t\in[0,1]$, raw PA-GAP lies in
\[
\left[0,\frac{T+1}{2T}\right].
\]
The upper bound is attained when the reference optimum is found in the first trial and remains the best value so far thereafter. Thus, PA-GAP is a raw trajectory-weighted efficiency score rather than a unit-normalized metric. Larger values indicate faster and stronger trajectory-level improvement, and raw PA-GAP values should be compared directly only under the same trial budget.
Appendix~\ref{appendix:pagap_motivation} gives a counterexample where GAP ranks a worse final trajectory above a better late-improving one, while PA-GAP resolves the ordering.

\paragraph{Trajectory visualization.}
In addition to scalar metrics, we plot best-so-far objective trajectories over trials. For a fixed dataset and optimization direction, the best-so-far trajectory is equivalent to a simple-regret trajectory up to an additive constant and, for maximization tasks, a sign reversal. We therefore use best-so-far plots as simple-regret-equivalent visualizations in the original objective scale. We also report normalized average-reward trajectories in Appendix~\ref{appendix:average_reward_visualization} as a diagnostic complement to GAP and PA-GAP.

\subsection{Datasets and task conventions}
\label{subsec:datasets}

Table~\ref{tab:dataset_usage} summarizes the datasets used in the CBO literature. Most evaluations rely on SCM-generated data, where observational samples are drawn from the unintervened SCM and interventional samples are generated by applying hard or soft interventions to the SCM mechanisms.
The optimization direction is data-dependent. For hard-intervention SCM benchmarks, we follow the standard CGO/CBO convention and treat the target as a minimization objective unless explicitly stated otherwise. ToyGraph, Synthetic, Synthetic-2, Chain-hard, Healthcare, Protein-reconstructed, and Epidemiology are minimization tasks, while Ecology is a maximization task. For the soft-intervention benchmarks for the function-network, Ackley, Rosenbrock, Dropwave, and Alpine2 are treated as reward-maximization tasks following the convention used in the BO, MCBO, and ACBO function-network. Chain-soft is a minimization task because the objective is to reduce the target variable $Y$.

\subsubsection{Hard-intervention SCM benchmarks}

Hard-intervention benchmarks evaluate optimization when intervened variables are clamped to fixed values. They include small synthetic SCMs, a reconstructed protein-signaling SCM, and
real-world-inspired SCMs from ecology, healthcare, and epidemiology. The benchmark also provides optional high-dimensional generators, but the main cross-method comparison focuses on the eight curated scenarios for which multiple implementations can be evaluated under the same protocol.

The eight scenarios are as follows; per-dataset structural equations, causal graphs, intervention domains, and provenance are documented in Appendix~\ref{appendix:dataset_summary} (Tables~\ref{tab:dataset_provenance} and~\ref{tab:dataset_consolidated}) and the per-dataset appendix subsections. \textbf{ToyGraph} \citep{aglietti2020CBO} is a minimal nonlinear three-node chain $X\to Z\to Y$ with manipulable $X,Z$, serving as a sanity check (Appendix~\ref{appendix:ToyGraph}). \textbf{Synthetic} \citep{aglietti2020CBO} has seven observed variables plus two latent confounders, with manipulable $B,D,E$, combining nonlinearity, latent structure, and multiple intervention options (Appendix~\ref{appendix:Synthetic}). \textbf{Synthetic-2} \citep{aglietti2023cCBO} is a lightweight nonlinear three-node chain with Gaussian exogenous noise (Appendix~\ref{appendix:Synthetic-2}). \textbf{Chain-hard} \citep{gultchin2023FCBO} has four observed variables with manipulable $W,Z$ clamped in $[-1,1]$ and a non-manipulable context $X$ whose interaction with $Z$ drives the target (Appendix~\ref{appendix:Chain_hard}). \textbf{Ecology} \citep{aglietti2020CBO,courtney2017environmental} is a Bermuda-reef ecosystem model in which net ecosystem calcification is maximized under environmental interventions, using the SCM and data released in the public CBO repository (Appendix~\ref{appendix:Ecology}). \textbf{Protein-reconstructed} is based on the protein-signaling data of \citet{Sachs2005} and the cCBO graph \citep{aglietti2023cCBO}: Erk is minimized by perturbing Mek, PKC, PKA, and Akt under biological constraints on PKC and PKA; because the fitted SCM used in the original cCBO experiments is unreleased, this task is a transparent reconstruction fitted on 852 observational samples rather than an exact reproduction (Appendix~\ref{appendix:Protein}). \textbf{Healthcare} \citep{aglietti2020CBO,b17,ferro2015use} minimizes PSA through statin and aspirin interventions, with the common continuous relaxation of the binary treatment variables. \textbf{Epidemiology} \citep{branchini2023CEO,havercroft2012simulating} minimizes HIV viral load in a modified epidemiology SCM with code-authoritative intervention domains on $L$ and $B$ (Appendices~\ref{appendix:Healthcare} and~\ref{appendix:Epidemiology}).

\subsubsection{Soft-intervention function-network benchmarks}

Soft-intervention benchmarks evaluate settings where interventions modify structural mechanisms rather than simply clamping variables to fixed values. Ackley, Rosenbrock, Dropwave, and Alpine2 are classical global-optimization functions represented as function-network BO benchmarks introduced by \citet{astudillo2021functionnetworks} and adopted in MCBO and ACBO \citep{Sussex2022MCBO,ACBOSussex}. We preserve the inherited task domains after any internal normalization used by wrappers: Ackley and Rosenbrock use $[-2,2]$, Dropwave uses $[-5.12,5.12]$, and Alpine2 uses $[0,10]$. Thus, the optimizer may operate on normalized coordinates internally, but the benchmark task domains are not changed. In these tasks, the output node is treated as a reward to be maximized, even when this corresponds to a sign-flipped version of a standard minimization benchmark.

We explicitly distinguish these \emph{computational-graph benchmarks} from \emph{SCM benchmarks}. In the function networks, the manipulable input variables are mutually independent, orthogonal search dimensions: there are no causal relationships among the inputs themselves, so an intervention on one input cannot sever incoming edges of, or cascade into, another manipulable variable. These tasks therefore probe optimization over a known computational composition, but they do not test confounding or cascading effects among manipulable variables, which are core challenges of causal optimization tested by SCM benchmarks such as Healthcare or Ecology. Chain-soft is the only soft-intervention SCM task in the current repository, and results on the two families are labeled and interpreted separately (Table~\ref{tab:formulation-labels}).

The five scenarios are as follows; equations, graphs, and visualizations are given in Appendices~\ref{appendix:Ackley}--\ref{appendix:Chain_soft}. \textbf{Ackley} ($D=6$) decomposes the highly multimodal Ackley objective into a squared-magnitude component and a cosine component combined at the output node. \textbf{Rosenbrock} ($D\in\{3,5,7\}$) is a chain-structured accumulation of adjacent-variable-pair terms with a narrow curved valley. \textbf{Dropwave} ($D=2$) computes the radial quantity $\sqrt{x_1^2+x_2^2}$ at an intermediate node before the oscillatory output transformation. \textbf{Alpine2} ($D=6$) is a multiplicative chain of per-variable Alpine2 terms, separable and multimodal. \textbf{Chain-soft} \citep{gultchin2023FCBO} extends Chain-hard by letting the mechanism of $Z$ be modified through a context-dependent policy $\pi_0(X)$ while $W$ remains under hard interventions in $[-1,1]$; because the target depends on the interaction between $Z$ and $X$, the optimal intervention on $Z$ varies with $X$, making it a benchmark for context-adaptive intervention strategies (Appendix~\ref{appendix:Chain_soft}).

\subsection{Hard-intervention benchmark}
\label{subsec:hard_benchmark}

To contextualize the benchmark results, Table~\ref{tab:formulation-labels} defines the formulation labels used throughout this section. These labels clarify which methods are solving matched optimization problems and which are included as stress tests or descriptive comparisons.

\begin{table}[t]
\centering
\small
\caption{Formulation labels used to contextualize benchmark results. Several labels are not independently rankable because they contain only one method or one dataset; therefore the statistical rank analysis is reported by budget rather than as separate formulation-track leaderboards.}
\label{tab:formulation-labels}
\begin{tabular}{p{0.20\linewidth}p{0.34\linewidth}p{0.36\linewidth}}
\toprule
Label & What it identifies & How it is used \\
\midrule
Unconstrained hard SCM & Static hard-intervention SCM tasks without explicit feasibility constraints & Descriptive rank summaries by budget (Appendix~\ref{appendix:ranking_statistical_testing}); comparisons are qualified because some methods are stress-test or graph-uncertain variants \\
Constrained hard SCM & Tasks where cCBO-style feasibility constraints are part of the intended formulation & Descriptive only: no fully matched constraint-adapted baseline is available across all datasets \\
Contextual / policy & Settings where scope or intervention choice is policy/context dependent & Descriptive only: CoCaBO is the only currently integrated method in this formulation \\
Soft computational graph & Ackley, Rosenbrock, Dropwave, and Alpine2 function-network tasks & Descriptive rank summaries by budget (Appendix~\ref{appendix:ranking_statistical_testing}); explicitly not treated as SCM benchmarks with confounding or cascading manipulable variables \\
Soft SCM / policy & Chain-soft & Descriptive only: single soft-SCM dataset in the current repository \\
\bottomrule
\end{tabular}
\end{table}

\begin{figure}[t]
    \centering
    \includegraphics[width=\linewidth]{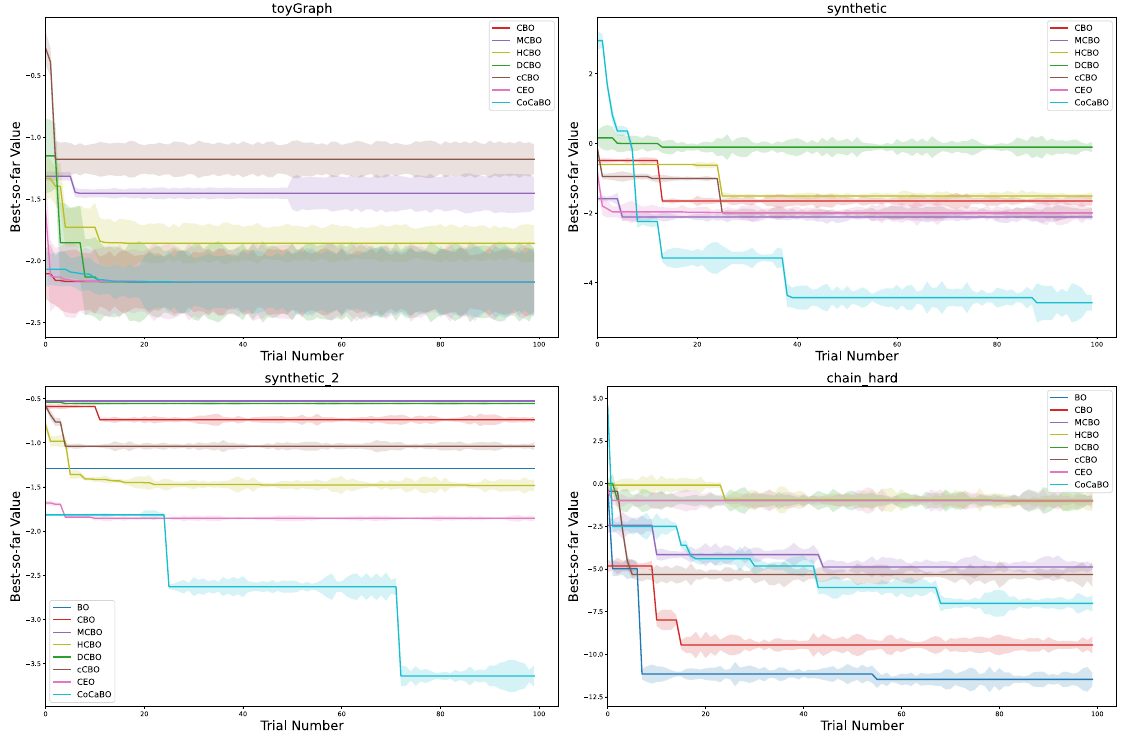}
    \caption{Best-so-far trajectories on four synthetic hard-intervention datasets: ToyGraph, Synthetic, Synthetic-2, and Chain-hard. Solid lines show the mean across 20 seeds and shaded regions show the standard deviation. Lower values indicate better objective values for these minimization tasks.}
    \label{fig:synthetic_group}
\end{figure}

\begin{figure}[t]
    \centering
    \includegraphics[width=\linewidth]{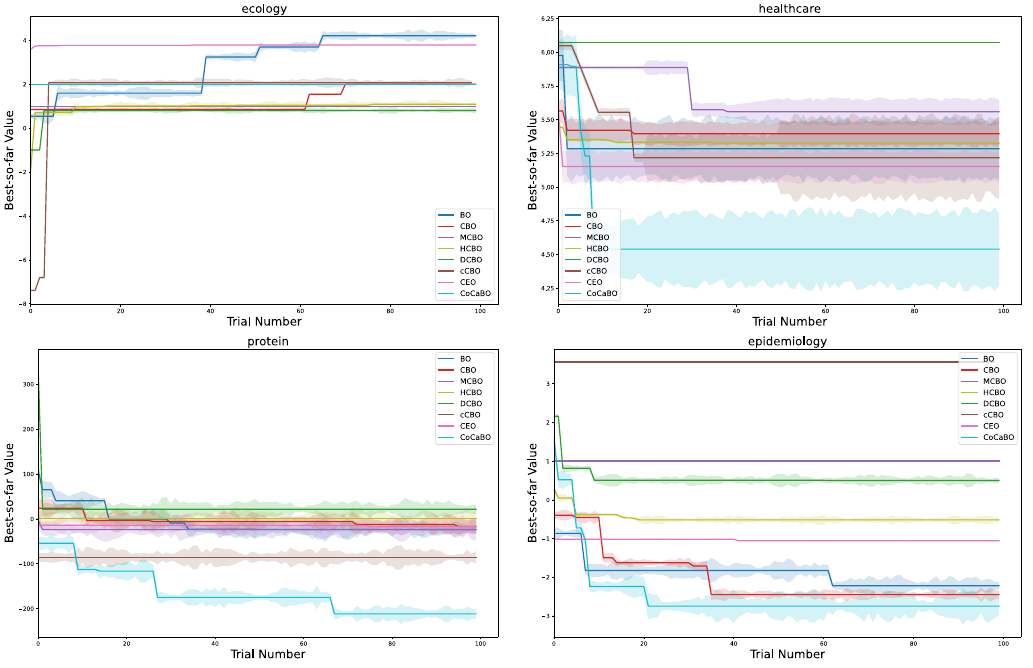}
    \caption{Best-so-far trajectories on four real or fitted-SCM hard-intervention datasets: Ecology, Protein-reconstructed, Healthcare, and Epidemiology. Solid lines show the mean across 20 seeds and shaded regions show the standard deviation. Ecology is a maximization task, while Protein-reconstructed, Healthcare, and Epidemiology are minimization tasks.}
    \label{fig:real_group}
\end{figure}
\begin{table}[t]
\setlength{\belowcaptionskip}{6pt}
\caption{Hard-intervention performance across different trial limits and datasets. Each value is the average $\pm$ standard error across 20 random seeds and different initializations of the observational and interventional data. Higher is better for both GAP and PA-GAP. Best mean values in each row are in \textbf{bold}.}
\label{tab:trials_metrics_models}
\centering
\scriptsize
\setlength{\tabcolsep}{2pt}
\renewcommand{\arraystretch}{1.0}
\resizebox{\textwidth}{!}{%
\begin{tabular}{c|c|c|cccccccc}
\toprule
Trial limit & Metric & Dataset & BO & CBO & cCBO & DCBO & MCBO & HCBO & CoCaBO & CEO \\
\midrule
\multirow{16}{*}{100} 
& \multirow{8}{*}{GAP} 
& ToyGraph & 0.346$\pm$.046 & 0.810$\pm$.044 & 0.730$\pm$.023 & \textbf{0.949$\pm$.048} & 0.547$\pm$.032 & 0.727$\pm$.030 & 0.625$\pm$.054 & 0.654$\pm$.040 \\
& & Synthetic & 0.401$\pm$.038 & 0.658$\pm$.023 & 0.691$\pm$.028 & 0.482$\pm$.007 & 0.652$\pm$.030 & 0.559$\pm$.024 & \textbf{0.757$\pm$.049} & 0.565$\pm$.038 \\
& & Synthetic-2 & 0.000$\pm$.000 & 0.544$\pm$.021 & 0.549$\pm$.017 & 0.485$\pm$.003 & 0.000$\pm$.000 & 0.174$\pm$.016 & \textbf{0.562$\pm$.045} & 0.378$\pm$.038 \\
& & Chain-hard & 0.721$\pm$.032 & 0.688$\pm$.014 & \textbf{0.734$\pm$.025} & 0.542$\pm$.027 & 0.502$\pm$.046 & 0.135$\pm$.024 & 0.205$\pm$.018 & 0.502$\pm$.034 \\
& & Ecology & \textbf{0.799$\pm$.041} & 0.511$\pm$.017 & 0.697$\pm$.027 & 0.637$\pm$.026 & 0.000$\pm$.000 & 0.334$\pm$.030 & 0.000$\pm$.000 & 0.336$\pm$.032 \\
& & Protein-reconstructed & 0.751$\pm$.030 & 0.449$\pm$.043 & 0.000$\pm$.000 & 0.000$\pm$.000 & \textbf{0.878$\pm$.064} & 0.412$\pm$.032 & 0.664$\pm$.052 & 0.275$\pm$.021 \\
& & Healthcare & 0.310$\pm$.020 & 0.577$\pm$.024 & 0.842$\pm$.030 & 0.487$\pm$.001 & 0.510$\pm$.023 & 0.335$\pm$.007 & \textbf{0.964$\pm$.044} & 0.696$\pm$.066 \\
& & Epidemiology & 0.634$\pm$.049 & 0.844$\pm$.029 & 0.000$\pm$.000 & 0.375$\pm$.028 & 0.000$\pm$.000 & 0.557$\pm$.024 & \textbf{0.898$\pm$.059} & 0.308$\pm$.021 \\
\cline{2-11}
& \multirow{8}{*}{PA-GAP} 
& ToyGraph & 0.224$\pm$.031 & 0.069$\pm$.009 & 0.235$\pm$.030 & \textbf{0.337$\pm$.043} & 0.071$\pm$.010 & 0.285$\pm$.036 & 0.087$\pm$.012 & 0.295$\pm$.037 \\
& & Synthetic & 0.198$\pm$.035 & 0.172$\pm$.028 & 0.184$\pm$.027 & 0.036$\pm$.004 & 0.162$\pm$.023 & 0.107$\pm$.020 & \textbf{0.411$\pm$.003} & 0.358$\pm$.045 \\
& & Synthetic-2 & 0.000$\pm$.000 & 0.055$\pm$.008 & 0.065$\pm$.008 & 0.002$\pm$.000 & 0.000$\pm$.000 & 0.099$\pm$.012 & \textbf{0.129$\pm$.021} & 0.035$\pm$.005 \\
& & Chain-hard & \textbf{0.495$\pm$.002} & 0.280$\pm$.042 & 0.244$\pm$.032 & 0.049$\pm$.006 & 0.111$\pm$.015 & 0.026$\pm$.005 & 0.038$\pm$.004 & 0.002$\pm$.000 \\
& & Ecology & \textbf{0.303$\pm$.039} & 0.117$\pm$.006 & 0.204$\pm$.027 & 0.144$\pm$.019 & 0.000$\pm$.000 & 0.210$\pm$.026 & 0.000$\pm$.000 & 0.016$\pm$.002 \\
& & Protein-reconstructed & \textbf{0.407$\pm$.044} & 0.254$\pm$.036 & 0.000$\pm$.000 & 0.000$\pm$.000 & 0.045$\pm$.006 & 0.000$\pm$.000 & 0.240$\pm$.030 & 0.014$\pm$.004 \\
& & Healthcare & 0.125$\pm$.014 & 0.175$\pm$.020 & 0.170$\pm$.022 & 0.000$\pm$.000 & 0.100$\pm$.021 & 0.010$\pm$.001 & \textbf{0.329$\pm$.045} & 0.200$\pm$.025 \\
& & Epidemiology & 0.299$\pm$.038 & 0.346$\pm$.021 & 0.000$\pm$.000 & 0.175$\pm$.022 & 0.000$\pm$.000 & 0.138$\pm$.017 & \textbf{0.392$\pm$.047} & 0.005$\pm$.001 \\
\midrule
\multirow{16}{*}{50} 
& \multirow{8}{*}{GAP} 
& ToyGraph & 0.296$\pm$.043 & 0.676$\pm$.045 & 0.723$\pm$.022 & \textbf{0.896$\pm$.041} & 0.513$\pm$.024 & 0.643$\pm$.030 & 0.629$\pm$.048 & 0.690$\pm$.036 \\
& & Synthetic & 0.531$\pm$.018 & 0.597$\pm$.024 & 0.565$\pm$.023 & 0.422$\pm$.003 & 0.630$\pm$.022 & 0.432$\pm$.019 & \textbf{0.793$\pm$.056} & 0.372$\pm$.034 \\
& & Synthetic-2 & 0.000$\pm$.000 & 0.459$\pm$.021 & \textbf{0.531$\pm$.014} & 0.467$\pm$.007 & 0.000$\pm$.000 & 0.160$\pm$.020 & 0.438$\pm$.023 & 0.213$\pm$.043 \\
& & Chain-hard & \textbf{0.919$\pm$.018} & 0.462$\pm$.009 & 0.712$\pm$.023 & 0.535$\pm$.024 & 0.215$\pm$.027 & 0.302$\pm$.016 & 0.106$\pm$.016 & 0.505$\pm$.032 \\
& & Ecology & \textbf{0.798$\pm$.037} & 0.363$\pm$.015 & 0.680$\pm$.023 & 0.624$\pm$.021 & 0.000$\pm$.000 & 0.246$\pm$.016 & 0.000$\pm$.000 & 0.151$\pm$.018 \\
& & Protein-reconstructed & 0.656$\pm$.027 & 0.553$\pm$.038 & 0.000$\pm$.000 & 0.000$\pm$.000 & \textbf{0.752$\pm$.041} & 0.320$\pm$.017 & 0.731$\pm$.041 & 0.377$\pm$.025 \\
& & Healthcare & 0.366$\pm$.019 & 0.598$\pm$.019 & 0.758$\pm$.027 & 0.474$\pm$.000 & 0.315$\pm$.019 & 0.155$\pm$.002 & \textbf{0.927$\pm$.036} & 0.694$\pm$.046 \\
& & Epidemiology & 0.750$\pm$.043 & 0.644$\pm$.024 & 0.000$\pm$.000 & 0.596$\pm$.021 & 0.000$\pm$.000 & 0.461$\pm$.018 & \textbf{0.793$\pm$.041} & 0.096$\pm$.001 \\
\cline{2-11}
& \multirow{8}{*}{PA-GAP} 
& ToyGraph & 0.177$\pm$.029 & 0.068$\pm$.008 & 0.233$\pm$.030 & 0.316$\pm$.040 & 0.063$\pm$.010 & 0.262$\pm$.033 & \textbf{0.357$\pm$.050} & 0.296$\pm$.037 \\
& & Synthetic & 0.130$\pm$.021 & 0.131$\pm$.025 & 0.191$\pm$.017 & 0.033$\pm$.004 & 0.150$\pm$.022 & 0.050$\pm$.013 & 0.349$\pm$.040 & \textbf{0.357$\pm$.032} \\
& & Synthetic-2 & 0.000$\pm$.000 & 0.043$\pm$.008 & 0.063$\pm$.008 & 0.002$\pm$.000 & 0.000$\pm$.000 & \textbf{0.091$\pm$.011} & 0.051$\pm$.014 & 0.033$\pm$.004 \\
& & Chain-hard & \textbf{0.500$\pm$.002} & 0.178$\pm$.030 & 0.233$\pm$.031 & 0.049$\pm$.006 & 0.078$\pm$.013 & 0.012$\pm$.003 & 0.034$\pm$.004 & 0.002$\pm$.000 \\
& & Ecology & \textbf{0.306$\pm$.039} & 0.085$\pm$.013 & 0.194$\pm$.027 & 0.140$\pm$.019 & 0.000$\pm$.000 & 0.208$\pm$.025 & 0.000$\pm$.000 & 0.015$\pm$.002 \\
& & Protein-reconstructed & \textbf{0.339$\pm$.031} & 0.192$\pm$.034 & 0.000$\pm$.000 & 0.000$\pm$.000 & 0.045$\pm$.006 & 0.000$\pm$.000 & 0.157$\pm$.022 & 0.003$\pm$.001 \\
& & Healthcare & 0.120$\pm$.014 & 0.131$\pm$.018 & 0.263$\pm$.038 & 0.000$\pm$.000 & 0.034$\pm$.011 & 0.009$\pm$.001 & \textbf{0.295$\pm$.042} & 0.202$\pm$.025 \\
& & Epidemiology & 0.247$\pm$.039 & 0.222$\pm$.020 & 0.000$\pm$.000 & 0.169$\pm$.021 & 0.000$\pm$.000 & 0.127$\pm$.015 & \textbf{0.354$\pm$.041} & 0.000$\pm$.000 \\
\midrule
\multirow{16}{*}{20} 
& \multirow{8}{*}{GAP} 
& ToyGraph & 0.383$\pm$.022 & 0.510$\pm$.044 & 0.702$\pm$.023 & \textbf{0.729$\pm$.038} & 0.406$\pm$.020 & 0.372$\pm$.022 & 0.539$\pm$.043 & 0.508$\pm$.028 \\
& & Synthetic & 0.239$\pm$.014 & 0.392$\pm$.020 & 0.367$\pm$.022 & 0.231$\pm$.002 & 0.558$\pm$.016 & 0.166$\pm$.002 & \textbf{0.756$\pm$.052} & 0.390$\pm$.028 \\
& & Synthetic-2 & 0.000$\pm$.000 & 0.285$\pm$.017 & \textbf{0.474$\pm$.018} & 0.408$\pm$.002 & 0.000$\pm$.000 & 0.186$\pm$.017 & 0.000$\pm$.000 & 0.282$\pm$.023 \\
& & Chain-hard & \textbf{0.819$\pm$.014} & 0.136$\pm$.003 & 0.641$\pm$.022 & 0.510$\pm$.018 & 0.360$\pm$.023 & 0.000$\pm$.000 & 0.065$\pm$.011 & 0.513$\pm$.028 \\
& & Ecology & \textbf{0.795$\pm$.039} & 0.208$\pm$.003 & 0.625$\pm$.020 & 0.585$\pm$.016 & 0.000$\pm$.000 & 0.271$\pm$.012 & 0.000$\pm$.000 & 0.340$\pm$.043 \\
& & Protein-reconstructed & 0.503$\pm$.022 & 0.508$\pm$.034 & 0.000$\pm$.000 & 0.000$\pm$.000 & 0.540$\pm$.036 & 0.028$\pm$.016 & \textbf{0.648$\pm$.041} & 0.168$\pm$.016 \\
& & Healthcare & 0.604$\pm$.015 & 0.392$\pm$.021 & 0.490$\pm$.019 & 0.432$\pm$.000 & 0.000$\pm$.000 & 0.172$\pm$.001 & \textbf{0.810$\pm$.017} & 0.689$\pm$.036 \\
& & Epidemiology & 0.646$\pm$.030 & 0.452$\pm$.022 & 0.000$\pm$.000 & 0.457$\pm$.017 & 0.000$\pm$.000 & 0.154$\pm$.015 & \textbf{0.791$\pm$.028} & 0.000$\pm$.000 \\
\cline{2-11}
& \multirow{8}{*}{PA-GAP} 
& ToyGraph & 0.088$\pm$.021 & \textbf{0.392$\pm$.051} & 0.227$\pm$.029 & 0.260$\pm$.034 & 0.040$\pm$.008 & 0.200$\pm$.040 & 0.190$\pm$.029 & 0.300$\pm$.036 \\
& & Synthetic & 0.026$\pm$.008 & 0.034$\pm$.012 & 0.147$\pm$.017 & 0.024$\pm$.003 & 0.116$\pm$.021 & 0.000$\pm$.000 & 0.334$\pm$.035 & \textbf{0.361$\pm$.042} \\
& & Synthetic-2 & 0.000$\pm$.000 & 0.016$\pm$.005 & 0.057$\pm$.007 & 0.002$\pm$.000 & 0.000$\pm$.000 & \textbf{0.074$\pm$.008} & 0.000$\pm$.000 & 0.027$\pm$.004 \\
& & Chain-hard & \textbf{0.357$\pm$.035} & 0.032$\pm$.007 & 0.200$\pm$.029 & 0.047$\pm$.006 & 0.034$\pm$.009 & 0.000$\pm$.000 & 0.031$\pm$.004 & 0.002$\pm$.000 \\
& & Ecology & \textbf{0.316$\pm$.039} & 0.030$\pm$.008 & 0.166$\pm$.025 & 0.126$\pm$.019 & 0.000$\pm$.000 & 0.205$\pm$.024 & 0.000$\pm$.000 & 0.015$\pm$.002 \\
& & Protein-reconstructed & 0.250$\pm$.024 & 0.070$\pm$.020 & 0.000$\pm$.000 & \textbf{0.290$\pm$.036} & 0.047$\pm$.006 & 0.000$\pm$.000 & 0.182$\pm$.045 & 0.001$\pm$.000 \\
& & Healthcare & 0.121$\pm$.014 & 0.034$\pm$.018 & 0.143$\pm$.007 & 0.000$\pm$.000 & 0.000$\pm$.000 & 0.008$\pm$.001 & \textbf{0.268$\pm$.044} & 0.208$\pm$.025 \\
& & Epidemiology & 0.158$\pm$.033 & 0.078$\pm$.021 & 0.000$\pm$.000 & 0.152$\pm$.019 & 0.000$\pm$.000 & 0.102$\pm$.011 & \textbf{0.331$\pm$.034} & 0.000$\pm$.000 \\
\bottomrule
\end{tabular}%
}
\end{table}

\paragraph{Compared methods.}
We compare a non-causal baseline, BO, with seven CBO-family methods: CBO, cCBO, DCBO, MCBO,
HCBO, CoCaBO, and CEO
\citep{aglietti2020CBO,aglietti2023cCBO,aglietti2021,Sussex2022MCBO,
Wu2024HighDimensionalCBO,ContextualCBO,branchini2023CEO}. We label comparisons by how closely the benchmark matches the method's intended setting. CBO and CoCaBO are matched known-graph hard-intervention methods for the static SCM tasks. cCBO is included as a controlled constrained-CBO variant; on unconstrained tasks it should be interpreted as a modeling comparison rather than a pure constraint comparison. In particular, cCBO does not reduce to CBO when constraints are inactive: the released cCBO pipeline uses a different surrogate construction and exploration-set handling than the released CBO pipeline, so cCBO-versus-CBO differences on unconstrained tasks reflect these modeling and implementation differences rather than an effect of constraint handling. DCBO is designed for time-evolving SCMs, so its use on static hard-intervention tasks is a transfer/stress-test comparison. MCBO is a mechanism-level method and is most directly matched to function-network/mechanism-modeling settings. CEO is an unknown-graph reference method and spends part of its budget resolving structure uncertainty, whereas known-graph methods receive the benchmark graph. BO is a graph-free GP-EI baseline and is used as a robustness reference rather than a causal competitor. We exclude fCBO, CNEI, and GACBO because we did not obtain compatible public implementations of the benchmark protocol. We exclude ACBO and MO-CBO from this hard-intervention benchmark because their released implementations or task settings are not aligned with the single-objective hard-intervention protocol considered here.

\paragraph{Selected datasets.}
We evaluated eight hard-intervention datasets: ToyGraph, Synthetic, Synthetic-2, Chain-hard, Ecology, Protein-reconstructed, Healthcare, and Epidemiology. These datasets cover small synthetic SCMs, real or fitted-SCM benchmarks, and settings that are unevenly shared across prior CBO papers. Protein-reconstructed is included as a transparent reconstruction of the cCBO Protein benchmark, rather than as an exact reproduction of the unreleased fitted SCM used in the original cCBO experiments.

\paragraph{Results summary.}
The hard-intervention benchmark shows that no method dominates uniformly across datasets, budgets, and metrics within the benchmark-specified regime. In the main tables, the environment SCM, intervention domains, and reference optima are fixed by the benchmark, and graph-using methods are evaluated with the benchmark graph unless they are explicitly unknown-graph methods. Under this protocol, the descriptive rankings vary across datasets, budgets, and metrics, and the rank-reliability analysis in Appendix~\ref{appendix:ranking_statistical_testing} identifies no statistically separable uniformly best method: CoCaBO attains the lowest descriptive average rank at $T=100$ and $T=50$ and BO at $T=20$, but the Nemenyi critical difference exceeds the rank spread at every budget. We interpret CoCaBO's strong cells (Synthetic variants, Healthcare, Epidemiology) as conditional patterns on datasets where mixed scope selection and contextualized search align with the benchmark structure, rather than as evidence that CoCaBO is intrinsically strongest across CBO settings; DCBO's strength on ToyGraph and MCBO's on Protein-reconstructed are likewise dataset-conditional. BO remains competitive on Chain-hard, Ecology, and Healthcare, where early improvements earn favorable trajectory-level (PA-GAP) scores even when BO is not the best final (GAP) performer. These results show that causal structure can improve optimization, but its benefit depends on how the method's assumptions match the dataset and on whether performance is measured by final discovery or by the full optimization path.

\paragraph{Trajectory-level behavior.}
Figures~\ref{fig:synthetic_group} and~\ref{fig:real_group} complement the scalar table by showing how the methods reach their final scores. In synthetic datasets, many methods improve through a few large drops followed by long plateaus, suggesting that much of the gain comes from identifying a useful intervention scope rather than from gradual local refinement; this plateau structure is one source of GAP/PA-GAP disagreement (Section~\ref{subsec:metrics}). In Synthetic, CoCaBO shows strong path-level progress, consistent with its high scores under both metrics. On ToyGraph, DCBO is strong under GAP, while PA-GAP shows that early progress with CBO and HCBO-style can also matter. In Chain-hard, BO is highly competitive, indicating that a flexible black-box optimizer can perform well when the effective intervention landscape is simple or when the causal prior does not provide a decisive advantage.

The real or fitted-SCM datasets show even stronger metric dependence. Some methods remain nearly flat on particular datasets, which explains the zero or near-zero values in Table~\ref{tab:trials_metrics_models}; Appendix~\ref{appendix:protocol} defines these scores and reports the per-case diagnosis of all zero-score entries, including the cCBO Protein-reconstructed and MCBO small-budget cases. Ecology and Healthcare also show that strong performance can arise through different paths: some methods improve early and then plateau, while others discover a better intervention only after many trials.

\paragraph{Comparability across methods.}
The comparisons above are reported at the per-dataset level because the compared pipelines do not all solve the same optimization problem: cCBO solves a strictly harder constrained formulation, CoCaBO conditions its search on observed contextual information, DCBO runs here as a static-task stress test outside its temporal design setting, CEO spends part of its budget resolving graph uncertainty, and BO is a graph-free reference baseline. A single pooled ranking over these formulations would penalize methods for solving harder problems and reward methods that exploit additional information, so we restrict statements of algorithmic superiority to methods solving the same formulation (Table~\ref{tab:formulation-labels}). Readers who want a cross-method overview can consult Appendix~\ref{appendix:ranking_statistical_testing_hard}, which reports budget-stratified average ranks with bootstrap confidence intervals and Friedman/Nemenyi tests as purely \emph{descriptive} summaries: no pair of methods is statistically separable by average rank at any budget (the rank spread stays below the Nemenyi critical difference throughout, and the conservative per-dataset Friedman test never rejects), subject to the non-independence caveats stated there.

\subsection{Soft-intervention benchmark}
\label{subsec:soft_benchmark}

For soft interventions, we compare BO, MCBO, and ACBO in Ackley, Rosenbrock, Dropwave, Alpine2 and Chain-soft \citep{Sussex2022MCBO,ACBOSussex}. BO serves as a non-causal baseline, while MCBO and ACBO represent mechanism-level and adversarial CBO-style approaches. Because ACBO was designed for adversarial or non-stationary settings, its use on these stationary function-network tasks is interpreted as a stress-test/reference result rather than a claim of intrinsic dominance in its native setting. fCBO also supports policy interventions, but it is excluded from the comparison because we did not obtain a compatible reference implementation \citep{gultchin2023FCBO}.

\paragraph{Results summary.}
Because the soft benchmark combines two formulations: four computational-graph/function-network tasks and one SCM-based policy task, we report per-dataset comparisons rather than a pooled ordering. ACBO is especially strong on Ackley and Chain-soft, where it obtains the best GAP across budgets and often achieves strong PA-GAP values, indicating fast and sustained progress. A plausible explanation for why an adversarially designed method performs well on these stationary tasks is that, when the environment does not change, ACBO's CBO-MW procedure effectively reduces to optimistic causal exploration: the causal UCB oracle exploits the known function-network structure much as MCBO does, while the multiplicative-weights layer over a discretized action set concentrates quickly on high-reward actions once rewards are stable, since no adversary forces re-exploration. The discretized action set may additionally act as a coarse global search grid that is well suited to highly multimodal landscapes such as Ackley, where continuous acquisition optimization can stall in local optima. We present this as a mechanism-level hypothesis consistent with the trajectories in Figure~\ref{fig:soft_group}, not as a measured causal explanation. BO remains a highly competitive baseline on Rosenbrock and Alpine2. In these landscapes, standard black-box optimization often reaches strong final values and also receives favorable PA-GAP scores when it improves early and then maintains its best-so-far value. MCBO is most competitive on Dropwave, where it achieves the best GAP and PA-GAP at 50 and 20 trials; this pattern again reflects the metric distinction of Section~\ref{subsec:metrics}. Overall, the soft-intervention results reinforce the main conclusion from the hard-intervention benchmark: causal modeling can improve optimization, but its advantage depends on how well the method's assumptions match the intervention type, objective landscape, and evaluation metric.

\begin{figure}[t]
\centering
\includegraphics[width=\linewidth]{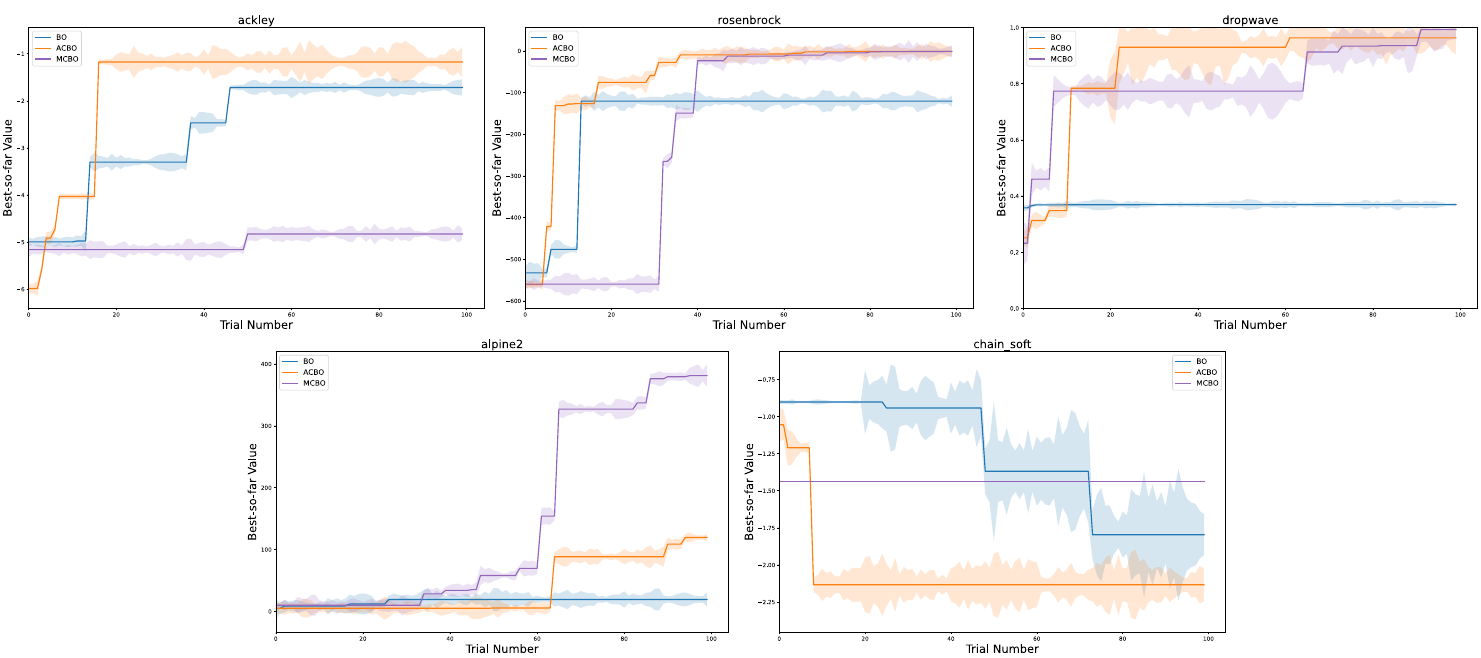} 
\caption{Best-so-far trajectories on five soft-intervention datasets: Ackley, Rosenbrock, Dropwave, Alpine2, and Chain-soft. Solid lines show the mean over 20 random seeds and shaded regions denote the standard deviation across runs.}
\label{fig:soft_group}
\end{figure}

\begin{table}[t]
\setlength{\belowcaptionskip}{6pt}
\caption{Soft-intervention performance across different trial limits and datasets. Each value is the average $\pm$ standard error across 20 random seeds and initializations of the observational and interventional data. Higher is better; best mean values per row are in \textbf{bold}.}
\label{tab:trials_metrics_mcbo_acbo}
\centering
\scriptsize
\setlength{\tabcolsep}{2pt}
\renewcommand{\arraystretch}{1.0}
\begin{tabular}{c|c|c|ccc}
\toprule
Trial limit & Metric & Dataset & BO & MCBO & ACBO \\
\midrule
\multirow{10}{*}{100}
& \multirow{5}{*}{GAP}
& Ackley      & 0.328$\pm$.022 & 0.310$\pm$.003 & \textbf{0.922$\pm$.082} \\
& & Rosenbrock & \textbf{0.825$\pm$.055} & 0.507$\pm$.041 & 0.623$\pm$.047 \\
& & Dropwave   & 0.024$\pm$.005 & 0.558$\pm$.065 & \textbf{0.690$\pm$.094} \\
& & Alpine2    & \textbf{0.536$\pm$.046} & 0.522$\pm$.052 & 0.178$\pm$.014 \\
& & Chain-soft & 0.634$\pm$.028 & 0.512$\pm$.022 & \textbf{0.693$\pm$.052} \\
\cline{2-6}
& \multirow{5}{*}{PA-GAP}
& Ackley      & 0.033$\pm$.004 & 0.012$\pm$.003 & \textbf{0.403$\pm$.006} \\
& & Rosenbrock & 0.308$\pm$.046 & 0.217$\pm$.004 & \textbf{0.412$\pm$.007} \\
& & Dropwave   & 0.008$\pm$.001 & 0.355$\pm$.041 & \textbf{0.379$\pm$.005} \\
& & Alpine2    & \textbf{0.120$\pm$.014} & 0.076$\pm$.019 & 0.015$\pm$.005 \\
& & Chain-soft & \textbf{0.290$\pm$.011} & 0.000$\pm$.001 & 0.205$\pm$.028 \\
\midrule
\multirow{10}{*}{50}
& \multirow{5}{*}{GAP}
& Ackley      & 0.090$\pm$.005 & 0.072$\pm$.002 & \textbf{0.843$\pm$.071} \\
& & Rosenbrock & \textbf{0.761$\pm$.042} & 0.515$\pm$.049 & 0.631$\pm$.052 \\
& & Dropwave   & 0.235$\pm$.024 & \textbf{0.803$\pm$.085} & 0.318$\pm$.040 \\
& & Alpine2    & \textbf{0.403$\pm$.034} & 0.085$\pm$.007 & 0.175$\pm$.009 \\
& & Chain-soft & 0.515$\pm$.023 & 0.000$\pm$.000 & \textbf{0.655$\pm$.033} \\
\cline{2-6}
& \multirow{5}{*}{PA-GAP}
& Ackley      & 0.018$\pm$.003 & 0.000$\pm$.000 & \textbf{0.320$\pm$.042} \\
& & Rosenbrock & 0.239$\pm$.041 & 0.053$\pm$.017 & \textbf{0.344$\pm$.047} \\
& & Dropwave   & 0.008$\pm$.001 & \textbf{0.318$\pm$.040} & 0.171$\pm$.024 \\
& & Alpine2    & \textbf{0.085$\pm$.008} & 0.003$\pm$.001 & 0.000$\pm$.000 \\
& & Chain-soft & 0.135$\pm$.038 & 0.000$\pm$.000 & \textbf{0.181$\pm$.026} \\
\midrule
\multirow{10}{*}{20}
& \multirow{5}{*}{GAP}
& Ackley      & 0.166$\pm$.013 & 0.000$\pm$.000 & \textbf{0.593$\pm$.049} \\
& & Rosenbrock & \textbf{0.559$\pm$.035} & 0.000$\pm$.000 & 0.499$\pm$.047 \\
& & Dropwave   & 0.252$\pm$.023 & \textbf{0.701$\pm$.067} & 0.293$\pm$.041 \\
& & Alpine2    & \textbf{0.142$\pm$.013} & 0.000$\pm$.000 & 0.000$\pm$.000 \\
& & Chain-soft & 0.000$\pm$.000 & 0.000$\pm$.000 & \textbf{0.534$\pm$.024} \\
\cline{2-6}
& \multirow{5}{*}{PA-GAP}
& Ackley      & 0.003$\pm$.001 & 0.000$\pm$.000 & \textbf{0.151$\pm$.017} \\
& & Rosenbrock & 0.082$\pm$.018 & 0.000$\pm$.000 & \textbf{0.216$\pm$.038} \\
& & Dropwave   & 0.007$\pm$.001 & \textbf{0.253$\pm$.030} & 0.130$\pm$.021 \\
& & Alpine2    & \textbf{0.049$\pm$.006} & 0.000$\pm$.000 & 0.000$\pm$.000 \\
& & Chain-soft & 0.000$\pm$.000 & 0.000$\pm$.000 & \textbf{0.116$\pm$.019} \\
\bottomrule
\end{tabular}
\end{table}

\paragraph{Trajectory-level behavior.}
Figure~\ref{fig:soft_group} shows that the three methods behave differently in objective
landscapes. ACBO improves quickly on Ackley and Chain-soft, which explains its strong scalar scores on these datasets. BO remains highly competitive in Rosenbrock and Alpine2, where its best-so-far trajectory often improves early enough to receive strong PA-GAP values as well as strong GAP values.
Dropwave shows a complementary pattern: MCBO and ACBO both improve substantially, but MCBO is
stronger under tighter budgets after applying the same PA-GAP clipping rule used throughout the benchmark. This correction caps the improvement ratio at one when a noisy observed best value exceeds the recorded expected-value reference optimum, ensuring that GAP and PA-GAP remain within their stated ranges. These trajectory patterns show why per-dataset conclusions differ across the two metrics: PA-GAP gives additional credit to methods that improve early and maintain progress, which is why BO remains competitive with ACBO on several tasks despite being a non-causal baseline.

\paragraph{Comparability across methods.}
As in the hard-intervention benchmark, statements of algorithmic superiority are restricted to matched formulations: Chain-soft supports only within-dataset ordering, and ACBO runs outside its adversarial design setting. The per-dataset results in Table~\ref{tab:trials_metrics_mcbo_acbo} and Figure~\ref{fig:soft_group} are therefore the primary evidence, while the budget-stratified descriptive rank-reliability tables in Appendix~\ref{appendix:ranking_statistical_testing_soft} provide a cross-method overview in which no pair of methods is statistically separable at any budget (here even the omnibus Friedman test does not reject).

\subsection{Cross-cutting analysis}
\label{subsec:cross_cutting}

The hard and soft-intervention benchmarks together span 78 dataset--budget-metric settings. Rather than summarizing each setting individually, we distill the main patterns into four cross-sectional observations.

\paragraph{When does causal structure help?}
Causal methods provide the clearest advantage when the graph enables effective scope reduction or when observational priors are well-calibrated. In the hard-intervention setting, CoCaBO and cCBO benefit from structured scope selection on datasets with multiple intervenable variables and clear causal pathways (Synthetic, Healthcare, Epidemiology). In the soft-intervention setting, ACBO benefits from mechanism-level modeling on datasets with complex intermediate structure (Ackley, Chain-soft). Conversely, causal structure provides little or no advantage on datasets where the effective intervention landscape is simple (Chain-hard) or where causal priors are poorly calibrated (Ecology under some methods). BO remains competitive precisely in these settings, suggesting that the value of causal modeling is conditional on the alignment between the method's structural assumptions and the problem's actual causal complexity.

\paragraph{Budget sensitivity.}
Rankings are moderately sensitive to budget. At tight budgets ($T=20$), methods that rely on observational priors or scope reduction (CBO, cCBO, DCBO) tend to perform relatively better because their prior information is most valuable when interventional data are scarce. At larger budgets ($T=100$), methods with richer exploration mechanisms (CoCaBO, ACBO) can overtake prior-driven methods by accumulating sufficient interventional evidence to overcome initial advantages from prior information.

\paragraph{Metric sensitivity.}
GAP and PA-GAP can produce different rankings for the same dataset and budget. GAP rewards the final best value and its discovery time, so methods that make a single large improvement early are favored. PA-GAP rewards steady trajectory-level progress, so methods that gradually improve throughout the budget can rank well even without the best final value. This discrepancy is most pronounced when comparing methods that plateau early (e.g., DCBO on Synthetic-2) versus methods that improve late (e.g., HCBO on Synthetic-2). We recommend that future CBO evaluations report both metrics, as well as trajectory plots, to avoid misleading conclusions from any single scalar summary.

\paragraph{Threats to validity.}

Several factors limit the generalizability of our benchmark findings. First, the benchmark is a native-wrapper reproducibility harness rather than a fully controlled reimplementation of every method; implementation details, released defaults, and dependency choices may affect performance. Second, the benchmark uses unit-cost interventions, which does not capture realistic cost heterogeneity. Third, the main hard-intervention tables assume benchmark-specified graphs and causal sufficiency for most graph-using methods. To probe this limitation rather than only state it, Appendix~\ref{appendix:robustness_stress_tests} reports controlled learner-side graph-misspecification stress tests under edge addition, edge deletion, and edge reversal, and a companion omitted-variable stress test for hidden confounding. In all of these diagnostics, the true environment SCM and reference optimum remain fixed, while only the learner's causal information is perturbed. Fourth, we use released default hyperparameters for all methods; per-dataset tuning could change relative rankings, as illustrated by the MCBO $\beta$ sensitivity analysis. Relatedly, the benchmark does not control for differences in identification strategy across pipelines (Table~\ref{tab:ident-audit}): the released methods differ in whether and how they construct observational priors, none exposes a general front-door or full-ID implementation in the benchmark path used here, and Monte Carlo approximation settings for identified estimands are not uniformly exposed. On datasets with latent variables (most notably Synthetic, whose confounders $U_1,U_2$ make the backdoor criterion inapplicable for some effects), performance differences therefore cannot be cleanly attributed to acquisition or decision design alone; they may partly reflect differences in identification completeness and prior-construction quality. Fifth, GAP and PA-GAP are useful efficiency summaries but should be interpreted alongside best-so-far trajectories, simple-regret-equivalent summaries, average-reward diagnostics, and rank-reliability tests. Finally, the set of compared methods is limited by the availability of compatible public implementations; fCBO, CNEI, GACBO, and ACBO for hard interventions are not included in the hard-intervention comparison.

\section{Open Problems and Future Directions}
\label{sec:open_problems}

The benchmark results and methodological analysis presented in this survey reveal that CBO is a maturing but incomplete field: no single method dominates across settings (Section~\ref{subsec:cross_cutting}). To make the deployment gap concrete, consider using CBO to tune a treatment policy, industrial controller, or ecological intervention from observational data and a partially trusted causal graph. In such a setting, hidden confounding, graph errors, measurement noise, heterogeneous intervention costs, and safety constraints are not secondary details; any one of them can make an apparently promising intervention unreliable or harmful. In the following, we organize the most important open challenges into six concrete research directions. Each direction is stated relative to this deployment scenario: robustness to causal assumptions addresses the partially trusted graph and possible hidden confounding; scalability addresses the fact that realistic systems have far more variables and candidate scopes than current benchmarks; richer intervention models address mixed hard/soft/stochastic actions and heterogeneous intervention costs; realistic evaluation addresses the gap between benchmark-specified conditions and deployment conditions such as distribution shift; theoretical foundations address the guarantees a practitioner would need before acting on a recommended intervention; and integration with representation learning addresses settings where the causal variables themselves are not directly measured.

\subsection{Robustness to causal assumptions}

\paragraph{Hidden confounding.}
When unobserved confounders are present, observational priors constructed via backdoor adjustment or the $g$-formula can be systematically biased, leading the acquisition function to favor interventions that appear promising under confounded associations but fail under true causal effects. Developing CBO methods that explicitly model sensitivity to hidden confounding, for example, by propagating bounds on the causal effect under varying degrees of unmeasured confounding \citep{pearl2009causality}, is an important open direction.

\paragraph{Graph misspecification.}
Even when a graph is provided, edges may be missing, spurious, or incorrectly oriented. The impact of such errors on CBO performance has been largely unexplored. Future work should characterize how graph misspecification degrades scope reduction and prior quality and develop methods that are robust to bounded graph errors.

\paragraph{Prior misspecification.}
CNEI \citep{Li2023} addresses noise robustness, but the broader problem of controlling how strongly observational information influences acquisition decisions under potential misspecification remains open. Methods that adaptively downweight observational priors as interventional data accumulate, analogous to Bayesian robustness techniques in other domains, would be valuable.

\subsection{Scalability}
\label{subsec:scalability}

\paragraph{High-dimensional graphs.}
HCBO \citep{Wu2024HighDimensionalCBO} addresses the selection of the scope in larger graphs, but the combinatorial explosion of candidate scopes remains a fundamental bottleneck. Future methods should take advantage of causal sparsity, hierarchical decomposition, or learned embeddings of the intervention scopes to avoid exhaustive enumeration. Notably, hierarchical decomposition across levels of causal granularity has already been formalized in bandit settings by the causal abstraction literature (Section~\ref{subsec:adjacent_fields}); transferring these abstraction-based pruning and transfer mechanisms to GP-based CBO with continuous intervention domains is a concrete and promising path toward scalability.

\paragraph{Unknown-graph scalability.}
CEO \citep{branchini2023CEO} and GACBO \citep{mukherjee2024} demonstrate unknown-graph CBO on small graphs, but maintaining a posterior or confidence set over graphs becomes intractable as the number of variables grows. Scalable alternatives might include local structure learning focused on the decision-relevant subgraph, amortized inference over graph posteriors, or graph neural network-based surrogates that learn causal structure implicitly.

\paragraph{Computational cost.}
Mechanism-level methods (MCBO, ACBO) require maintaining and updating multiple GPs and propagating uncertainty through the graph at each iteration. The per-iteration cost scales with the number of mechanism models and the depth of the graph. Sparse GP approximations, variational inference, or neural process surrogates could reduce this computational burden.

\subsection{Richer intervention models}

\paragraph{Mixed intervention types.}
Many practical systems involve a mix of hard interventions (e.g. turning a device on or off), soft interventions (e.g., adjusting a controller's parameters) and stochastic interventions (e.g., randomizing a treatment assignment probability). No existing CBO method handles arbitrary mixtures of these types within a single optimization loop.

\paragraph{Multi-agent and sequential interventions.}
ACBO \citep{ACBOSussex} considers adversarial environments, but the more general setting of multi-agent intervention, where multiple decision-makers intervene in the same system, possibly with conflicting objectives, remains unexplored. Similarly, extending CBO to sequential multi-stage intervention policies, where the intervention at each stage depends on the outcomes of previous stages, would bridge CBO and causal reinforcement learning.

\paragraph{Intervention cost heterogeneity.}
Our benchmark uses unit-cost interventions to isolate sample efficiency, but real interventions differ substantially in cost. The cost may depend on the target variable, the magnitude of the intervention, the safety requirements, or the logistics of implementation. Future CBO benchmarks should incorporate explicit cost models and future methods should optimize cost-adjusted efficiency rather than pure sample efficiency.

\subsection{Realistic evaluation}

\paragraph{Standardized benchmarks.}
The fragmentation of the data set documented in Table~\ref{tab:dataset_usage} makes the comparison between papers unreliable. The community would benefit from a shared and versioned benchmark suite with standardized SCMs, intervention domains, observational datasets, and evaluation code, analogous to established benchmarks in reinforcement learning or supervised learning. In this spirit, the benchmark released with this survey should be read as a unified repository of existing environments together with an execution and scoring harness, not as a definitive benchmark suite. Establishing the latter would additionally require high-dimensional, complex, nonlinear simulators with rich causal structure among manipulable variables (e.g., realistic epidemiological or climate models) to stress-test scalability, combinatorial scope reduction, and robustness; building such simulators is an important direction for future work.

\paragraph{Evaluation under distribution shift.}
Most benchmarks assume that the test-time SCM matches the training-time SCM. In practice, the system can change between the observational data collection phase and the intervention deployment phase. Evaluating CBO under a controlled distribution shift would test the robustness of observational priors and structure assumptions.

\paragraph{Trajectory-aware metrics.}
Our analysis shows that GAP and PA-GAP can produce different rankings, highlighting the need for trajectory-aware evaluation. Future work should develop principled metrics that capture the full optimization trajectory, potentially incorporating cost, risk, and constraint satisfaction alongside objective improvement.

\subsection{Theoretical foundations}

\paragraph{Finite-sample regret bounds.}
MCBO \citep{Sussex2022MCBO} provides regret bounds under mechanism-level GP models, and ACBO \citep{ACBOSussex} extends these to adversarial settings. However, regret bounds for effect-level CBO with observational priors, for unknown-graph CBO, and for constrained CBO are largely missing. Developing such bounds would clarify when causal structure provably improves over black-box BO and under what conditions.

\paragraph{Identifiability-aware acquisition.}
Current methods either assume full identifiability or ignore identifiability entirely. An intermediate approach would design acquisition functions that explicitly account for the degree of identifiability: using observational information aggressively when identification is strong and falling back to interventional exploration when it is weak. This would connect CBO to the literature on partial identification and sensitivity analysis in causal inference.

\paragraph{Sample complexity of scope reduction.}
The POMIS-based scope reduction in CBO eliminates dominated scopes using observational evidence, but the sample complexity of this elimination, i.e., how many observational samples are needed to reliably identify the correct POMIS, has not been analyzed. Such an analysis would clarify when scope reduction is reliable and when it may prematurely discard useful scopes.

\subsection{Integration with modern representation learning}

\paragraph{Causal representation learning.}
Methods for learning causal variables and structure from high-dimensional observations \citep{scholkopf2021toward} could enable CBO in settings where the causal graph operates over latent variables rather than directly observed quantities. This would extend CBO to image- or text-based intervention settings. A complementary and arguably more direct route is causal abstraction learning, which relates a fine-grained causal model to a coarser one through abstraction maps; integrating learned abstraction maps into the CBO loop (Section~\ref{subsec:adjacent_fields}) is a concrete instance of the latent-variable optimization advocated here.

\paragraph{Large language models for causal reasoning.}
Large language models (LLMs) have shown preliminary ability to reason about causal relationships from text descriptions. An intriguing direction is to use LLMs to elicit or refine causal graphs from domain experts, providing CBO with better structural priors without requiring formal causal discovery from data.

\paragraph{Neural surrogates.}
Replacing GP-based surrogates with neural network-based models (e.g., neural processes, transformers) could improve scalability and flexibility, particularly for high-dimensional mechanism functions or non-stationary environments. The challenge is maintaining the well-calibrated uncertainty quantification that makes GP-based CBO effective.

\section{Conclusion}
\label{sec:conclusion}

This survey has provided a comprehensive review of Causal Bayesian Optimization, organizing the growing literature through a unified design-space perspective that separates graph and system-knowledge assumptions, environmental assumptions, intervention representation, surrogate architecture, and decision rules. By analyzing each major CBO variant through this lens, we have shown that the field's diversity reflects systematic responses to specific bottlenecks of the original CBO framework (temporal dynamics, safety constraints, high dimensionality, unknown graphs, soft interventions, context dependence, and adversarial non-stationarity) rather than ad hoc extensions.

Our benchmark study reinforces a central finding: the value of causal structure in optimization is conditional. Causal methods provide clear advantages when their structural assumptions align with the problem, when scope reduction eliminates irrelevant interventions, when observational priors are well-calibrated, or when mechanism-level modeling captures the system's compositional structure. However, when these conditions are not met, strong non-causal BO baselines remain competitive, and method rankings depend substantially on the choice of dataset, budget, and evaluation metric. The introduction of PA-GAP as a trajectory-aware complement to GAP highlights that even the definition of ``good performance'' in CBO is not straightforward: final-value metrics and trajectory-level metrics can favor different methods.

Looking ahead, the most pressing challenges are not purely algorithmic. Our benchmark does not establish that CBO methods are robust in high-stakes deployment settings; rather, it shows that causal methods can be useful under benchmark-specified assumptions and that performance can change substantially when learner-side causal assumptions are perturbed. Robustness to graph misspecification, hidden confounding, imperfect interventions, measurement error, heterogeneous intervention costs, and safety constraints remains a prerequisite for CBO to move from controlled benchmarks to real-world deployment. We hope that the unified perspective, benchmark, stress-test protocols, and structured open problems presented in this survey will help researchers navigate the CBO landscape and focus their efforts on directions with practical impact.

\paragraph{Broader impact and deployment caution.}
CBO is increasingly positioned for deployment in sensitive domains such as healthcare dosing, resource allocation, and public policy. Deploying CBO with misspecified graphs, hidden confounding, or miscalibrated causal priors can select harmful physical interventions in real-world systems. A practitioner reading this survey should not conclude that CBO methods offer reliable advantages in such settings; the evidence applies only to the correctly specified, benchmark-controlled regime. Causal assumptions and failure logging should be treated as part of deployment risk, not merely as technical details. The graph-misspecification and omitted-variable stress tests in Appendix~\ref{appendix:robustness_stress_tests} are a first step toward the robustness evaluation needed for deployment, but comprehensive robustness across the dimensions listed above remains an open evaluation dimension.
\bibliographystyle{tmlr}
\bibliography{main}

\clearpage
\appendix
\startcontents[appendix]

\printappendixtoc
\section{Technical Appendix}
\label{appendix:main}

\subsection{Metrics analysis}
\label{appendix:metrics}

DCBO introduced the GAP metric as a scalar that combines how much an optimizer improves on its initial value with how early it reaches its best value \citep{aglietti2021}. For a minimization problem, let
$y(\mathbf{x})=\mathbb{E}[Y\mid do(X=\mathbf{x})]$, let $y^*$ denote the reference optimum (see Appendix~\ref{appendix:reference_optima}), let
$\mathbf{x}_{\text{init}}$ be the initial best point, and let $\mathbf{x}^*_t$ be the best point
found up to trial $t$. Let $T$ be the total number of trials, and let $t^*$ be the first trial at
which the final best value $y(\mathbf{x}^*_T)$ is reached. GAP is
\begin{equation}
\mathrm{GAP}
=
\left[
\underbrace{\frac{y(\mathbf{x}^*_T)-y(\mathbf{x}_{\text{init}})}
{y^*-y(\mathbf{x}_{\text{init}})}}_{\text{Improvement term}}
+
\underbrace{\frac{T-t^*}{T}}_{\text{Efficiency term}}
\right]
\Bigg/
\underbrace{\left(1+\frac{T-1}{T}\right)}_{\text{Normalization}}.
\label{eq:gap_app}
\end{equation}

\subsubsection{GAP value range}
\label{appendix:gap_range}

When the global optimum is found in the first trial ($t^*=1$), the efficiency term is
equal to $\tfrac{T-1}{T}$, and normalization makes the maximum GAP equal to~$1$. However, DCBO
sets $t^*=0$ when no improvement is observed. In that case, the improvement term is $0$, but the
efficiency term becomes $1$, which produces a nonzero GAP even without improvement. Specifically, GAP is then lower-bounded by
\begin{equation}
\frac{1}{1+\frac{T-1}{T}}=\frac{T}{2T-1}.
\end{equation}
To correct for this, we set $t^*=T$ when no improvement occurs. Then both the improvement term and the efficiency term are $0$, so the GAP ranges over $[0,1]$ as intended.

\subsubsection{Why PA-GAP can rank trajectories more consistently}
\label{appendix:pagap_motivation}

GAP depends only on the final best value and the trial in which that final best value is first
reached. This can bias the rankings when two optimization runs share the same early progress but differ
in later improvements. Figure~\ref{fig:demo_comp} illustrates this with a two-model example: both models reach the same improved value at trial $10$, but only Model~2 later reaches the global optimum (at trial $30$ of $T=40$), and GAP nonetheless scores Model~1 higher because it evaluates only the final best value and its discovery time rather than the full optimization path.

To better reflect the full optimization process, PA-GAP computes a weighted improvement value at
each trial and then averages these values over the trajectory. For minimization tasks, the
best-so-far improvement ratio in the trial $t$ is
\begin{equation}
R_t
=
\min\left\{
\frac{y(\mathbf{x}_{\text{init}})-y(\mathbf{x}^*_t)}
     {y(\mathbf{x}_{\text{init}})-y^*},
1
\right\},
\end{equation}
and for maximization tasks, the numerator and denominator are reversed:
\begin{equation}
R_t
=
\min\left\{
\frac{y(\mathbf{x}^*_t)-y(\mathbf{x}_{\text{init}})}
     {y^*-y(\mathbf{x}_{\text{init}})},
1
\right\}.
\end{equation}
The clipping prevents numerical artifacts when an observed best value exceeds the offline reference
optimum. The PA-GAP score used in our benchmark is then
\begin{equation}
\mathrm{PA\mbox{-}GAP}
=
\frac{1}{T}\sum_{t=1}^{T}
\left(
\underbrace{R_t}_{\text{Improvement ratio}}
\cdot
\underbrace{\frac{T-(t-1)}{T}}_{\text{Efficiency weight}}
\right).
\label{eq:pagap_app}
\end{equation}
Here $T$ denotes the number of interventional trials, excluding the initial value in the
best-so far trajectory.

\begin{figure}[t]
    \centering
    \includegraphics[width=.82\linewidth]{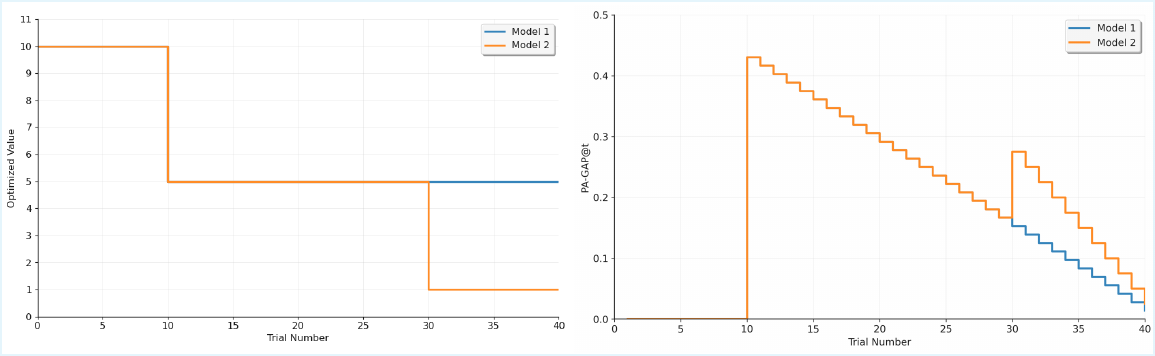}
    \caption{Demo comparison illustrating the ranking bias of GAP. The left panel shows best-so-far trajectories, while the
    right panel shows the corresponding per-trial PA-GAP (PA-GAP@t) contributions. The two models share the same
    early improvement, but Model 2 later reaches the global optimum. GAP can favor the earlier but
    worse final trajectory, while PA-GAP assigns credit to later improvement by averaging weighted
    best-so-far progress across the full path.}
    \label{fig:demo_comp}
\end{figure}
\begin{figure}[t]
    \centering
    \includegraphics[width=.82\linewidth]{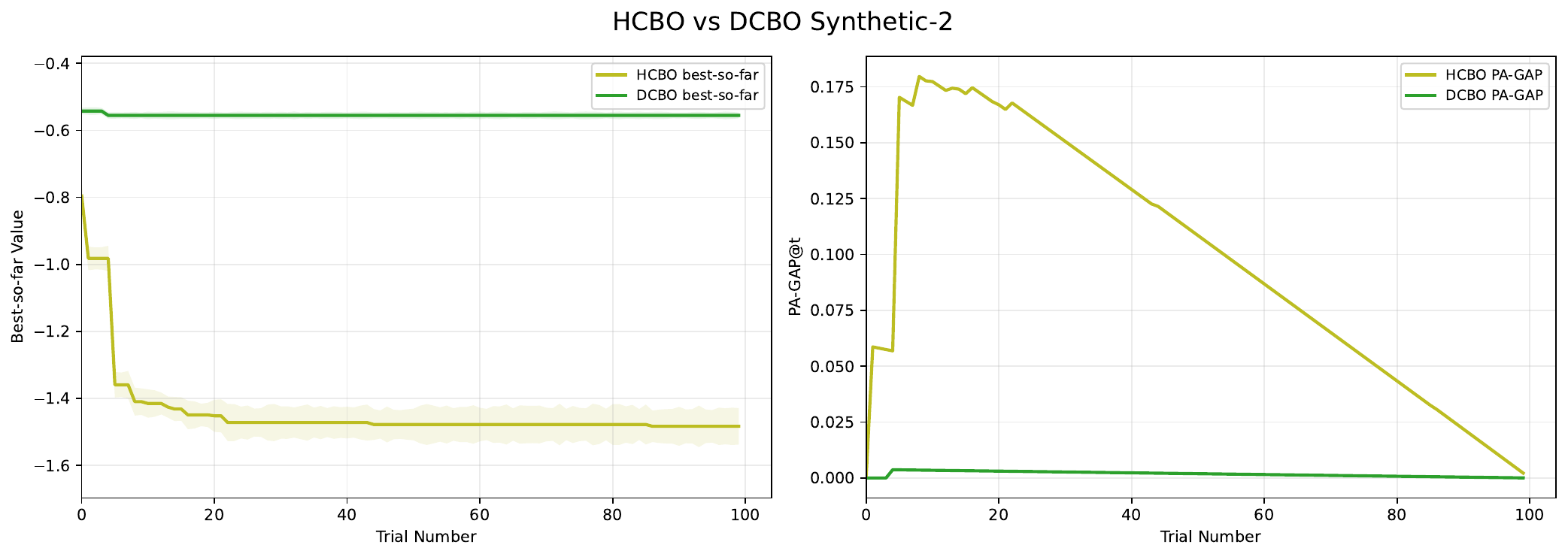}
    \caption{HCBO and DCBO on Synthetic-2. The left panel shows best-so-far trajectories, while the
    right panel shows the corresponding per-trial PA-GAP (PA-GAP@t) contributions. DCBO reaches its own
    best-so-far value early, which gives it a higher GAP score, but HCBO reaches better objective
    values over the trajectory and receives the higher PA-GAP score.}
    \label{fig:hcbo_dcbo_synthetic2_pagap}
\end{figure}
This definition is not normalized to have a maximum value $1$. Since $R_t\in[0,1]$ and the efficiency
weight decreases linearly from $1$ at the first trial to $1/T$ at the last trial, the raw PA-GAP
range is
\begin{equation}
0
\leq
\mathrm{PA\mbox{-}GAP}
\leq
\frac{1}{T}\sum_{t=1}^{T}\frac{T-(t-1)}{T}
=
\frac{T+1}{2T}.
\end{equation}
The upper bound is attained when the global optimum is found in the first trial and remains the
best value so far for the rest of the budget. Thus, PA-GAP should be interpreted as a raw
trajectory-weighted efficiency score, not as a unit-normalized score. Larger values indicate faster
and stronger best-so-far progress, but raw PA-GAP values should be compared directly only within
the same trial budget.

Figure~\ref{fig:hcbo_dcbo_synthetic2_pagap} shows the same issue on the Synthetic-2 benchmark: DCBO quickly reaches a stable best-so-far value, so GAP rewards its early discovery ($0.485\pm.012$ versus $0.174\pm.071$ for HCBO), but that early value is much worse than the values HCBO eventually reaches, and PA-GAP reverses the ranking ($0.002\pm.001$ for DCBO versus $0.099\pm.054$ for HCBO), better reflecting HCBO's sustained progress toward the reference optimum.

\paragraph{Summary and formal properties.}
GAP is simple and interpretable, but can produce counterintuitive rankings when methods differ primarily in their late-stage behavior or trajectory shape. PA-GAP complements GAP by rewarding sustained progress throughout the budget. Table~\ref{tab:gap_vs_pagap_properties} contrasts the two metrics along key evaluation dimensions.

\begin{table}[t]
\centering
\small
\caption{Formal comparison of GAP and PA-GAP evaluation properties.}
\label{tab:gap_vs_pagap_properties}
\setlength{\tabcolsep}{4pt}
\renewcommand{\arraystretch}{1.15}
\begin{tabular}{@{}p{0.28\textwidth}p{0.32\textwidth}p{0.32\textwidth}@{}}
\toprule
Property & GAP & PA-GAP \\
\midrule
\textbf{What is evaluated} & Final best value $+$ discovery time & Full best-so-far trajectory \\
\textbf{Value range} & $[0,1]$ (after correction) & $[0,\frac{T+1}{2T}]$ \\
\textbf{Normalization} & Unit-normalized & Raw; budget-dependent upper bound \\
\textbf{Sensitivity to late improvement} & Low: only the time of the \emph{final} best matters & High: all improvements contribute \\
\textbf{Sensitivity to budget $T$} & Moderate: efficiency term scales with $T$ & Strong: values not comparable across different $T$ \\
\textbf{Plateau handling} & Rewards early plateau at good value & Penalizes early plateau if further improvement is possible \\
\textbf{Computational cost} & $\mathcal{O}(T)$ & $\mathcal{O}(T)$ \\
\bottomrule
\end{tabular}
\end{table}

We state two key properties of PA-GAP that justify its use as a trajectory-aware complement to GAP.

\begin{lemma}[Trajectory monotonicity of PA-GAP]
\label{lem:pagap_monotone}
Let $\{R_t\}_{t=1}^T$ and $\{\tilde{R}_t\}_{t=1}^T$ be two sequences of the best improvement ratios so far satisfying $R_t \geq \tilde{R}_t$ for all $t \in \{1,\dots,T\}$. Then $\mathrm{PA\mbox{-}GAP}(\{R_t\}) \geq \mathrm{PA\mbox{-}GAP}(\{\tilde{R}_t\})$, with equality if and only if $R_t = \tilde{R}_t$ for all $t$.
\end{lemma}
\begin{proof}
Since the efficiency weights $w_t = \tfrac{T-(t-1)}{T} > 0$ for all $t$, the PA-GAP functional $\tfrac{1}{T}\sum_{t=1}^T R_t \cdot w_t$ is a positively weighted average of the improvement ratios. The conclusion follows from the strict positivity of the weights.
\end{proof}

This property ensures that if one trajectory Pareto-dominates another at every trial, PA-GAP reflects this dominance. GAP does not satisfy this property in general: a trajectory with uniformly higher $R_t$ can receive a lower GAP score if its final best is reached later.

\begin{lemma}[Budget dependence of PA-GAP]
\label{lem:pagap_budget}
The maximum attainable PA-GAP value is $\tfrac{T+1}{2T}$, which is strictly decreasing in $T$ and approaches $\tfrac{1}{2}$ as $T \to \infty$. Consequently, raw PA-GAP values are comparable only within the same budget $T$.
\end{lemma}
\begin{proof}
The maximum is achieved when $R_t = 1$ for all $t$, giving $\mathrm{PA\mbox{-}GAP} = \tfrac{1}{T}\sum_{t=1}^{T}\tfrac{T-(t-1)}{T} = \tfrac{T+1}{2T}$. Since $\tfrac{d}{dT}\!\left(\tfrac{T+1}{2T}\right) = -\tfrac{1}{2T^2} < 0$, the upper bound is strictly decreasing in $T$.
\end{proof}

\paragraph{Practical recommendation.}
We recommend reporting both metrics alongside trajectory plots: GAP identifies methods that quickly find a strong final solution, while PA-GAP evaluates whether progress is sustained throughout the budget.

\begin{table*}[t]
\centering
\small
\caption{Implementation settings used in the benchmark. Method-specific defaults are preserved unless a wrapper override is needed for shared evaluation.}
\label{tab:baseline_hyperparameters}
\setlength{\tabcolsep}{5pt}
\renewcommand{\arraystretch}{1.12}
\begin{tabular}{p{0.13\textwidth}p{0.37\textwidth}p{0.42\textwidth}}
\toprule
Method & Benchmark-controlled settings & Method-specific notes \\
\midrule
BO &
Trial budget, seed, initial design, objective direction, and intervention domain fixed by wrapper; no graph input used. &
Graph-free GP-EI baseline on the joint intervention vector: GPy GPRegression with RBF kernel (lengthscale $=1$, variance $=1$, ARD disabled), fixed observation noise $10^{-10}$, emukit Expected Improvement, and GradientAcquisitionOptimizer. All methods, including BO, use the same fixed initial interventional design per seed and the same intervention parameterization and optimization direction. \\
\cline{1-3}
MCBO &
Trial budget and seed fixed by wrapper; noise scale set to $0$; $\beta=10$. &
Uses GP networks with MC acquisition. Acquisition optimization follows the released implementation defaults, including restart and raw-sample rules. \\
\cline{1-3}
ACBO &
Trial budget and seed fixed by wrapper; noise scale set to $0$; $\beta=10$; default discrete action cardinality $4$. &
Used only for soft-intervention function-network tasks. Runs CBO-MW with multiplicative-weights updates. \\
\cline{1-3}
cCBO &
Trial budget fixed by wrapper; dataset configs determine observational samples, intervention sets, and constraints. &
Uses the single-task cCBO setting in our benchmark. Protein uses an internal offset so that the exported trajectory contains the intended number of evaluated trials. \\
\cline{1-3}
DCBO &
Trial budget and seed fixed by wrapper; dataset YAML files specify the SCM and intervention domain. &
Benchmark uses a single-time setting to make DCBO comparable with static hard-intervention tasks. \\
\cline{1-3}
CEO &
Trial budget and seed fixed by wrapper; small candidate graph and anchor-point settings follow the benchmark runner. &
Uses the benchmark-native CEO runner by default. The original vendored CEO implementation can be enabled separately. \\
\cline{1-3}
CoCaBO &
Trial budget, seed, and objective direction fixed by wrapper. &
Uses UCB over active policy scopes and HEBO-style inner optimization. Benchmark configs do not include additional contextual variables unless specified. \\
\cline{1-3}
HCBO &
Seed and trial settings fixed by wrapper when supported. &
Used for compatible hard-intervention settings. Acquisition and initialization constants follow the released HCBO-style implementation. \\
\bottomrule
\end{tabular}
\end{table*}
\subsection{Implementation protocol}
\label{appendix:protocol}

\paragraph{Reproducibility philosophy.}
All methods run with released default hyperparameters unless noted in Table~\ref{tab:baseline_hyperparameters}. These are wrapper-default reproducibility results, not per-dataset tuned head-to-head comparisons. This choice prioritizes reproducibility and isolates the behavior of public implementations, but it can understate what a carefully tuned method might achieve. Appendix~\ref{appendix:mcbo_beta_sensitivity} illustrates this sensitivity for MCBO: the best UCB $\beta$ varies across datasets, budgets, and metrics, so the main table keeps the wrapper setting $\beta=10$ and reports the sweep separately.

\paragraph{Software environment.}
All experiments are conducted in Python~3.10 using PyTorch~2.0 and GPyTorch~1.11 as the primary GP backend for methods that depend on Gaussian process surrogates. Each CBO method is run in its own isolated virtual environment to prevent dependency conflicts. The benchmark evaluation pipeline--including trajectory export, GAP/PA-GAP computation, and figure generation--is implemented in a shared codebase that is independent of any individual method's software stack. This separation ensures that scoring differences reflect algorithmic behavior rather than library version effects. The exact package versions for each method are documented in the benchmark repository.

\paragraph{Computational resources.}
All experiments were conducted on a single-GPU CUDA workstation equipped with one NVIDIA A100 GPU
with 40 GiB of VRAM, Intel Xeon Platinum 8481C CPUs, and 754 GiB of system memory. Most methods
complete 100 trials on a single dataset within 5--30 minutes. The main exceptions are CEO, which
incurs additional overhead from graph posterior sampling, and MCBO, which propagates uncertainty
through multiple mechanism-level GPs. GPU acceleration is used when supported by the corresponding
implementation; otherwise, experiments are executed on the CPU.

\paragraph{Wrapper design.}
Wrapper-level overrides are limited to shared experimental controls: trial budget, random seed, dataset
mapping, objective direction (minimization or maximization), and trajectory export format. Method-internal choices, including the selection of the GP kernel, the acquisition optimization strategy, the restart schedules, the scope enumeration, and the internal noise models, are left at their released defaults. Each wrapper converts the method's internal output into the common trajectory format described below. Wrappers do not modify the optimization loop itself; they only control initialization and output extraction.

\paragraph{Trajectory normalization.}
To make results comparable across heterogeneous codebases, every method's output is converted to a
common trajectory format with two fields: trial number and best-so-far objective value. For
minimization tasks, this value is the running minimum observed up to the trial $t$; for maximization tasks, the running maximum. The trajectory includes an initial best-of-so-far value (before any interventional trial), followed by one entry per interventional trial. Thus, a budget of $B$ interventions produces a trajectory of $B+1$ entries. GAP and PA-GAP are then computed from these exported
trajectories using a single evaluation script shared by all methods, ensuring that scoring differences reflect algorithmic behavior rather than implementation artifacts. This design also makes it easy to add new methods or metrics without modifying existing implementations.

\paragraph{Zero-score diagnosis.}
A zero or near-zero GAP or PA-GAP value in the main tables means that the method did not improve over initialization under the shared scoring protocol; it does not mean that the run was silently removed (genuine failures such as missing outputs or non-finite trajectories are logged separately under the failure-handling protocol). In principle, a zero score could arise from (a)~a poorly calibrated observational prior, (b)~a loose or unreliable reference optimum, or (c)~graph-induced scope selection that excludes useful interventions. In the current logs, the dominant cause is failure to improve over the initial design; in a smaller number of cases, the zero score reflects method--task incompatibility, unsupported or empty candidate scopes (cause~c), or poor calibration of the observational prior (cause~a). We found no case attributable to an unreliable reference optimum (cause~b): the protocol in Appendix~\ref{appendix:reference_optima} computes references offline on the ground-truth SCM, and reference error would rescale all methods' scores on a dataset rather than zero out a single method. Two named cases illustrate the distinction. cCBO on Protein-reconstructed completes all 20 seeds and returns valid, finite trajectories, but its best-so-far value remains flat over the entire budget under the reconstructed SCM and its wide Sachs-derived intervention domains, so GAP and PA-GAP are zero by definition; the runs do not crash, and the discrepancy with the original cCBO Protein results is consistent with our task being a transparent reconstruction rather than the unreleased fitted SCM used in the original experiments (Appendix~\ref{appendix:Protein}). MCBO's zero entries at $T=20$ and $T=50$ on several datasets reflect delayed improvement: the same runs reach non-zero scores at $T=100$, consistent with a method that requires many trials before improving on the initial design.

\subsection{Reference optima}
\label{appendix:reference_optima}

The reference optimum $y^*$ used by GAP and PA-GAP is computed offline from the SCM benchmark, function generator or oracle interventional data, as detailed in Table~\ref{tab:reference_optima}. It is used \emph{only} for scoring and is never provided to any optimizer during a run. At evaluation time, all methods on a given data set are compared against the same fixed reference value using the same task direction.

\paragraph{Reliability categories.}
The reference optima fall into two reliability classes. For datasets with analytically known optima (Dropwave, Alpine2, Ackley, Rosenbrock), the reference is exact and introduces no scoring uncertainty. For SCM-based datasets, the reference is obtained by dense grid search or large-sample uniform random search over the admissible intervention domain. In these cases, $y^*$ represents the \emph{best known achievable outcome}, not necessarily the true global optimum. This distinction is important for datasets with complex, potentially multimodal objective landscapes: if the reference underestimates the true optimum, both GAP and PA-GAP will overestimate the improvement ratios $R_t$, potentially making all methods appear closer to optimal than they actually are. In the opposite direction, an overestimated reference (which cannot occur by construction in our protocol, since we use oracle SCM access) would inflate the denominator and deflate all scores. To mitigate this risk, we use high-budget offline search (10{,}000 samples for Ecology and Epidemiology, $200 \times 200$ grids for Protein-reconstructed, and 1{,}000-point grids for Synthetic-2) and cross-validate against known interventional datasets where available.

\paragraph{Impact on metric comparisons.}
Because both GAP and PA-GAP normalize the improvement relative to $y^* - y(\mathbf{x}_{\mathrm{init}})$, the quality of $y^*$ affects the absolute metric values, but does not affect \emph{the relative rankings} between methods in the same dataset, provided that all methods are scored against the same reference. Cross data set comparisons of absolute GAP or PA-GAP values should therefore be made with caution, as differences in reference quality across datasets can introduce systematic biases.

\begin{table*}[h]
\centering
\scriptsize
\renewcommand{\arraystretch}{0.5}
\setlength{\belowcaptionskip}{6pt}
\caption{Reference optima used for GAP and PA-GAP scoring. Analytic optima are exact; grid/random-search references represent the best known achievable value.}
\label{tab:reference_optima}

\newlength{\RefRowH}
\setlength{\RefRowH}{0.82cm}

\newcommand{\centercell}[2]{\parbox[c][\RefRowH][c]{#1}{\centering\arraybackslash #2}}

\begin{tabular}{@{}cccc@{}}
\toprule
\centercell{0.18\textwidth}{Dataset} &
\centercell{0.10\textwidth}{Task} &
\centercell{0.14\textwidth}{$y^*$} &
\centercell{0.48\textwidth}{Offline computation protocol} \\
\midrule
\centercell{0.18\textwidth}{Dropwave} & \centercell{0.10\textwidth}{Max} & \centercell{0.14\textwidth}{$1.0000$} & \centercell{0.48\textwidth}{Analytic optimum from the known function generator.} \\
\cline{1-4}
\centercell{0.18\textwidth}{Alpine2} & \centercell{0.10\textwidth}{Max} & \centercell{0.14\textwidth}{$400.0000$} & \centercell{0.48\textwidth}{Analytic optimum from the known function generator.} \\
\cline{1-4}
\centercell{0.18\textwidth}{Ackley} & \centercell{0.10\textwidth}{Max} & \centercell{0.14\textwidth}{$0.0000$} & \centercell{0.48\textwidth}{Analytic optimum of the negated Ackley objective.} \\
\cline{1-4}
\centercell{0.18\textwidth}{Rosenbrock} & \centercell{0.10\textwidth}{Max} & \centercell{0.14\textwidth}{$0.0000$} & \centercell{0.48\textwidth}{Analytic optimum of the negated Rosenbrock objective.} \\
\cline{1-4}
\centercell{0.18\textwidth}{Chain-soft} & \centercell{0.10\textwidth}{Min} & \centercell{0.14\textwidth}{$-3.3936$} & \centercell{0.48\textwidth}{Reference value from the Chain-soft function generator.} \\
\cline{1-4}
\centercell{0.18\textwidth}{Protein-recon.} & \centercell{0.10\textwidth}{Min} & \centercell{0.14\textwidth}{$-240.0234$} & \centercell{0.48\textwidth}{Reconstructed linear SEM; $200\times200$ grid over PKC and PKA; minimum Erk.} \\
\cline{1-4}
\centercell{0.18\textwidth}{Synthetic-2} & \centercell{0.10\textwidth}{Min} & \centercell{0.14\textwidth}{$-3.9690$} & \centercell{0.48\textwidth}{Best offline oracle outcome; 1000-point grids on $X$ and $Z$.} \\
\cline{1-4}
\centercell{0.18\textwidth}{Synthetic} & \centercell{0.10\textwidth}{Min} & \centercell{0.14\textwidth}{$-4.4474$} & \centercell{0.48\textwidth}{Best value across BO and CBO oracle interventional datasets.} \\
\cline{1-4}
\centercell{0.18\textwidth}{ToyGraph} & \centercell{0.10\textwidth}{Min} & \centercell{0.14\textwidth}{$-2.5718$} & \centercell{0.48\textwidth}{Best value across BO and CBO oracle interventional datasets.} \\
\cline{1-4}
\centercell{0.18\textwidth}{Chain-hard} & \centercell{0.10\textwidth}{Min} & \centercell{0.14\textwidth}{$-15.8060$} & \centercell{0.48\textwidth}{Best value across BO and CBO oracle interventional datasets.} \\
\cline{1-4}
\centercell{0.18\textwidth}{Ecology} & \centercell{0.10\textwidth}{Max} & \centercell{0.14\textwidth}{$9.2652$} & \centercell{0.48\textwidth}{10{,}000-sample uniform search over the manipulative-variable domain.} \\
\cline{1-4}
\centercell{0.18\textwidth}{Epidemiology} & \centercell{0.10\textwidth}{Min} & \centercell{0.14\textwidth}{$-3.3906$} & \centercell{0.48\textwidth}{10{,}000-sample uniform search over the manipulative-variable domain.} \\
\cline{1-4}
\centercell{0.18\textwidth}{Healthcare} & \centercell{0.10\textwidth}{Min} & \centercell{0.14\textwidth}{$4.0401$} & \centercell{0.48\textwidth}{Reference value from the distributed real-data benchmark artifact.} \\
\bottomrule
\end{tabular}
\end{table*}

\subsection{Consolidated dataset summary}
\label{appendix:dataset_summary}

Table~\ref{tab:dataset_provenance} records benchmark provenance, including whether a task is inherited, reconstructed, or newly added, and Table~\ref{tab:dataset_consolidated} provides the code-authoritative structural and domain summary. Together these tables clarify that the benchmark is a standardized harness over heterogeneous prior tasks rather than a claim that every dataset is newly introduced here.

Table~\ref{tab:dataset_consolidated} provides a unified overview of all 13 benchmark datasets, consolidating key structural and experimental properties in a single reference. This table complements the SCM specifications per-dataset in Appendices~\ref{appendix:synthetic_scms}--\ref{appendix:soft_scms} and the usage-coverage table (Table~\ref{tab:dataset_usage}) in the main text. The column ``\#Nodes'' counts endogenous (observed) variables only; latent variables are listed separately where applicable. The column ``Intervention domain'' specifies the admissible range per manipulable variable as used in the benchmark.
\begin{table*}[h]
\scriptsize
\caption{Benchmark provenance and domain reconciliation. ``Current domain'' refers to the code-authoritative domain used in our benchmark.}
\label{tab:dataset_provenance}
\setlength{\tabcolsep}{4pt}
\renewcommand{\arraystretch}{1}
\begin{center}
\begin{tabular}{@{}>{\raggedright\arraybackslash}p{29mm}>{\raggedright\arraybackslash}p{45mm}>{\raggedright\arraybackslash}p{45mm}>{\raggedright\arraybackslash}p{34mm}@{}}
\toprule
Dataset & Source / provenance & Current domain note & Modification status \\
\midrule
ToyGraph & \citet{aglietti2020CBO} & $X{\in}[-5,5]$, $Z{\in}[-5,20]$ & Inherited and standardized \\
Synthetic & CBO-style synthetic SCM & $B,D,E$ code domains & Newly added SCM \\
Synthetic-2 & cCBO-style synthetic SCM & $X{\in}[-3,2]$, $Z{\in}[-1,1]$ & Newly added / reconstructed SCM \\
Chain-hard & Chain benchmark from fCBO-style SCM & $W,Z{\in}[-1,1]$ & Hard-intervention variant \\
Ecology & \citet{aglietti2020CBO} / ecological SCM & Code feature-parameter ranges & Inherited and repackaged \\
Protein-reconstructed & Sachs protein data / cCBO graph & Sachs-derived intervention ranges & Reconstructed fitted SCM \\
Healthcare & \citet{aglietti2020CBO} PSA example & Aspirin, Statin in $[0,1]$ & Inherited and repackaged \\
Epidemiology & CEO/DCBO epidemiology SCM & $L{\in}[0.479,801.787]$, $B{\in}[-0.994,0.999]$ & Reconciled to code-authoritative intervention variables \\
Ackley / Rosenbrock / Dropwave / Alpine2 & Function-network BO \citep{astudillo2021functionnetworks}; used by MCBO/ACBO & Standard function-network domains after wrapper normalization & Inherited function-network tasks \\
Chain-soft & \citet{gultchin2023FCBO} & Policy coefficients in $[-0.27,0.27]$, $W{\in}[-1,1]$ & Inherited policy-intervention task \\
\bottomrule
\end{tabular}
\end{center}
\end{table*}

\begin{table*}[h]
\centering
\small
\caption{Consolidated summary of all thirteen benchmark datasets. For each dataset, we list the category, number of observed nodes, manipulable variables, intervention type, admissible domain per variable, optimization direction, and source reference. Latent variables (if any) are noted in parentheses.}
\label{tab:dataset_consolidated}
\setlength{\tabcolsep}{3pt}
\renewcommand{\arraystretch}{1}
\begin{tabular}{@{}>{\raggedright\arraybackslash}p{19.5mm}>{\raggedright\arraybackslash}p{15.5mm}>{\centering\arraybackslash}p{11mm}>{\raggedright\arraybackslash}p{21mm}>{\raggedright\arraybackslash}p{16mm}>{\raggedright\arraybackslash}p{33mm}>{\raggedright\arraybackslash}p{10mm}>{\raggedright\arraybackslash}p{20mm}@{}}
\toprule
Dataset & Category & \#Nodes & Manipulable & Type & Domain & Task & Source \\
\midrule
ToyGraph & Synthetic (hard) & 3 & $X, Z$ & Hard & $X{\in}[-5,5]$; $Z{\in}[-5,20]$ & Min $Y$ & \citet{aglietti2020CBO} \\
Synthetic & Synthetic (hard) & 7 (+2 latent) & $B, D, E$ & Hard & $B{\in}[-5,4]$; $D{\in}[-5,5]$; $E{\in}[-6,3]$ & Min $Y$ & \citet{aglietti2020CBO} \\
Synthetic-2 & Synthetic (hard) & 3 & $X, Z$ & Hard & $X{\in}[-3,2]$; $Z{\in}[-1,1]$ & Min $Y$ & \citet{aglietti2023cCBO} \\
Chain-hard & Synthetic (hard) & 4 & $W, Z$ & Hard & $[-1,1]$ each & Min $Y$ & \citet{gultchin2023FCBO} \\
\midrule
Ecology & Real/fitted (hard) & 11 & $N, O, C, T, D$ & Hard & $N{\in}[-2,5]$; $O{\in}[2,4]$; $C{\in}[0,1]$; $T{\in}[2200,2500]$; $D{\in}[1950,2100]$ & Max $Y$ & \citet{aglietti2020CBO} \\
Protein-recon. & Real/fitted (hard) & 8 & PKC, PKA, Mek, Akt & Hard & PKC${\in}[0.5,106.5]$; PKA${\in}[1.45,4491.5]$; Mek${\in}[0.5,389.5]$; Akt${\in}[1.2,3555.5]$ & Min Erk & \citet{aglietti2023cCBO} \\
Healthcare & Real/fitted (hard) & 6 & Aspirin, Statin & Hard & $[0,1]$ each & Min PSA & \citet{aglietti2020CBO} \\
Epidemiology & Real/fitted (hard) & 5 & $L, B$ & Hard & $L{\in}[0.479,801.787]$; $B{\in}[-0.994,0.999]$ & Min $Y$ & \citet{branchini2023CEO} \\
\midrule
Ackley & Synthetic (soft) & 9 & $a_0,\ldots,a_5$ & Soft (function network) & $[-2,2]$ each & Max $Y$ & \citet{astudillo2021functionnetworks,Sussex2022MCBO} \\
Rosenbrock & Synthetic (soft) & $D{+}D{-}1$ & $a_0,\ldots,a_{D-1}$ & Soft (function network) & $[-2,2]$ each & Max $Y$ & \citet{astudillo2021functionnetworks,Sussex2022MCBO} \\
Dropwave & Synthetic (soft) & 4 & $a_0, a_1$ & Soft (function network) & $[-5.12,5.12]$ each & Max $Y$ & \citet{astudillo2021functionnetworks,Sussex2022MCBO} \\
Alpine2 & Synthetic (soft) & 12 & $a_0,\ldots,a_5$ & Soft (function network) & $[0,10]$ each & Max $Y$ & \citet{astudillo2021functionnetworks,Sussex2022MCBO} \\
Chain-soft & Synthetic (soft) & 4 & $W$ (hard), $Z$ (policy) & Mixed hard/soft & $W{\in}[-1,1]$; $Z{=}\pi(X)$ with coeffs.~${\in}[-0.27,0.27]$ & Min $Y$ & \citet{gultchin2023FCBO} \\
\bottomrule
\end{tabular}
\end{table*}
\FloatBarrier
\subsection{Synthetic dataset SCMs}
\label{appendix:synthetic_scms}

This section provides the full structural equations, noise distributions, intervention domains, and causal graph diagrams for the four synthetic hard-intervention datasets used in our benchmark. For each data set, we follow a standardized format: (i)~a brief description including the source reference, the number of nodes, and the direction of the task; (ii)~the exogenous variables and their distributions; (iii)~the endogenous structural equations; (iv)~the noise model with explicit distributional parameters; (v)~the intervention domain; and (vi)~a figure of the causal graph with a consistent color scheme (green = manipulable, red = target, light-gray = non-manipulable, dark-gray = latent).

\subsubsection{ToyGraph}
\label{appendix:ToyGraph}
A minimal three-node chain-structured SCM introduced by \citet{aglietti2020CBO}, consisting of variables $X$, $Z$, and $Y$. Both $X$ and $Z$ are manipulable (intervenable); $Y$ is the target. The task is to minimize $Y$. Figure~\ref{fig:toyGraph} shows the causal graph.

\paragraph{Exogenous variable.}
\begin{align*}
X &\sim P_X \quad \text{(input variable, exogenous)}
\end{align*}

\paragraph{Endogenous variables.}
\begin{align*}
Z &= e^{-X} + \varepsilon_{Z}, \\[4pt]
Y &= \cos Z - e^{-Z/20} + \varepsilon_{Y}.
\end{align*}

\paragraph{Noise model.}
All noise terms $\varepsilon_Z, \varepsilon_Y$ are mutually independent. In the original CBO implementation, these are drawn from $\mathcal{N}(0, \sigma^2)$ with a small variance $\sigma^2$.

\paragraph{Intervention domain.}
Hard interventions on $X$ and $Z$ use dataset-specific domains:
\[
X \in [-5,5], \qquad Z \in [-5,20].
\]
The admissible scopes are $\{X\}$, $\{Z\}$, and $\{X,Z\}$.
\begin{figure}[ht]
\centering
\includegraphics[scale=1]{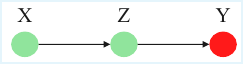} 
\caption{ToyGraph SCM. Green nodes represent manipulable variables; the red node is the target variable $Y$.}
\label{fig:toyGraph}
\end{figure}

\subsubsection{Synthetic}
\label{appendix:Synthetic}
A seven-node SCM with two latent variables, introduced by \citet{aglietti2020CBO}. The observed nodes are $A, B, C, D, E, F, Y$, with latent confounders $U_1$ and $U_2$. Among the observed nodes, $B$, $D$, and $E$ are manipulable; $A$, $C$, and $F$ are not manipulable; and $Y$ is the target. The task is to minimize $Y$. Figure~\ref{fig:completeGraph} shows the causal graph.

\paragraph{Latent variables.}
\begin{align*}
U_1 &= \varepsilon_{YA} \sim \mathcal{N}(0,1), \\[4pt]
U_2 &= \varepsilon_{YB} \sim \mathcal{N}(0,1).
\end{align*}

\paragraph{Exogenous variable.}
\begin{align*}
F &= \varepsilon_{F} \sim \mathcal{N}(0,1).
\end{align*}

\paragraph{Endogenous variables.}
\begin{align*}
A &= F^{2} + U_1 + \varepsilon_{A}, \\[4pt]
B &= U_2 + \varepsilon_{B}, \\[4pt]
C &= e^{-B} + \varepsilon_{C}, \\[4pt]
D &= \frac{e^{-C}}{10} + \varepsilon_{D}, \\[4pt]
E &= \cos A + \frac{C}{10} + \varepsilon_{E}, \\[4pt]
Y &= \cos D + \sin E + U_1 + U_2\,\varepsilon_{Y}.
\end{align*}

\paragraph{Noise model.}
Each $\varepsilon_{\bullet}$ is an independent noise term. In the original CBO implementation, the endogenous noise terms are drawn from $\mathcal{N}(0, \sigma^2)$ with small variance $\sigma^2$.

\paragraph{Intervention domain.}
Hard interventions on $B$, $D$, and $E$ use dataset-specific domains:
\[
B \in [-5,4], \qquad D \in [-5,5], \qquad E \in [-6,3].
\]
The admissible scopes include all non-empty subsets of $\{B,D,E\}$ considered by the benchmark
intervention-set construction.

\begin{figure}[ht]
\centering
\includegraphics[scale=1]{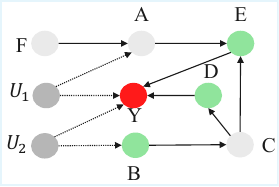} 
\caption{Synthetic SCM. Light-gray nodes are non-manipulable variables; dark-gray nodes are latent confounders; green nodes are manipulable variables; the red node is the target variable $Y$.}
\label{fig:completeGraph}
\end{figure}

\subsubsection{Synthetic-2}
\label{appendix:Synthetic-2}
A three-node SCM introduced in the cCBO line of work \citep{aglietti2023cCBO}, consisting of variables $X$, $Z$, and $Y$. Both $X$ and $Z$ are manipulable; $Y$ is the target. The task is to minimize $Y$. The structure is similar to ToyGraph but uses explicitly specified Gaussian noise. Figure~\ref{fig:synthetic-2} shows the causal graph.

\paragraph{Exogenous variable.}
\begin{align*}
X &= \varepsilon_X, \qquad \varepsilon_X \sim \mathcal{N}(0,1).
\end{align*}

\paragraph{Endogenous variables.}
\begin{align*}
Z &= e^{-X} + \varepsilon_Z, \\[4pt]
Y &= \cos Z - e^{-Z/20} + \varepsilon_Y.
\end{align*}

\paragraph{Noise model.}
$\varepsilon_X \sim \mathcal{N}(0,1)$, $\varepsilon_Z \sim \mathcal{N}(0,1)$, and $\varepsilon_Y \sim \mathcal{N}(0,1)$. All noise terms are mutually independent.

\paragraph{Intervention domain.}
Hard interventions on $X$ and $Z$ use $X \in [-3,2]$ and $Z \in [-1,1]$. The admissible scopes are $\{X\}$, $\{Z\}$, and $\{X, Z\}$.

\begin{figure}[ht]
\centering
\includegraphics[scale=0.8]{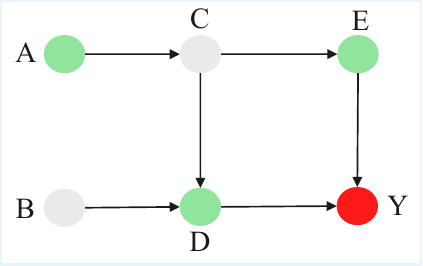} 
\caption{Synthetic-2 SCM. Green nodes represent manipulable variables; the red node is the target variable $Y$.}
\label{fig:synthetic-2}
\end{figure}

\subsubsection{Chain-hard}
\label{appendix:Chain_hard}
A four-node SCM introduced by \citet{gultchin2023FCBO} as the hard-intervention version of the Chain benchmark. The observed variables are $X$, $W$, $Z$, and $Y$. Among them, $W$ and $Z$ are manipulable; $X$ is a non-manipulable context variable; and $Y$ is the target. The task is to minimize $Y$. Figure~\ref{fig:chain-hard} illustrates the causal graph.

\paragraph{Structural equations.}
\begin{align*}
X &= U_X, \\[4pt]
W &= U_W, \\[4pt]
Z &= -0.5X + U_Z, \\[4pt]
Y &= -W - 3ZX + U_Y.
\end{align*}

\paragraph{Noise model.}
$U_X, U_W, U_Z, U_Y \sim \mathcal{N}(0,1)$, mutually independent.

\paragraph{Intervention domain.}
Hard interventions in $W$ and $Z$ are each restricted to $[-1, 1]$. A hard intervention replaces the corresponding structural equation with a fixed constant. The admissible scopes are $\{W\}$, $\{Z\}$, and $\{W, Z\}$.

\begin{figure}[ht]
\centering
\includegraphics[scale=0.8]{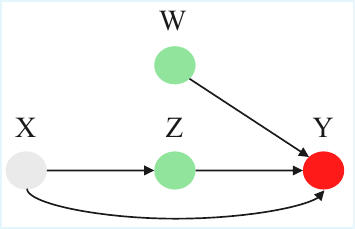} 
\caption{Chain-hard SCM. Light-gray nodes are non-manipulable variables; green nodes are manipulable variables; the red node is the target variable $Y$.}
\label{fig:chain-hard}
\end{figure}

\subsection{Real dataset SCMs}
\label{appendix:real_scms}

This section provides the full structural equations, noise distributions, fitting details, and intervention domains for the four real or fitted-SCM hard-intervention datasets. Unlike the synthetic datasets above, these SCMs are either fitted to real observational data or derived from domain-specific causal models. We follow the same standardized format as in the Appendix~\ref{appendix:synthetic_scms}. For protein reconstruction, we additionally document the fitting procedure to ensure transparent reproducibility.

\subsubsection{Ecology}
\label{appendix:Ecology}
An 11-node SCM based on the Bermuda reef calcification model of \citet{courtney2017environmental}, introduced in CBO \citep{aglietti2020CBO}. The manipulable variables are $\mathrm{Nut}$, $\mathrm{Chl}\alpha$, $\mathrm{TA}$, $\mathrm{DIC}$, and $\Omega_A$. The non-manipulable variables are $\mathrm{Tem}$, $\mathrm{Sal}$, $P_{co_2}$, $\mathrm{Light}$, and $\mathrm{pHsw}$. The target is $\mathrm{NEC}$ (net ecosystem calcification). The task is to \emph{maximize} NEC. Figure~\ref{fig:Ecology} shows the causal graph.

\paragraph{Exogenous variables.}
\begin{align*}
\mathrm{Tem} &= U_{\mathrm{Tem}} \sim \mathcal{N}(24.184130,\; 3.220405^2), \\[4pt]
\mathrm{Sal} &= U_{\mathrm{Sal}} \sim \mathcal{N}(36.591624,\; 0.149197^2), \\[4pt]
\mathrm{Nut} &= U_{\mathrm{Nut}} \sim \mathcal{N}(0.492065,\; 1.592408^2), \\[4pt]
\mathrm{TA} &= U_{\mathrm{TA}} \sim \mathcal{N}(2357.893696,\; 27.609355^2).
\end{align*}

\paragraph{Endogenous variables.}
\begin{align*}
P_{co_2} &= 18.798174 + 15.797384\,\mathrm{Tem} + U_{P_{co_2}}, \\[4pt]
\mathrm{Chl}\alpha &= 0.373420 - 0.002400\,\mathrm{Nut} + U_{\mathrm{Chl}\alpha}, \\[4pt]
\mathrm{Light} &= 6665.081996 - 10737.462582\,\mathrm{Chl}\alpha + U_{\mathrm{Light}}, \\[4pt]
\mathrm{pHsw} &= 8.427706 - 0.000966\,P_{co_2} + U_{\mathrm{pHsw}}, \\[4pt]
\mathrm{DIC} &= 2131.672107 - 0.216560\,P_{co_2} + U_{\mathrm{DIC}}, \\[4pt]
\Omega_A &= 3.245248 + 0.094332\,\mathrm{Tem} + 0.006754\,\mathrm{Sal} - 0.005737\,P_{co_2} + U_{\Omega_A}, \\[4pt]
\mathrm{NEC} &= 211.422555 - 0.000030\,\mathrm{Light} + 0.016680\,\mathrm{Nut} - 25.719277\,\mathrm{pHsw} - 0.500403\,\Omega_A + U_{\mathrm{NEC}}.
\end{align*}

\paragraph{Noise model.}
$U_{P_{co_2}} \sim \mathcal{N}(0, 28.896639^2)$, $U_{\mathrm{Chl}\alpha} \sim \mathcal{N}(0, 0.039295^2)$, $U_{\mathrm{Light}} \sim \mathcal{N}(0, 1546.913034^2)$, $U_{\mathrm{pHsw}} \sim \mathcal{N}(0, 0.005258^2)$, $U_{\mathrm{DIC}} \sim \mathcal{N}(0, 18.412100^2)$, $U_{\Omega_A} \sim \mathcal{N}(0, 0.036958^2)$, and $U_{\mathrm{NEC}} \sim \mathcal{N}(0, 1.073340^2)$. All noise terms are mutually independent.

\paragraph{Intervention domain.}
Hard interventions in $\mathrm{Nut}$, $\mathrm{Chl}\alpha$, $\mathrm{TA}$, $\mathrm{DIC}$, and $\Omega_A$ are applied in domain-specific ranges corresponding to the natural variability reported in \citet{courtney2017environmental}. The benchmark uses the intervention domains specified in the public CBO repository.

\begin{figure}[ht]
\centering
\includegraphics[scale=0.6]{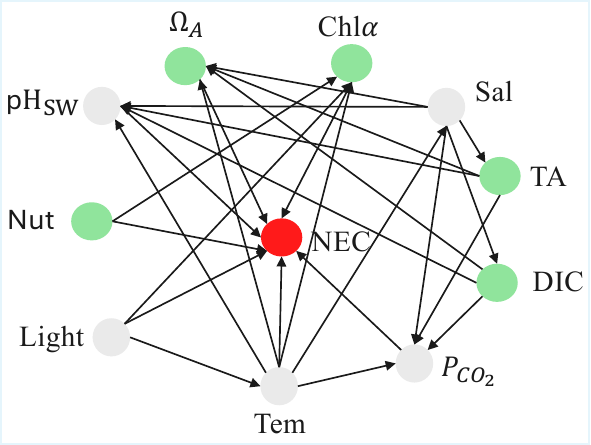} 
\caption{Ecology SCM. Light-gray nodes are non-manipulable variables; green nodes are manipulable variables; the red node is the target variable NEC.}
\label{fig:Ecology}
\end{figure}
\subsubsection{Protein-reconstructed}
\label{appendix:Protein}
An eight-node SCM reconstructed from protein-signaling data from \citet{Sachs2005}, following the DAG used in cCBO \citep{aglietti2023cCBO}. The manipulable variables are $\mathrm{PKC}$, $\mathrm{PKA}$, $\mathrm{Mek}$, and $\mathrm{Akt}$. The non-manipulable variables are $\mathrm{Raf}$, $\mathrm{P38}$, and $\mathrm{Jnk}$. The target is $\mathrm{Erk}$ (extracellular signal-regulated kinase). The task is to minimize Erk. Figure~\ref{fig:Protein} shows the causal graph.

\paragraph{Fitting procedure.}
This is a transparent reconstruction, not an exact reproduction of the unreleased fitted SCM from the original cCBO experiments. We retain the connected subgraph used in cCBO and fit \emph{linear} structural mechanisms using ordinary least squares on 852 observational samples from \citet{Sachs2005}. The parent sets are determined by the cCBO DAG. The linear model class was chosen for reproducibility and interpretability; nonlinear alternatives (e.g. kernel regression) were not explored.

\paragraph{Exogenous variable.}
\begin{align*}
\mathrm{PKC} &= U_{\mathrm{PKC}} \sim \mathrm{Uniform}(1,\,106).
\end{align*}

\paragraph{Endogenous variables.}
\begin{align*}
\mathrm{PKA} &= 554.390731 + 0.841153\,\mathrm{PKC} + U_{\mathrm{PKA}}, \\[4pt]
\mathrm{Raf} &= 62.199046 - 0.177745\,\mathrm{PKC} - 0.000379\,\mathrm{PKA} + U_{\mathrm{Raf}}, \\[4pt]
\mathrm{Mek} &= -1.090275 + 0.039662\,\mathrm{PKC} - 0.000652\,\mathrm{PKA} + 0.520845\,\mathrm{Raf} + U_{\mathrm{Mek}}, \\[4pt]
\mathrm{P38} &= 15.144328 + 1.234783\,\mathrm{PKC} + 0.000591\,\mathrm{PKA} + U_{\mathrm{P38}}, \\[4pt]
\mathrm{Jnk} &= 52.953603 - 0.764801\,\mathrm{PKC} - 0.005306\,\mathrm{PKA} + U_{\mathrm{Jnk}}, \\[4pt]
\mathrm{Akt} &= -31.110747 + 0.128905\,\mathrm{PKA} + U_{\mathrm{Akt}}, \\[4pt]
\mathrm{Erk} &= -23.248743 + 0.081707\,\mathrm{PKA} - 0.029886\,\mathrm{Mek} + U_{\mathrm{Erk}}.
\end{align*}

\paragraph{Noise model.}
$U_{\mathrm{PKA}} \sim \mathcal{N}(0,\,427.437996^2)$, $U_{\mathrm{Raf}} \sim \mathcal{N}(0,\,41.768782^2)$, $U_{\mathrm{Mek}} \sim \mathcal{N}(0,\,16.695865^2)$, $U_{\mathrm{P38}} \sim \mathcal{N}(0,\,13.111164^2)$, $U_{\mathrm{Jnk}} \sim \mathcal{N}(0,\,42.074715^2)$, $U_{\mathrm{Akt}} \sim \mathcal{N}(0,\,113.958504^2)$, and $U_{\mathrm{Erk}} \sim \mathcal{N}(0,\,82.760198^2)$. All noise terms are mutually independent.

\paragraph{Intervention domain and constraints.}
Hard interventions in $\mathrm{PKC}$ and $\mathrm{PKA}$ are drawn from the range $[0.5, 106.5]$ and $[1.45, 4491.5]$, respectively. The interventions in $\mathrm{Mek}$ and $\mathrm{Akt}$ use the domains observed in the Sachs data, which are $[0.5, 389.5]$ and $[1.2, 3555.5]$. Biological plausibility constraints on $\mathrm{PKC}$ and $\mathrm{PKA}$ are enforced by clamping the proposed values to their observed ranges.

\begin{figure}[ht]
\centering
\includegraphics[scale=0.8]{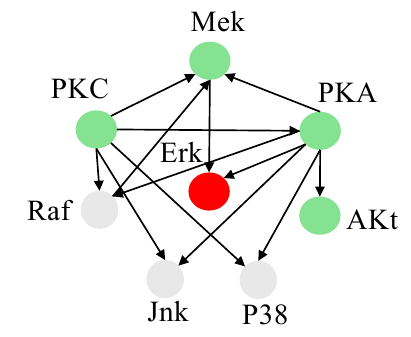}
\caption{Protein-signaling SCM (reconstructed). Light-gray nodes are non-manipulable variables; green nodes are manipulable variables; the red node is the target variable Erk.}
\label{fig:Protein}
\end{figure}

\subsubsection{Healthcare}
\label{appendix:Healthcare}
A six-node SCM derived from a real-world clinical setting \citep{b17}, involving Age, BMI, Aspirin, Statin, Cancer, and PSA. The manipulable variables are Aspirin and Statin. Non-manipulable variables are Age, BMI, and Cancer. The target is PSA (prostate-specific antigen). The task is to minimize PSA. Figure~\ref{fig:Healthcare} shows the causal graph.

\paragraph{Exogenous variable.}
\begin{align*}
\text{Age} &\sim \mathrm{Uniform}(55,\,75).
\end{align*}

\paragraph{Endogenous variables.}
\begin{align*}
\text{BMI} &= 27.0 - 0.01\,\text{Age} + \varepsilon_{\text{BMI}}, \\[4pt]
\text{Aspirin} &= \sigma\!\bigl(-8.0 + 0.10\,\text{Age} + 0.03\,\text{BMI}\bigr) + \varepsilon_{\text{Aspirin}}, \\[4pt]
\text{Statin} &= \sigma\!\bigl(-13.0 + 0.10\,\text{Age} + 0.20\,\text{BMI}\bigr) + \varepsilon_{\text{Statin}}, \\[4pt]
\text{Cancer} &= \sigma\!\bigl(2.2 - 0.05\,\text{Age} + 0.01\,\text{BMI} - 0.04\,\text{Statin} + 0.02\,\text{Aspirin}\bigr) + \varepsilon_{\text{Cancer}}, \\[4pt]
\text{PSA} &= 6.8 + 0.04\,\text{Age} - 0.15\,\text{BMI} - 0.60\,\text{Statin} + 0.55\,\text{Aspirin} + 1.00\,\text{Cancer} + \varepsilon_{\text{PSA}},
\end{align*}
where $\sigma(x) = \tfrac{1}{1+e^{-x}}$ denotes the sigmoid function.

\paragraph{Noise model.}
$\varepsilon_{\text{BMI}} \sim \mathcal{N}(0,\,0.7^2)$, $\varepsilon_{\text{PSA}} \sim \mathcal{N}(0,\,0.4^2)$. Each remaining $\varepsilon_{\bullet}$ is an independent noise term with small variance.

\paragraph{Intervention domain.}
Although the original data template specifies Aspirin and Statin as binary variables, existing CBO implementations commonly relax them to continuous values in $(0,1)$ for BO-style optimization. We follow this convention to remain comparable to the prior work. The admissible scopes are $\{\text{Aspirin}\}$, $\{\text{Statin}\}$, and $\{\text{Aspirin}, \text{Statin}\}$.

\begin{figure}[ht]
\centering
\includegraphics[scale=1]{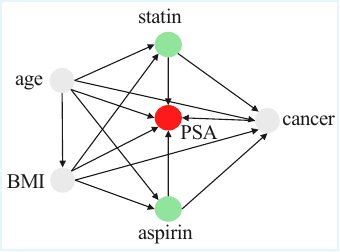} 
\caption{Healthcare SCM. Light-gray nodes are non-manipulable variables; green nodes are manipulable variables; the red node is the target variable PSA.}
\label{fig:Healthcare}
\end{figure}

\subsubsection{Epidemiology}
\label{appendix:Epidemiology}
A five-node SCM introduced in CEO \citep{branchini2023CEO}, modeling an HIV treatment scenario based on \citet{havercroft2012simulating}. The observed variables are $B$, $T$, $L$, $R$, and $Y$. In the code-authoritative benchmark configuration, the manipulable variables are $L$ and $B$; $T$ and $R$ are non-intervention SCM variables. The target is $Y$ (viral load). The task is to minimize $Y$. Figure~\ref{fig:Epidemiology} shows the causal graph.

\paragraph{Exogenous variables.}
\begin{align*}
B &\sim \mathrm{Uniform}(-1,\,1), \\[4pt]
T &\sim \mathrm{Uniform}(4,\,8).
\end{align*}

\paragraph{Endogenous variables.}
\begin{align*}
L &= \exp(0.5\,T + U), \qquad U \sim \mathcal{N}(0,1), \\[4pt]
R &= 4 + L\,T, \\[4pt]
Y &= 0.5 + \cos(4T) + \sin(-L + 2R) + B + \varepsilon, \qquad \varepsilon \sim \mathcal{N}(0,1).
\end{align*}

\paragraph{Noise model.}
$U \sim \mathcal{N}(0,1)$ and $\varepsilon \sim \mathcal{N}(0,1)$. All noise terms are mutually independent.

\paragraph{Intervention domain.}
Hard interventions in $L$ are drawn from $[0.479,801.787]$ and in $B$ from $[-0.994,0.999]$. The admissible scopes are $\{L\}$, $\{B\}$, and $\{L,B\}$. Since CEO uses a modified HIV SCM, more specific clinical interpretations of $B$ and $L$ are not provided in the original article.

\begin{figure}[ht]
\centering
\includegraphics[scale=0.8]{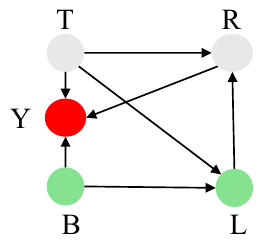} 
\caption{Epidemiology SCM. Light-gray nodes are non-manipulable variables; green nodes are manipulable variables; the red node is the target variable $Y$ (viral load).}
\label{fig:Epidemiology}
\end{figure}

\subsection{Soft-intervention datasets}
\label{appendix:soft_scms}

This section details the five soft-intervention datasets used in the benchmark. These datasets fall into two categories:

\begin{enumerate}[leftmargin=*,itemsep=2pt]
\item \textbf{Function network benchmarks} (Ackley, Rosenbrock, Dropwave, Alpine2): Classical global optimization benchmarks reformulated as deterministic function networks, where each action variable propagates through intermediate nodes to the target. Interventions correspond to setting action variables, and the network structure defines soft causal relationships. These are \emph{maximization} tasks.
\item \textbf{SCM-based policy benchmark} (Chain-soft): A full structural causal model in which a soft intervention replaces the structural equation of one variable with a context-dependent policy. This is a \emph{minimization} task.
\end{enumerate}

For each data set, we specify the network architecture, all structural equations, the action domains, and the target computation. Figures use a consistent color scheme: blue = intermediate variables, orange squares = action (intervention) variables, red = target.

\subsubsection{Ackley}
\label{appendix:Ackley}
\begin{figure}[ht]
\centering
\includegraphics[scale=1]{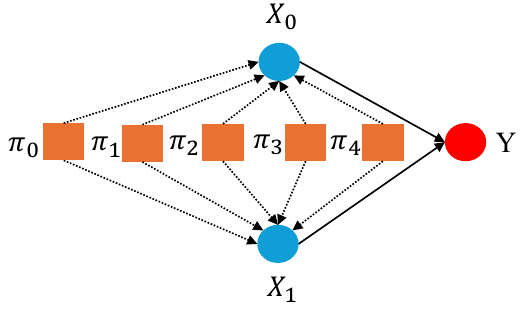} 
\caption{Ackley function network ($D=6$). Blue nodes indicate intermediate variables; orange squares represent action (intervention) variables; the red node is the target variable $Y$.}
\label{fig:Ackley}
\end{figure}

A multimodal, non-separable function-network benchmark with six action variables $a_0,\dots,a_5$, two intermediate nodes $X_0$ and $X_1$, and target node $Y$. The task is to \emph{maximize} $Y$. Following the standard configuration in the BO function-network \citep{Sussex2022MCBO,ACBOSussex}, we use the domain $a_i \in [-2,2]$ for $i=0,\dots,5$, corresponding to $D=6$. Figure~\ref{fig:Ackley} shows the function network.

\paragraph{Action variables.}
$a_i \in [-2,2]$, \quad $i=0,\dots,5$.

\paragraph{Intermediate variables.}
\begin{align*}
X_0 &= \frac{1}{D}\sum_{i=0}^{5} a_i^2, \\[4pt]
X_1 &= \frac{1}{D}\sum_{i=0}^{5} \cos(2\pi a_i),
\end{align*}
where $D=6$.

\paragraph{Target variable.}
\begin{align*}
Y &= 20\exp\!\left(-0.2\sqrt{X_0}\right) + \exp(X_1) - 20 - e.
\end{align*}

\paragraph{Interpretation.}
$X_0$ aggregates the average squared magnitude of the action variables, while $X_1$ aggregates the average cosine term. The target $Y$ combines these two intermediate quantities to recover the Ackley objective in the form of maximum reward. The observational (unintervened) setting corresponds to $a_i=0$ for all $i$.

\subsubsection{Rosenbrock}
\label{appendix:Rosenbrock}
\begin{figure}[ht]
\centering
\includegraphics[scale=1]{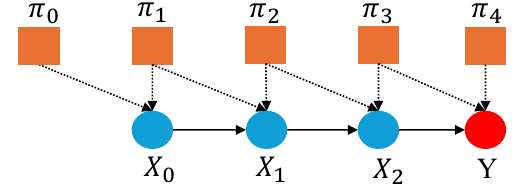} 
\caption{Rosenbrock function network (shown for $D=3$). Blue nodes indicate intermediate variables; orange squares represent action (intervention) variables; the red node is the target variable $Y$.}
\label{fig:Rosenbrock}
\end{figure}

A unimodal, non-separable function-network benchmark with a narrow curved valley. The network has action variables $a_0,\dots,a_{D-1}$, intermediate nodes $X_0,\dots,X_{D-2}$, and target node $Y$. The task is to \emph{maximize} $Y$. We use the domain $a_i \in [-2,2]$ and consider $D \in \{3,5,7\}$. Figure~\ref{fig:Rosenbrock} shows the chain-structured function network.

\paragraph{Action variables.}
$a_i \in [-2,2]$, \quad $i=0,\dots,D-1$.

\paragraph{Intermediate variables.}
The first intermediate node computes the first Rosenbrock term:
\begin{align*}
X_0 &= -100\bigl(a_1-a_0^2\bigr)^2 - (1-a_0)^2.
\end{align*}
Each subsequent node recursively accumulates the next pairwise contribution:
\begin{align*}
X_k &= -100\bigl(a_{k+1}-a_k^2\bigr)^2 - (1-a_k)^2 + X_{k-1},
\qquad k=1,\dots,D-2.
\end{align*}

\paragraph{Target variable.}
\begin{align*}
Y &= X_{D-2}.
\end{align*}

\paragraph{Interpretation.}
The Rosenbrock function network forms a causal chain in which the first node computes the initial Rosenbrock component, and each later node adds one additional pairwise contribution to the accumulated objective. The target $Y$ recovers the Rosenbrock objective in the form of maximum reward. The observational setting corresponds to $a_i=0$ for all $i$.
\subsubsection{Dropwave}
\label{appendix:Dropwave}
\begin{figure}[ht]
\centering
\includegraphics[scale=1]{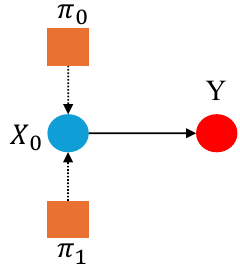} 
\caption{Dropwave function network. Blue nodes indicate intermediate variables; orange squares represent action (intervention) variables; the red node is the target variable $Y$.}
\label{fig:Dropwave}
\end{figure}

A highly multimodal two-dimensional function-network benchmark with two action variables $a_0$ and $a_1$, an intermediate node $X_0$ and a target node $Y$. The task is to \emph{maximize} $Y$. Following the BO setup of the standard function-network \citep{Sussex2022MCBO,ACBOSussex}, we use the domain $a_0,a_1 \in [-5.12,5.12]$. Figure~\ref{fig:Dropwave} shows the function network.

\paragraph{Action variables.}
$a_0, a_1 \in [-5.12,5.12]$.

\paragraph{Intermediate variable.}
\begin{align*}
X_0 &= \sqrt{a_0^2 + a_1^2} + \epsilon_0.
\end{align*}

\paragraph{Target variable.}
\begin{align*}
Y &= \frac{1+\cos(12X_0)}{2+0.5X_0^2} + \epsilon_Y.
\end{align*}

\paragraph{Interpretation.}
The intermediate node $X_0$ computes the radial distance from the origin induced by the two action variables. The target $Y$ applies the oscillatory Dropwave transformation to this radial value. Following MCBO and related function-network BO settings, this benchmark can also be evaluated in a noisy setting by choosing non-zero $\epsilon_0$ and $\epsilon_Y$.

\subsubsection{Alpine2}
\label{appendix:Alpine2}
\begin{figure}[ht]
\centering
\includegraphics[scale=1]{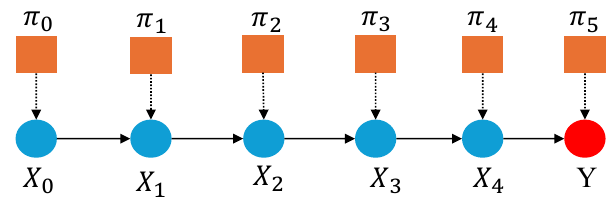} 
\caption{Alpine2 function network ($K=6$). Blue nodes indicate intermediate variables; orange squares represent action (intervention) variables; the red node is the target variable $Y$.}
\label{fig:Alpine2}
\end{figure}

A multimodal, separable chain-structured function-network benchmark with action variables $a_0,\dots,a_5$, intermediate nodes $X_0,\dots,X_4$, and target node $Y$. The task is to \emph{maximize} $Y$. We use the setting $K=6$ with the action domain $a_i \in [0,10]$ for $i=0,\dots,5$. Figure~\ref{fig:Alpine2} shows the function network.

\paragraph{Action variables.}
$a_i \in [0,10]$, \quad $i=0,\dots,5$.

\paragraph{Intermediate variables.}
The first intermediate node is defined as
\begin{align*}
X_0 &= -\sqrt{a_0}\sin(a_0).
\end{align*}
Each subsequent node combines its own action variable with the output of the previous node through the Alpine2 multiplicative structure:
\begin{align*}
X_k &= \sqrt{a_k}\sin(a_k)\,X_{k-1},
\qquad k=1,\dots,4.
\end{align*}

\paragraph{Target variable.}
\begin{align*}
Y &= \sqrt{a_5}\sin(a_5)\,X_4.
\end{align*}

\paragraph{Interpretation.}
Alpine2 forms a causal chain in which each stage combines its own decision variable with the output of the previous stage through a multiplicative structure. The analytical global maximum is $\prod_{i=0}^{5}\sqrt{a_i^*}\sin(a_i^*) = 2.808^6 \approx 400$ at $a_i^* \approx 7.917$ for all $i$. This subsection corresponds to the $K=6$ configuration used in our experiments.
\subsubsection{Chain-soft}
\label{appendix:Chain_soft}

\begin{figure}[ht]
\centering
\includegraphics[scale=0.8]{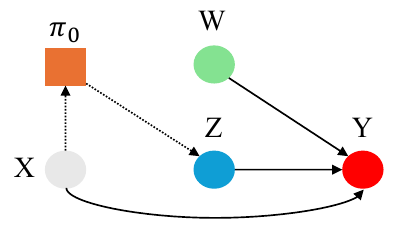} 
\caption{Chain SCM under soft intervention. Light-gray nodes are non-manipulable variables; green nodes indicate hard-intervenable variables; blue nodes indicate intermediate variables; the orange rectangle denotes the intervention policy $\pi_0$; the red node is the target variable $Y$.}
\label{fig:chain-soft}
\end{figure}

The soft-intervention Chain benchmark uses the same base SCM as Chain-hard (Appendix~\ref{appendix:Chain_hard}), with observed nodes $X, W, Z, Y$. Here, $X$ is a non-manipulable context variable, $W$ is hard-intervenable, $Z$ is \emph{soft}-intervenable via a context-dependent policy, and $Y$ is the target. The task is to \emph{minimize} $Y$. Figure~\ref{fig:chain-soft} illustrates the SCM with the intervention policy.

\paragraph{Exogenous variables.}
\begin{align*}
X &= U_X \sim \mathcal{N}(0,1), \\[4pt]
W &= U_W \sim \mathcal{N}(0,1).
\end{align*}

\paragraph{Endogenous variables (unintervened).}
\begin{align*}
Z &= -0.5X + U_Z, \\[4pt]
Y &= -W - 3ZX + U_Y,
\end{align*}
where $U_Z \sim \mathcal{N}(0,1)$ and $U_Y \sim \mathcal{N}(0,1)$, all mutually independent.

\paragraph{Intervention policy.}
Under soft intervention, the structural mechanism of $Z$ is replaced by a context-dependent policy $\pi_0$:
\[
Z := \pi_0(X),
\]
so that the intervened SCM becomes:
\begin{align*}
X &= U_X, \\[4pt]
W &= U_W \quad \text{or} \quad W := w,\; w \in [-1,1], \\[4pt]
Z &= \pi_0(X), \\[4pt]
Y &= -W - 3ZX + U_Y.
\end{align*}

\paragraph{Intervention domain and policy parameterization.}
The admissible scopes include a purely functional intervention in $Z$ (through the policy $\pi_0$) and a mixed intervention combining the functional intervention in $Z$ with a hard intervention in $W$ over $[-1,1]$. Following \citet{gultchin2023FCBO}, each functional intervention is represented using samples $\texttt{GridSize}=10$ and $N_{\alpha}=N_{\beta}=10$ for the context variable, with coefficients $\alpha_i$ and $\beta_j$ uniformly sampled from $[-0.27,0.27]$. This parameterization keeps the induced intervention values comparable to the hard-intervention range $[-1,1]$.

\paragraph{Interpretation.}
Because the target $Y = -W - 3ZX + U_Y$ depends on the interaction between $Z$ and the context $X$, the optimal intervention in $Z$ generally varies with $X$. This makes Chain-soft a key benchmark for evaluating context-adaptive intervention strategies. In the figure, the dashed arrows $X \dashrightarrow \pi_0 \dashrightarrow Z$ indicate that the policy takes $X$ as input and determines the intervened value of $Z$.

\section{Additional Results}
\label{appendix:additional_results}

\subsection{Robustness Stress Tests}
\label{appendix:robustness_stress_tests}

\subsubsection{Graph-misspecification stress test}
To go beyond a limitation statement, we added a graph-misspecification stress test. In this
experiment, the true environment is unchanged: interventional outcomes and reference optima are
always computed from the true benchmark SCM. Only the graph supplied to graph-using optimizers is
perturbed. Thus, the experiment tests sensitivity to incorrect causal information without changing
the true intervention task. We consider three perturbation types: edge addition, edge deletion, and
edge reversal. The results are shown in Tables~\ref{tab:misspec-mainshape-edge_add},
\ref{tab:misspec-mainshape-edge_delete}, and~\ref{tab:misspec-mainshape-edge_reverse}. Each
perturbation setting is evaluated over 20 random seeds, matching the seed count used in the main
benchmark. We present these stress tests as robustness diagnostics under controlled graph
misspecification, not as a complete deployment-robustness benchmark.

To make the perturbations concrete and reproducible, we visualize them directly on the benchmark
graphs. Figure~\ref{fig:misspec-schematic} illustrates the three perturbation types on the canonical
ToyGraph chain $X\to Z\to Y$, and Figures~\ref{fig:misspec-graphs-structured}
and~\ref{fig:misspec-graphs-real} show, for every other benchmark graph, the specific edge that is
added, deleted, or reversed and, where a hidden-confounder variant is defined, the node that is
removed from the learner's view. All panels share the visual encoding of Figure~\ref{fig:misspec-schematic}, with a hidden node drawn as an open dashed circle. In every panel exactly one edge (or, for the hidden-confounder
variant, one node) is perturbed relative to the true SCM, so the schematics also document the
minimal, single-perturbation nature of each stress test.

\begin{figure}[t]
\centering
\includegraphics[width=\linewidth]{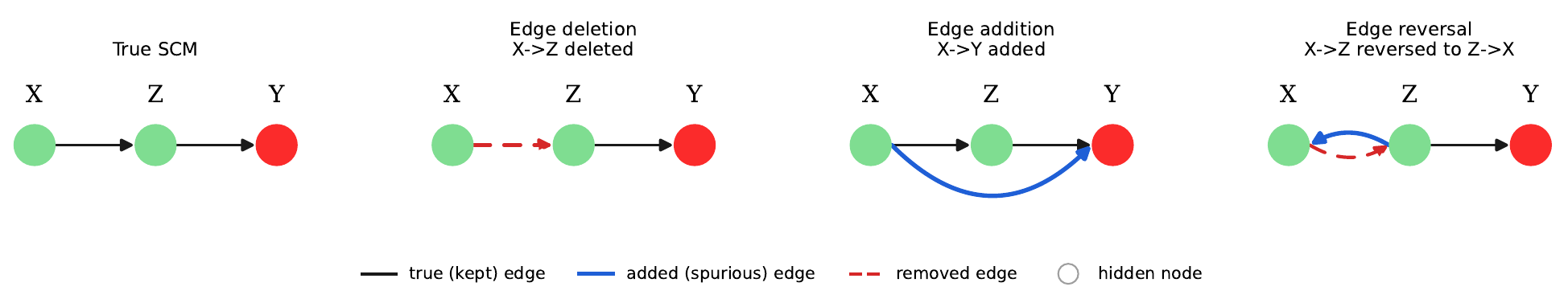}
\caption{Perturbation taxonomy on the ToyGraph SCM ($X\to Z\to Y$). From left to right: the true
benchmark graph, edge deletion ($X\to Z$ removed), edge addition (a spurious $X\to Y$ inserted), and
edge reversal ($X\to Z$ flipped to $Z\to X$). Green nodes are manipulable and the red node is the
target $Y$; kept edges are solid black, the spurious added edge is solid blue, and the deleted edge
is dashed red. Interventional outcomes and the reference optimum $y^\star$ are always computed from
the true SCM (leftmost panel); only the graph handed to the graph-using optimizers is perturbed.}
\label{fig:misspec-schematic}
\end{figure}

The stress test confirms the concern that motivated it: method rankings change substantially under graph
perturbation. For example, under edge addition at 100 trials, the GAP winner varies across datasets:
CoCaBO is best on ToyGraph, BO is best on Synthetic, DCBO is best on Synthetic-2, Ecology, and
Epidemiology, CBO is best on Chain-hard and Healthcare, and MCBO is best on
Protein-reconstructed (Table~\ref{tab:misspec-mainshape-edge_add}). Under edge deletion and edge
reversal, the winners change again (Tables~\ref{tab:misspec-mainshape-edge_delete}
and~\ref{tab:misspec-mainshape-edge_reverse}). PA-GAP also produces different rankings from GAP
because it rewards full-trajectory progress rather than only final discovery time. These results
support a more cautious conclusion: causal structure can help, but current CBO pipelines are
sensitive to how that structure is specified and operationalized. This particular stress test
addresses learner-side graph misspecification; hidden confounding is a separate dimension, which we
probe with a companion omitted-variable experiment (Table~\ref{tab:hidden-mainstyle}, below). We
treat comprehensive robustness to unobserved confounding as still open rather than resolved by these
pilots.

Some misspecified-graph runs outperform the corresponding correct-graph run. We do not interpret
this as evidence that the misspecified graph is causally preferable. Since outcomes and \(y^*\) are
still computed from the true SCM, such improvements reflect finite-budget search effects. A
perturbed graph can weaken an overconfident observational prior, change candidate intervention
scopes, alter exploration, or accidentally induce a more favorable acquisition trajectory. We
therefore report these cases as evidence that CBO is sensitive to causal inductive bias, not as
evidence that misspecification is beneficial.

To make the comparison interpretable, we compare each perturbed-graph score against the
corresponding correct-graph score obtained with the same dataset, method, metric, and budget.
Table~\ref{tab:misspec-degradation-summary} reports the resulting compact degradation
summary: for each perturbation type, the fraction of graph-using method cells (6 methods $\times$ 8
datasets $\times$ 2 metrics $\times$ 3 budgets, $N{=}288$) whose mean score degrades, improves, or
is unchanged relative to the correct-graph score in the main benchmark tables. Across all three
perturbation types, roughly half of the cells degrade (52--54\%), about a quarter improve
(26--27\%), and about a fifth are unchanged (20--21\%); CBO is the most consistently affected
pipeline (34--37 of 48 cells degrade per perturbation), while HCBO improves nearly as often as it
degrades. Because the stress-test pilots were executed as separate runs (visible in the shared BO
control, whose scores differ slightly from the main-table BO scores), these fractions should be
read as coarse sensitivity indicators rather than exact paired effect sizes. This summary prevents
individual improvements under misspecification from being misread as evidence that the perturbed
graph is causally better: improvements are a minority pattern and are consistent with the
finite-budget search effects described above.

\begin{table}[t]
\centering
\small
\caption{Degradation summary for the graph-misspecification stress tests. Each cell compares a
graph-using method's mean score under the perturbed graph against the corresponding correct-graph
mean score in the main benchmark tables, over 6 methods $\times$ 8 datasets $\times$ 2 metrics
$\times$ 3 budgets ($N{=}288$ per perturbation type).}
\label{tab:misspec-degradation-summary}
\begin{tabular}{lccc}
\toprule
Perturbation & Degrade & Improve & Unchanged \\
\midrule
Edge addition & 150 (52\%) & 78 (27\%) & 60 (21\%) \\
Edge deletion & 153 (53\%) & 76 (26\%) & 59 (20\%) \\
Edge reversal & 156 (54\%) & 75 (26\%) & 57 (20\%) \\
\bottomrule
\end{tabular}
\end{table}

For methods not directly affected by the supplied fixed graph in this pilot, the perturbation does
not change the optimizer. The non-causal BO baseline never reads the causal graph, and the CEO
runner used in this pilot does not consume a single fixed supplied graph in the same way as the
known-graph methods. We therefore treat BO and CEO as isolation controls and report a single shared
baseline for each dataset--budget--metric setting, identical across the edge-addition, edge-deletion,
and edge-reversal tables. Thus, cross-table differences are confined to the graph-using methods
(CBO, cCBO, DCBO, MCBO, HCBO, and CoCaBO).

\begin{table}[h]
\caption{Graph-misspecification stress test under edge addition (mean $\pm$ standard error over 20 seeds; illustrative pilot stress test; higher is better; best mean among methods under the same perturbation setting is in bold).}
\label{tab:misspec-mainshape-edge_add}
\centering
\scriptsize
\setlength{\tabcolsep}{2pt}
\renewcommand{\arraystretch}{1}
\resizebox{\textwidth}{!}{%
\begin{tabular}{c|c|c|cccccccc}
\toprule
Trial limit & Metric & Dataset & BO & CBO & cCBO & DCBO & MCBO & HCBO & CoCaBO & CEO \\
\midrule
\multirow{16}{*}{100}
& \multirow{8}{*}{GAP}
& ToyGraph & 0.232$\pm$.079 & 0.288$\pm$.071 & 0.567$\pm$.001 & 0.569$\pm$.055 & 0.558$\pm$.286 & 0.509$\pm$.049 & \textbf{0.682$\pm$.109} & 0.459$\pm$.084 \\
& & Synthetic & \textbf{0.651$\pm$.065} & 0.454$\pm$.072 & 0.642$\pm$.056 & 0.337$\pm$.036 & 0.456$\pm$.082 & 0.391$\pm$.088 & 0.617$\pm$.093 & 0.396$\pm$.042 \\
& & Synthetic-2 & 0.001$\pm$.004 & 0.281$\pm$.125 & 0.289$\pm$.128 & \textbf{0.578$\pm$.025} & 0.000$\pm$.000 & 0.472$\pm$.200 & 0.122$\pm$.076 & 0.526$\pm$.183 \\
& & Chain-hard & 0.787$\pm$.103 & \textbf{0.814$\pm$.128} & 0.618$\pm$.151 & 0.333$\pm$.333 & 0.284$\pm$.041 & 0.095$\pm$.062 & 0.786$\pm$.107 & 0.347$\pm$.046 \\
& & Ecology & 0.160$\pm$.128 & 0.018$\pm$.018 & 0.426$\pm$.068 & \textbf{0.519$\pm$.001} & 0.000$\pm$.000 & 0.234$\pm$.102 & 0.000$\pm$.000 & 0.228$\pm$.051 \\
& & Protein-reconstructed & 0.331$\pm$.125 & 0.363$\pm$.139 & 0.000$\pm$.000 & 0.000$\pm$.000 & \textbf{0.615$\pm$.097} & 0.288$\pm$.091 & 0.474$\pm$.045 & 0.193$\pm$.048 \\
& & Healthcare & 0.315$\pm$.054 & \textbf{0.510$\pm$.035} & 0.509$\pm$.083 & 0.341$\pm$.093 & 0.357$\pm$.029 & 0.234$\pm$.054 & 0.257$\pm$.129 & 0.483$\pm$.086 \\
& & Epidemiology & 0.506$\pm$.056 & 0.563$\pm$.038 & 0.000$\pm$.000 & \textbf{0.599$\pm$.008} & 0.000$\pm$.000 & 0.390$\pm$.077 & 0.300$\pm$.048 & 0.218$\pm$.081 \\
\cline{2-11}
& \multirow{8}{*}{PA-GAP}
& ToyGraph & 0.051$\pm$.039 & 0.089$\pm$.058 & 0.073$\pm$.001 & 0.088$\pm$.061 & 0.172$\pm$.122 & 0.199$\pm$.082 & \textbf{0.320$\pm$.032} & 0.203$\pm$.099 \\
& & Synthetic & 0.214$\pm$.051 & 0.041$\pm$.096 & 0.286$\pm$.027 & 0.025$\pm$.056 & 0.113$\pm$.082 & 0.075$\pm$.093 & \textbf{0.397$\pm$.033} & 0.253$\pm$.054 \\
& & Synthetic-2 & 0.005$\pm$.002 & 0.003$\pm$.002 & 0.139$\pm$.069 & 0.089$\pm$.029 & 0.000$\pm$.000 & 0.145$\pm$.077 & 0.090$\pm$.054 & \textbf{0.155$\pm$.057} \\
& & Chain-hard & \textbf{0.497$\pm$.015} & 0.437$\pm$.005 & 0.345$\pm$.129 & 0.168$\pm$.169 & 0.077$\pm$.014 & 0.066$\pm$.051 & 0.424$\pm$.069 & 0.007$\pm$.033 \\
& & Ecology & 0.031$\pm$.011 & 0.002$\pm$.002 & 0.174$\pm$.038 & 0.178$\pm$.001 & 0.000$\pm$.000 & \textbf{0.210$\pm$.120} & 0.000$\pm$.000 & 0.007$\pm$.069 \\
& & Protein-reconstructed & \textbf{0.153$\pm$.061} & 0.057$\pm$.018 & 0.000$\pm$.000 & 0.000$\pm$.000 & 0.031$\pm$.036 & 0.000$\pm$.000 & 0.051$\pm$.024 & 0.010$\pm$.072 \\
& & Healthcare & 0.120$\pm$.021 & 0.077$\pm$.015 & \textbf{0.202$\pm$.014} & 0.000$\pm$.000 & 0.071$\pm$.051 & 0.010$\pm$.006 & 0.116$\pm$.058 & 0.138$\pm$.049 \\
& & Epidemiology & 0.242$\pm$.021 & 0.185$\pm$.059 & 0.000$\pm$.000 & \textbf{0.341$\pm$.002} & 0.000$\pm$.000 & 0.138$\pm$.080 & 0.122$\pm$.078 & 0.001$\pm$.042 \\
\midrule
\multirow{16}{*}{50}
& \multirow{8}{*}{GAP}
& ToyGraph & 0.202$\pm$.099 & 0.235$\pm$.111 & 0.560$\pm$.009 & 0.545$\pm$.046 & 0.296$\pm$.296 & 0.450$\pm$.093 & \textbf{0.796$\pm$.041} & 0.489$\pm$.073 \\
& & Synthetic & 0.567$\pm$.091 & 0.000$\pm$.097 & 0.440$\pm$.048 & 0.295$\pm$.050 & 0.456$\pm$.099 & 0.391$\pm$.088 & \textbf{0.651$\pm$.028} & 0.254$\pm$.057 \\
& & Synthetic-2 & 0.007$\pm$.005 & 0.310$\pm$.033 & 0.377$\pm$.192 & \textbf{0.562$\pm$.017} & 0.000$\pm$.000 & 0.428$\pm$.217 & 0.122$\pm$.076 & 0.484$\pm$.222 \\
& & Chain-hard & 0.686$\pm$.086 & \textbf{0.889$\pm$.016} & 0.645$\pm$.198 & 0.333$\pm$.333 & 0.360$\pm$.107 & 0.095$\pm$.062 & 0.725$\pm$.143 & 0.356$\pm$.101 \\
& & Ecology & 0.259$\pm$.148 & 0.134$\pm$.134 & \textbf{0.502$\pm$.105} & 0.357$\pm$.001 & 0.000$\pm$.000 & 0.234$\pm$.102 & 0.000$\pm$.000 & 0.101$\pm$.031 \\
& & Protein-reconstructed & 0.365$\pm$.061 & 0.413$\pm$.046 & 0.000$\pm$.000 & 0.000$\pm$.000 & 0.526$\pm$.097 & \textbf{0.591$\pm$.040} & 0.392$\pm$.120 & 0.264$\pm$.028 \\
& & Healthcare & 0.439$\pm$.121 & 0.432$\pm$.073 & \textbf{0.604$\pm$.007} & 0.332$\pm$.093 & 0.357$\pm$.029 & 0.234$\pm$.054 & 0.244$\pm$.132 & 0.489$\pm$.076 \\
& & Epidemiology & 0.358$\pm$.052 & 0.355$\pm$.021 & 0.000$\pm$.000 & \textbf{0.671$\pm$.005} & 0.000$\pm$.000 & 0.390$\pm$.077 & 0.246$\pm$.079 & 0.060$\pm$.050 \\
\cline{2-11}
& \multirow{8}{*}{PA-GAP}
& ToyGraph & 0.024$\pm$.016 & 0.070$\pm$.053 & 0.073$\pm$.006 & 0.083$\pm$.057 & 0.108$\pm$.109 & 0.199$\pm$.082 & \textbf{0.321$\pm$.034} & 0.213$\pm$.099 \\
& & Synthetic & 0.194$\pm$.054 & 0.000$\pm$.036 & 0.201$\pm$.034 & 0.023$\pm$.050 & 0.113$\pm$.082 & 0.075$\pm$.093 & \textbf{0.344$\pm$.043} & 0.246$\pm$.077 \\
& & Synthetic-2 & 0.003$\pm$.003 & 0.003$\pm$.002 & 0.110$\pm$.056 & 0.085$\pm$.026 & 0.000$\pm$.000 & 0.081$\pm$.041 & 0.090$\pm$.054 & \textbf{0.146$\pm$.056} \\
& & Chain-hard & \textbf{0.481$\pm$.029} & 0.379$\pm$.010 & 0.327$\pm$.124 & 0.170$\pm$.170 & 0.050$\pm$.025 & 0.066$\pm$.051 & 0.391$\pm$.098 & 0.001$\pm$.054 \\
& & Ecology & 0.017$\pm$.010 & 0.002$\pm$.002 & 0.166$\pm$.040 & 0.175$\pm$.001 & 0.000$\pm$.000 & \textbf{0.208$\pm$.120} & 0.000$\pm$.000 & 0.006$\pm$.075 \\
& & Protein-reconstructed & \textbf{0.137$\pm$.054} & 0.047$\pm$.017 & 0.000$\pm$.000 & 0.000$\pm$.000 & 0.031$\pm$.036 & 0.000$\pm$.000 & 0.045$\pm$.019 & 0.009$\pm$.075 \\
& & Healthcare & 0.105$\pm$.017 & 0.069$\pm$.014 & \textbf{0.178$\pm$.010} & 0.000$\pm$.000 & 0.071$\pm$.051 & 0.009$\pm$.005 & 0.100$\pm$.051 & 0.147$\pm$.087 \\
& & Epidemiology & 0.232$\pm$.025 & 0.129$\pm$.061 & 0.000$\pm$.000 & \textbf{0.315$\pm$.003} & 0.000$\pm$.000 & 0.127$\pm$.073 & 0.085$\pm$.069 & 0.008$\pm$.002 \\
\midrule
\multirow{16}{*}{20}
& \multirow{8}{*}{GAP}
& ToyGraph & 0.272$\pm$.097 & 0.167$\pm$.064 & 0.538$\pm$.000 & 0.472$\pm$.034 & 0.239$\pm$.239 & 0.450$\pm$.093 & \textbf{0.763$\pm$.057} & 0.489$\pm$.073 \\
& & Synthetic & 0.390$\pm$.138 & 0.000$\pm$.000 & 0.475$\pm$.057 & 0.295$\pm$.050 & 0.456$\pm$.099 & 0.391$\pm$.088 & \textbf{0.590$\pm$.095} & 0.254$\pm$.057 \\
& & Synthetic-2 & 0.003$\pm$.002 & 0.185$\pm$.125 & \textbf{0.515$\pm$.009} & 0.129$\pm$.129 & 0.000$\pm$.000 & 0.428$\pm$.217 & 0.000$\pm$.000 & 0.484$\pm$.222 \\
& & Chain-hard & \textbf{0.761$\pm$.124} & 0.718$\pm$.039 & 0.486$\pm$.119 & 0.333$\pm$.333 & 0.380$\pm$.129 & 0.000$\pm$.000 & 0.731$\pm$.228 & 0.356$\pm$.101 \\
& & Ecology & 0.145$\pm$.145 & 0.080$\pm$.080 & 0.422$\pm$.159 & \textbf{0.442$\pm$.000} & 0.000$\pm$.000 & 0.234$\pm$.102 & 0.000$\pm$.000 & 0.101$\pm$.031 \\
& & Protein-reconstructed & 0.370$\pm$.073 & 0.194$\pm$.091 & 0.000$\pm$.000 & 0.000$\pm$.000 & 0.526$\pm$.097 & \textbf{0.556$\pm$.044} & 0.516$\pm$.027 & 0.264$\pm$.028 \\
& & Healthcare & 0.348$\pm$.119 & 0.268$\pm$.142 & 0.413$\pm$.036 & 0.332$\pm$.093 & 0.000$\pm$.000 & 0.234$\pm$.054 & 0.198$\pm$.100 & \textbf{0.489$\pm$.076} \\
& & Epidemiology & \textbf{0.694$\pm$.042} & 0.244$\pm$.122 & 0.000$\pm$.000 & 0.373$\pm$.005 & 0.000$\pm$.000 & 0.390$\pm$.077 & 0.259$\pm$.165 & 0.001$\pm$.005 \\
\cline{2-11}
& \multirow{8}{*}{PA-GAP}
& ToyGraph & 0.006$\pm$.004 & 0.042$\pm$.040 & 0.070$\pm$.000 & 0.069$\pm$.047 & 0.044$\pm$.044 & 0.199$\pm$.082 & \textbf{0.324$\pm$.039} & 0.213$\pm$.099 \\
& & Synthetic & 0.156$\pm$.059 & 0.000$\pm$.000 & 0.154$\pm$.036 & 0.023$\pm$.050 & 0.113$\pm$.082 & 0.000$\pm$.000 & 0.244$\pm$.054 & \textbf{0.246$\pm$.077} \\
& & Synthetic-2 & 0.004$\pm$.001 & 0.057$\pm$.029 & 0.075$\pm$.017 & 0.001$\pm$.001 & 0.000$\pm$.000 & 0.081$\pm$.041 & 0.000$\pm$.000 & \textbf{0.146$\pm$.056} \\
& & Chain-hard & \textbf{0.463$\pm$.062} & 0.228$\pm$.019 & 0.275$\pm$.112 & 0.175$\pm$.175 & 0.036$\pm$.022 & 0.000$\pm$.000 & 0.378$\pm$.108 & 0.001$\pm$.054 \\
& & Ecology & 0.011$\pm$.011 & 0.001$\pm$.001 & 0.150$\pm$.047 & 0.168$\pm$.001 & 0.000$\pm$.000 & \textbf{0.208$\pm$.120} & 0.000$\pm$.000 & 0.006$\pm$.075 \\
& & Protein-reconstructed & 0.113$\pm$.049 & 0.025$\pm$.013 & 0.000$\pm$.000 & \textbf{0.193$\pm$.097} & 0.031$\pm$.036 & 0.000$\pm$.000 & 0.045$\pm$.020 & 0.009$\pm$.075 \\
& & Healthcare & 0.091$\pm$.015 & 0.056$\pm$.011 & 0.119$\pm$.005 & 0.000$\pm$.000 & 0.000$\pm$.000 & 0.009$\pm$.005 & 0.074$\pm$.040 & \textbf{0.147$\pm$.087} \\
& & Epidemiology & 0.228$\pm$.028 & 0.087$\pm$.045 & 0.000$\pm$.000 & \textbf{0.279$\pm$.005} & 0.000$\pm$.000 & 0.127$\pm$.073 & 0.063$\pm$.060 & 0.007$\pm$.001 \\
\bottomrule
\end{tabular}%
}
\end{table}

\begin{table}[h]
\caption{Graph-misspecification stress test under edge deletion (mean $\pm$ standard error over 20 seeds; illustrative pilot stress test; higher is better; best mean among methods under the same perturbation setting is in bold).}
\label{tab:misspec-mainshape-edge_delete}
\centering
\scriptsize
\setlength{\tabcolsep}{2pt}
\renewcommand{\arraystretch}{1}
\resizebox{\textwidth}{!}{%
\begin{tabular}{c|c|c|cccccccc}
\toprule
Trial limit & Metric & Dataset & BO & CBO & cCBO & DCBO & MCBO & HCBO & CoCaBO & CEO \\
\midrule
\multirow{16}{*}{100}
& \multirow{8}{*}{GAP}
& ToyGraph & 0.232$\pm$.079 & 0.392$\pm$.009 & \textbf{0.826$\pm$.001} & 0.564$\pm$.060 & 0.394$\pm$.200 & 0.234$\pm$.054 & 0.335$\pm$.096 & 0.459$\pm$.084 \\
& & Synthetic & 0.651$\pm$.065 & 0.355$\pm$.184 & 0.642$\pm$.056 & 0.337$\pm$.036 & 0.456$\pm$.082 & 0.391$\pm$.088 & \textbf{0.663$\pm$.020} & 0.396$\pm$.042 \\
& & Synthetic-2 & 0.001$\pm$.004 & 0.281$\pm$.125 & 0.306$\pm$.145 & 0.000$\pm$.004 & 0.000$\pm$.000 & \textbf{0.657$\pm$.139} & 0.122$\pm$.076 & 0.526$\pm$.183 \\
& & Chain-hard & \textbf{0.787$\pm$.103} & 0.700$\pm$.128 & 0.767$\pm$.060 & 0.448$\pm$.300 & 0.428$\pm$.085 & 0.095$\pm$.062 & 0.688$\pm$.136 & 0.347$\pm$.046 \\
& & Ecology & 0.160$\pm$.128 & 0.012$\pm$.012 & \textbf{0.426$\pm$.068} & 0.255$\pm$.008 & 0.000$\pm$.000 & 0.234$\pm$.102 & 0.000$\pm$.000 & 0.228$\pm$.051 \\
& & Protein-reconstructed & 0.331$\pm$.125 & 0.363$\pm$.139 & 0.000$\pm$.000 & 0.000$\pm$.000 & \textbf{0.615$\pm$.097} & 0.288$\pm$.091 & 0.578$\pm$.082 & 0.193$\pm$.048 \\
& & Healthcare & 0.315$\pm$.054 & \textbf{0.510$\pm$.035} & 0.445$\pm$.100 & 0.341$\pm$.093 & 0.357$\pm$.029 & 0.234$\pm$.054 & 0.431$\pm$.216 & 0.483$\pm$.086 \\
& & Epidemiology & 0.506$\pm$.056 & 0.563$\pm$.038 & 0.000$\pm$.000 & \textbf{0.661$\pm$.092} & 0.000$\pm$.000 & 0.390$\pm$.077 & 0.572$\pm$.134 & 0.218$\pm$.081 \\
\cline{2-11}
& \multirow{8}{*}{PA-GAP}
& ToyGraph & 0.051$\pm$.039 & 0.104$\pm$.054 & \textbf{0.355$\pm$.001} & 0.095$\pm$.056 & 0.248$\pm$.129 & 0.199$\pm$.082 & 0.103$\pm$.051 & 0.203$\pm$.099 \\
& & Synthetic & 0.214$\pm$.051 & 0.087$\pm$.044 & 0.286$\pm$.027 & 0.025$\pm$.056 & 0.113$\pm$.082 & 0.075$\pm$.093 & \textbf{0.396$\pm$.039} & 0.253$\pm$.054 \\
& & Synthetic-2 & 0.005$\pm$.002 & 0.003$\pm$.002 & 0.133$\pm$.069 & 0.000$\pm$.006 & 0.000$\pm$.000 & \textbf{0.178$\pm$.131} & 0.090$\pm$.054 & 0.155$\pm$.057 \\
& & Chain-hard & \textbf{0.497$\pm$.015} & 0.437$\pm$.005 & 0.414$\pm$.030 & 0.221$\pm$.147 & 0.085$\pm$.029 & 0.066$\pm$.051 & 0.388$\pm$.098 & 0.007$\pm$.033 \\
& & Ecology & 0.031$\pm$.011 & 0.000$\pm$.005 & 0.174$\pm$.038 & 0.041$\pm$.008 & 0.000$\pm$.000 & \textbf{0.210$\pm$.120} & 0.000$\pm$.000 & 0.007$\pm$.069 \\
& & Protein-reconstructed & \textbf{0.153$\pm$.061} & 0.057$\pm$.018 & 0.000$\pm$.000 & 0.000$\pm$.000 & 0.031$\pm$.036 & 0.000$\pm$.000 & 0.094$\pm$.077 & 0.010$\pm$.072 \\
& & Healthcare & 0.120$\pm$.021 & 0.077$\pm$.015 & \textbf{0.158$\pm$.052} & 0.000$\pm$.000 & 0.071$\pm$.051 & 0.010$\pm$.006 & 0.134$\pm$.068 & 0.138$\pm$.049 \\
& & Epidemiology & 0.242$\pm$.021 & 0.185$\pm$.059 & 0.000$\pm$.000 & \textbf{0.273$\pm$.038} & 0.000$\pm$.000 & 0.138$\pm$.080 & 0.227$\pm$.087 & 0.001$\pm$.042 \\
\midrule
\multirow{16}{*}{50}
& \multirow{8}{*}{GAP}
& ToyGraph & 0.202$\pm$.099 & 0.363$\pm$.141 & \textbf{0.780$\pm$.009} & 0.525$\pm$.062 & 0.569$\pm$.288 & 0.450$\pm$.093 & 0.195$\pm$.059 & 0.489$\pm$.073 \\
& & Synthetic & 0.567$\pm$.091 & 0.271$\pm$.177 & 0.440$\pm$.048 & 0.295$\pm$.050 & 0.456$\pm$.099 & 0.391$\pm$.088 & \textbf{0.688$\pm$.082} & 0.254$\pm$.057 \\
& & Synthetic-2 & 0.007$\pm$.005 & 0.310$\pm$.033 & 0.229$\pm$.118 & 0.000$\pm$.002 & 0.000$\pm$.000 & 0.467$\pm$.251 & 0.122$\pm$.076 & \textbf{0.484$\pm$.222} \\
& & Chain-hard & 0.686$\pm$.086 & \textbf{0.889$\pm$.016} & 0.741$\pm$.151 & 0.448$\pm$.300 & 0.346$\pm$.109 & 0.095$\pm$.062 & 0.601$\pm$.165 & 0.356$\pm$.101 \\
& & Ecology & 0.259$\pm$.148 & 0.000$\pm$.005 & \textbf{0.502$\pm$.105} & 0.158$\pm$.008 & 0.000$\pm$.000 & 0.234$\pm$.102 & 0.000$\pm$.000 & 0.101$\pm$.031 \\
& & Protein-reconstructed & 0.365$\pm$.061 & 0.413$\pm$.046 & 0.000$\pm$.000 & 0.000$\pm$.000 & 0.526$\pm$.097 & \textbf{0.591$\pm$.040} & 0.562$\pm$.088 & 0.264$\pm$.028 \\
& & Healthcare & 0.439$\pm$.121 & 0.432$\pm$.073 & \textbf{0.516$\pm$.102} & 0.332$\pm$.093 & 0.357$\pm$.029 & 0.234$\pm$.054 & 0.380$\pm$.197 & 0.489$\pm$.076 \\
& & Epidemiology & 0.358$\pm$.052 & 0.355$\pm$.021 & 0.000$\pm$.000 & \textbf{0.687$\pm$.032} & 0.000$\pm$.000 & 0.390$\pm$.077 & 0.482$\pm$.140 & 0.060$\pm$.050 \\
\cline{2-11}
& \multirow{8}{*}{PA-GAP}
& ToyGraph & 0.024$\pm$.016 & 0.093$\pm$.045 & \textbf{0.338$\pm$.006} & 0.088$\pm$.054 & 0.180$\pm$.110 & 0.199$\pm$.082 & 0.074$\pm$.036 & 0.213$\pm$.099 \\
& & Synthetic & 0.194$\pm$.054 & 0.076$\pm$.038 & 0.201$\pm$.034 & 0.023$\pm$.050 & 0.113$\pm$.082 & 0.075$\pm$.093 & \textbf{0.346$\pm$.055} & 0.246$\pm$.077 \\
& & Synthetic-2 & 0.003$\pm$.003 & 0.003$\pm$.002 & 0.104$\pm$.058 & 0.000$\pm$.005 & 0.000$\pm$.000 & \textbf{0.158$\pm$.137} & 0.090$\pm$.054 & 0.146$\pm$.056 \\
& & Chain-hard & \textbf{0.481$\pm$.029} & 0.379$\pm$.010 & 0.387$\pm$.030 & 0.222$\pm$.148 & 0.062$\pm$.030 & 0.066$\pm$.051 & 0.380$\pm$.099 & 0.001$\pm$.054 \\
& & Ecology & 0.017$\pm$.010 & 0.000$\pm$.002 & 0.166$\pm$.040 & 0.031$\pm$.010 & 0.000$\pm$.000 & \textbf{0.210$\pm$.120} & 0.000$\pm$.000 & 0.006$\pm$.075 \\
& & Protein-reconstructed & \textbf{0.137$\pm$.054} & 0.047$\pm$.017 & 0.000$\pm$.000 & 0.000$\pm$.000 & 0.031$\pm$.036 & 0.000$\pm$.000 & 0.095$\pm$.078 & 0.009$\pm$.075 \\
& & Healthcare & 0.105$\pm$.017 & 0.069$\pm$.014 & 0.141$\pm$.044 & 0.000$\pm$.000 & 0.071$\pm$.051 & 0.009$\pm$.005 & 0.122$\pm$.061 & \textbf{0.147$\pm$.087} \\
& & Epidemiology & 0.232$\pm$.025 & 0.129$\pm$.061 & 0.000$\pm$.000 & \textbf{0.258$\pm$.033} & 0.000$\pm$.000 & 0.127$\pm$.073 & 0.196$\pm$.085 & 0.008$\pm$.002 \\
\midrule
\multirow{16}{*}{20}
& \multirow{8}{*}{GAP}
& ToyGraph & 0.272$\pm$.097 & 0.375$\pm$.145 & \textbf{0.639$\pm$.000} & 0.405$\pm$.074 & 0.274$\pm$.274 & 0.450$\pm$.093 & 0.097$\pm$.039 & 0.489$\pm$.073 \\
& & Synthetic & 0.390$\pm$.138 & 0.276$\pm$.163 & 0.475$\pm$.057 & 0.295$\pm$.050 & 0.456$\pm$.099 & 0.391$\pm$.088 & \textbf{0.511$\pm$.058} & 0.254$\pm$.057 \\
& & Synthetic-2 & 0.003$\pm$.002 & 0.245$\pm$.125 & 0.000$\pm$.000 & \textbf{0.507$\pm$.026} & 0.000$\pm$.000 & 0.467$\pm$.251 & 0.000$\pm$.000 & 0.484$\pm$.222 \\
& & Chain-hard & 0.761$\pm$.124 & 0.718$\pm$.039 & \textbf{0.778$\pm$.019} & 0.318$\pm$.318 & 0.250$\pm$.077 & 0.000$\pm$.000 & 0.671$\pm$.202 & 0.356$\pm$.101 \\
& & Ecology & 0.145$\pm$.145 & 0.000$\pm$.000 & \textbf{0.422$\pm$.159} & 0.221$\pm$.147 & 0.000$\pm$.000 & 0.234$\pm$.102 & 0.000$\pm$.000 & 0.101$\pm$.031 \\
& & Protein-reconstructed & 0.370$\pm$.073 & 0.194$\pm$.091 & 0.000$\pm$.000 & 0.000$\pm$.000 & 0.526$\pm$.097 & \textbf{0.556$\pm$.044} & 0.515$\pm$.109 & 0.264$\pm$.028 \\
& & Healthcare & 0.348$\pm$.119 & 0.268$\pm$.142 & 0.413$\pm$.036 & 0.332$\pm$.093 & 0.000$\pm$.000 & 0.234$\pm$.054 & 0.317$\pm$.180 & \textbf{0.489$\pm$.076} \\
& & Epidemiology & \textbf{0.694$\pm$.042} & 0.244$\pm$.122 & 0.000$\pm$.000 & 0.514$\pm$.079 & 0.000$\pm$.000 & 0.390$\pm$.077 & 0.483$\pm$.102 & 0.001$\pm$.005 \\
\cline{2-11}
& \multirow{8}{*}{PA-GAP}
& ToyGraph & 0.006$\pm$.004 & 0.066$\pm$.025 & \textbf{0.296$\pm$.000} & 0.068$\pm$.047 & 0.076$\pm$.076 & 0.199$\pm$.082 & 0.047$\pm$.034 & 0.213$\pm$.099 \\
& & Synthetic & 0.156$\pm$.059 & 0.063$\pm$.034 & 0.154$\pm$.036 & 0.023$\pm$.050 & 0.113$\pm$.082 & 0.000$\pm$.000 & \textbf{0.286$\pm$.062} & 0.246$\pm$.077 \\
& & Synthetic-2 & 0.004$\pm$.001 & 0.068$\pm$.037 & 0.000$\pm$.000 & \textbf{0.158$\pm$.017} & 0.000$\pm$.000 & \textbf{0.158$\pm$.137} & 0.000$\pm$.000 & 0.146$\pm$.056 \\
& & Chain-hard & \textbf{0.463$\pm$.062} & 0.228$\pm$.019 & 0.312$\pm$.031 & 0.160$\pm$.160 & 0.029$\pm$.017 & 0.000$\pm$.000 & 0.368$\pm$.102 & 0.001$\pm$.054 \\
& & Ecology & 0.011$\pm$.011 & 0.000$\pm$.000 & 0.150$\pm$.047 & 0.017$\pm$.016 & 0.000$\pm$.000 & \textbf{0.210$\pm$.120} & 0.000$\pm$.000 & 0.006$\pm$.075 \\
& & Protein-reconstructed & 0.113$\pm$.049 & 0.025$\pm$.013 & 0.000$\pm$.000 & \textbf{0.193$\pm$.097} & 0.031$\pm$.036 & 0.000$\pm$.000 & 0.096$\pm$.080 & 0.009$\pm$.075 \\
& & Healthcare & 0.091$\pm$.015 & 0.056$\pm$.011 & 0.119$\pm$.005 & 0.000$\pm$.000 & 0.000$\pm$.000 & 0.009$\pm$.005 & 0.102$\pm$.052 & \textbf{0.147$\pm$.087} \\
& & Epidemiology & 0.228$\pm$.028 & 0.087$\pm$.045 & 0.000$\pm$.000 & \textbf{0.241$\pm$.013} & 0.000$\pm$.000 & 0.127$\pm$.073 & 0.143$\pm$.067 & 0.007$\pm$.001 \\
\bottomrule
\end{tabular}%
}
\end{table}

\begin{table}[h]
\caption{Graph-misspecification stress test under edge reversal (mean $\pm$ standard error over 20 seeds; illustrative pilot stress test; higher is better; best mean among methods under the same perturbation setting is in bold).}
\label{tab:misspec-mainshape-edge_reverse}
\centering
\scriptsize
\setlength{\tabcolsep}{2pt}
\renewcommand{\arraystretch}{1}
\resizebox{\textwidth}{!}{%
\begin{tabular}{c|c|c|cccccccc}
\toprule
Trial limit & Metric & Dataset & BO & CBO & cCBO & DCBO & MCBO & HCBO & CoCaBO & CEO \\
\midrule
\multirow{16}{*}{100}
& \multirow{8}{*}{GAP}
& ToyGraph & 0.232$\pm$.079 & 0.333$\pm$.118 & \textbf{0.826$\pm$.001} & 0.564$\pm$.060 & 0.513$\pm$.257 & 0.358$\pm$.054 & 0.279$\pm$.074 & 0.459$\pm$.084 \\
& & Synthetic & \textbf{0.651$\pm$.065} & 0.340$\pm$.084 & 0.548$\pm$.015 & 0.337$\pm$.036 & 0.456$\pm$.082 & 0.391$\pm$.088 & 0.589$\pm$.042 & 0.396$\pm$.042 \\
& & Synthetic-2 & 0.001$\pm$.004 & 0.281$\pm$.125 & 0.335$\pm$.114 & 0.000$\pm$.004 & 0.000$\pm$.000 & \textbf{0.757$\pm$.012} & 0.122$\pm$.076 & 0.526$\pm$.183 \\
& & Chain-hard & 0.787$\pm$.103 & \textbf{0.874$\pm$.051} & 0.632$\pm$.118 & 0.295$\pm$.050 & 0.409$\pm$.093 & 0.095$\pm$.062 & 0.780$\pm$.071 & 0.347$\pm$.046 \\
& & Ecology & 0.160$\pm$.128 & 0.012$\pm$.012 & \textbf{0.426$\pm$.068} & 0.255$\pm$.008 & 0.000$\pm$.000 & 0.234$\pm$.102 & 0.000$\pm$.000 & 0.228$\pm$.051 \\
& & Protein-reconstructed & 0.331$\pm$.125 & 0.363$\pm$.139 & 0.000$\pm$.000 & 0.000$\pm$.000 & \textbf{0.615$\pm$.097} & 0.288$\pm$.091 & 0.423$\pm$.096 & 0.193$\pm$.048 \\
& & Healthcare & 0.315$\pm$.054 & \textbf{0.561$\pm$.004} & 0.509$\pm$.083 & 0.341$\pm$.093 & 0.357$\pm$.029 & 0.234$\pm$.054 & 0.432$\pm$.104 & 0.483$\pm$.086 \\
& & Epidemiology & 0.506$\pm$.056 & \textbf{0.533$\pm$.134} & 0.000$\pm$.000 & 0.456$\pm$.029 & 0.000$\pm$.000 & 0.390$\pm$.077 & 0.464$\pm$.200 & 0.218$\pm$.081 \\
\cline{2-11}
& \multirow{8}{*}{PA-GAP}
& ToyGraph & 0.051$\pm$.039 & 0.102$\pm$.054 & \textbf{0.355$\pm$.001} & 0.095$\pm$.056 & 0.271$\pm$.141 & 0.199$\pm$.082 & 0.139$\pm$.079 & 0.203$\pm$.099 \\
& & Synthetic & 0.214$\pm$.051 & 0.141$\pm$.036 & 0.257$\pm$.045 & 0.025$\pm$.056 & 0.113$\pm$.082 & 0.075$\pm$.093 & \textbf{0.376$\pm$.027} & 0.253$\pm$.054 \\
& & Synthetic-2 & 0.005$\pm$.002 & 0.003$\pm$.002 & 0.139$\pm$.070 & 0.000$\pm$.006 & 0.000$\pm$.000 & \textbf{0.213$\pm$.043} & 0.090$\pm$.054 & 0.155$\pm$.057 \\
& & Chain-hard & \textbf{0.497$\pm$.015} & 0.434$\pm$.006 & 0.325$\pm$.099 & 0.087$\pm$.050 & 0.087$\pm$.001 & 0.066$\pm$.051 & 0.454$\pm$.043 & 0.007$\pm$.033 \\
& & Ecology & 0.031$\pm$.011 & 0.000$\pm$.005 & 0.174$\pm$.038 & 0.041$\pm$.008 & 0.000$\pm$.000 & \textbf{0.210$\pm$.120} & 0.000$\pm$.000 & 0.007$\pm$.069 \\
& & Protein-reconstructed & \textbf{0.153$\pm$.061} & 0.057$\pm$.018 & 0.000$\pm$.000 & 0.000$\pm$.000 & 0.031$\pm$.036 & 0.000$\pm$.000 & 0.018$\pm$.009 & 0.010$\pm$.072 \\
& & Healthcare & 0.120$\pm$.021 & 0.090$\pm$.014 & \textbf{0.202$\pm$.014} & 0.000$\pm$.000 & 0.071$\pm$.051 & 0.010$\pm$.006 & 0.196$\pm$.010 & 0.138$\pm$.049 \\
& & Epidemiology & 0.242$\pm$.021 & 0.173$\pm$.077 & 0.000$\pm$.000 & 0.065$\pm$.074 & 0.000$\pm$.000 & 0.138$\pm$.080 & \textbf{0.251$\pm$.094} & 0.001$\pm$.042 \\
\midrule
\multirow{16}{*}{50}
& \multirow{8}{*}{GAP}
& ToyGraph & 0.202$\pm$.099 & 0.477$\pm$.097 & \textbf{0.780$\pm$.009} & 0.525$\pm$.062 & 0.411$\pm$.206 & 0.450$\pm$.093 & 0.276$\pm$.075 & 0.489$\pm$.073 \\
& & Synthetic & 0.567$\pm$.091 & 0.534$\pm$.087 & 0.425$\pm$.086 & 0.295$\pm$.050 & 0.456$\pm$.099 & 0.391$\pm$.088 & \textbf{0.569$\pm$.118} & 0.254$\pm$.057 \\
& & Synthetic-2 & 0.007$\pm$.005 & 0.310$\pm$.033 & 0.243$\pm$.127 & 0.000$\pm$.002 & 0.000$\pm$.000 & \textbf{0.536$\pm$.099} & 0.122$\pm$.076 & 0.484$\pm$.222 \\
& & Chain-hard & 0.686$\pm$.086 & \textbf{0.747$\pm$.102} & 0.411$\pm$.128 & 0.295$\pm$.050 & 0.383$\pm$.096 & 0.095$\pm$.062 & 0.744$\pm$.062 & 0.356$\pm$.101 \\
& & Ecology & 0.259$\pm$.148 & 0.000$\pm$.005 & \textbf{0.502$\pm$.105} & 0.158$\pm$.008 & 0.000$\pm$.000 & 0.234$\pm$.102 & 0.000$\pm$.000 & 0.101$\pm$.031 \\
& & Protein-reconstructed & 0.365$\pm$.061 & 0.413$\pm$.046 & 0.000$\pm$.000 & 0.000$\pm$.000 & 0.526$\pm$.097 & \textbf{0.591$\pm$.040} & 0.342$\pm$.170 & 0.264$\pm$.028 \\
& & Healthcare & 0.439$\pm$.121 & 0.522$\pm$.020 & \textbf{0.604$\pm$.007} & 0.332$\pm$.093 & 0.357$\pm$.029 & 0.234$\pm$.054 & 0.342$\pm$.077 & 0.489$\pm$.076 \\
& & Epidemiology & 0.358$\pm$.052 & 0.450$\pm$.170 & 0.000$\pm$.000 & 0.456$\pm$.039 & 0.000$\pm$.000 & 0.390$\pm$.077 & \textbf{0.624$\pm$.150} & 0.060$\pm$.050 \\
\cline{2-11}
& \multirow{8}{*}{PA-GAP}
& ToyGraph & 0.024$\pm$.016 & 0.089$\pm$.047 & \textbf{0.338$\pm$.006} & 0.088$\pm$.054 & 0.225$\pm$.134 & 0.199$\pm$.082 & 0.131$\pm$.073 & 0.213$\pm$.099 \\
& & Synthetic & 0.194$\pm$.054 & 0.123$\pm$.034 & 0.214$\pm$.051 & 0.023$\pm$.050 & 0.113$\pm$.082 & 0.075$\pm$.093 & \textbf{0.328$\pm$.036} & 0.246$\pm$.077 \\
& & Synthetic-2 & 0.003$\pm$.003 & 0.003$\pm$.002 & 0.110$\pm$.059 & 0.000$\pm$.005 & 0.000$\pm$.000 & 0.104$\pm$.055 & 0.090$\pm$.054 & \textbf{0.146$\pm$.056} \\
& & Chain-hard & \textbf{0.481$\pm$.029} & 0.373$\pm$.012 & 0.298$\pm$.088 & 0.087$\pm$.050 & 0.060$\pm$.010 & 0.066$\pm$.051 & 0.445$\pm$.051 & 0.001$\pm$.054 \\
& & Ecology & 0.017$\pm$.010 & 0.000$\pm$.002 & 0.166$\pm$.040 & 0.031$\pm$.010 & 0.000$\pm$.000 & \textbf{0.210$\pm$.120} & 0.000$\pm$.000 & 0.006$\pm$.075 \\
& & Protein-reconstructed & \textbf{0.137$\pm$.054} & 0.047$\pm$.017 & 0.000$\pm$.000 & 0.000$\pm$.000 & 0.031$\pm$.036 & 0.000$\pm$.000 & 0.018$\pm$.009 & 0.009$\pm$.075 \\
& & Healthcare & 0.105$\pm$.017 & 0.082$\pm$.010 & 0.178$\pm$.010 & 0.000$\pm$.000 & 0.071$\pm$.051 & 0.010$\pm$.006 & \textbf{0.180$\pm$.008} & 0.147$\pm$.087 \\
& & Epidemiology & \textbf{0.232$\pm$.025} & 0.150$\pm$.083 & 0.000$\pm$.000 & 0.065$\pm$.098 & 0.000$\pm$.000 & 0.138$\pm$.080 & 0.225$\pm$.093 & 0.008$\pm$.002 \\
\midrule
\multirow{16}{*}{20}
& \multirow{8}{*}{GAP}
& ToyGraph & 0.272$\pm$.097 & 0.264$\pm$.152 & \textbf{0.639$\pm$.000} & 0.405$\pm$.074 & 0.496$\pm$.275 & 0.450$\pm$.093 & 0.167$\pm$.073 & 0.489$\pm$.073 \\
& & Synthetic & 0.390$\pm$.138 & 0.350$\pm$.208 & \textbf{0.518$\pm$.156} & 0.295$\pm$.050 & 0.456$\pm$.099 & 0.391$\pm$.088 & 0.506$\pm$.067 & 0.254$\pm$.057 \\
& & Synthetic-2 & 0.003$\pm$.002 & 0.313$\pm$.163 & 0.000$\pm$.000 & 0.507$\pm$.026 & 0.000$\pm$.000 & \textbf{0.536$\pm$.099} & 0.000$\pm$.000 & 0.484$\pm$.222 \\
& & Chain-hard & 0.761$\pm$.124 & 0.650$\pm$.090 & 0.691$\pm$.051 & 0.866$\pm$.021 & 0.249$\pm$.090 & 0.000$\pm$.000 & \textbf{0.905$\pm$.066} & 0.356$\pm$.101 \\
& & Ecology & 0.145$\pm$.145 & 0.000$\pm$.000 & \textbf{0.422$\pm$.159} & 0.221$\pm$.147 & 0.000$\pm$.000 & 0.234$\pm$.102 & 0.000$\pm$.000 & 0.101$\pm$.031 \\
& & Protein-reconstructed & 0.370$\pm$.073 & 0.194$\pm$.091 & 0.000$\pm$.000 & 0.000$\pm$.000 & 0.526$\pm$.097 & \textbf{0.556$\pm$.044} & 0.326$\pm$.163 & 0.264$\pm$.028 \\
& & Healthcare & 0.348$\pm$.119 & 0.401$\pm$.079 & 0.413$\pm$.036 & 0.332$\pm$.093 & 0.000$\pm$.000 & 0.234$\pm$.054 & 0.407$\pm$.108 & \textbf{0.489$\pm$.076} \\
& & Epidemiology & \textbf{0.694$\pm$.042} & 0.378$\pm$.226 & 0.000$\pm$.000 & 0.105$\pm$.053 & 0.000$\pm$.000 & 0.390$\pm$.077 & 0.541$\pm$.079 & 0.001$\pm$.005 \\
\cline{2-11}
& \multirow{8}{*}{PA-GAP}
& ToyGraph & 0.006$\pm$.004 & 0.056$\pm$.028 & \textbf{0.296$\pm$.000} & 0.068$\pm$.047 & 0.161$\pm$.157 & 0.199$\pm$.082 & 0.113$\pm$.059 & 0.213$\pm$.099 \\
& & Synthetic & 0.156$\pm$.059 & 0.088$\pm$.038 & 0.174$\pm$.061 & 0.023$\pm$.050 & 0.113$\pm$.082 & 0.000$\pm$.000 & \textbf{0.259$\pm$.047} & 0.246$\pm$.077 \\
& & Synthetic-2 & 0.004$\pm$.001 & 0.065$\pm$.036 & 0.000$\pm$.000 & \textbf{0.158$\pm$.017} & 0.000$\pm$.000 & 0.104$\pm$.055 & 0.000$\pm$.000 & 0.146$\pm$.056 \\
& & Chain-hard & \textbf{0.463$\pm$.062} & 0.218$\pm$.023 & 0.231$\pm$.057 & 0.414$\pm$.021 & 0.024$\pm$.020 & 0.000$\pm$.000 & 0.426$\pm$.069 & 0.001$\pm$.054 \\
& & Ecology & 0.011$\pm$.011 & 0.000$\pm$.000 & 0.150$\pm$.047 & 0.017$\pm$.016 & 0.000$\pm$.000 & \textbf{0.210$\pm$.120} & 0.000$\pm$.000 & 0.006$\pm$.075 \\
& & Protein-reconstructed & 0.113$\pm$.049 & 0.025$\pm$.013 & 0.000$\pm$.000 & \textbf{0.193$\pm$.097} & 0.031$\pm$.036 & 0.000$\pm$.000 & 0.017$\pm$.009 & 0.009$\pm$.075 \\
& & Healthcare & 0.091$\pm$.015 & 0.064$\pm$.006 & 0.119$\pm$.005 & 0.000$\pm$.000 & 0.000$\pm$.000 & 0.010$\pm$.006 & \textbf{0.163$\pm$.005} & 0.147$\pm$.087 \\
& & Epidemiology & \textbf{0.228$\pm$.028} & 0.139$\pm$.087 & 0.000$\pm$.000 & 0.003$\pm$.001 & 0.000$\pm$.000 & 0.138$\pm$.080 & 0.176$\pm$.082 & 0.007$\pm$.001 \\
\bottomrule
\end{tabular}%
}
\end{table}
\FloatBarrier

\begin{figure}[h]
\centering
{\bfseries Synthetic}\par\smallskip
\includegraphics[width=\linewidth]{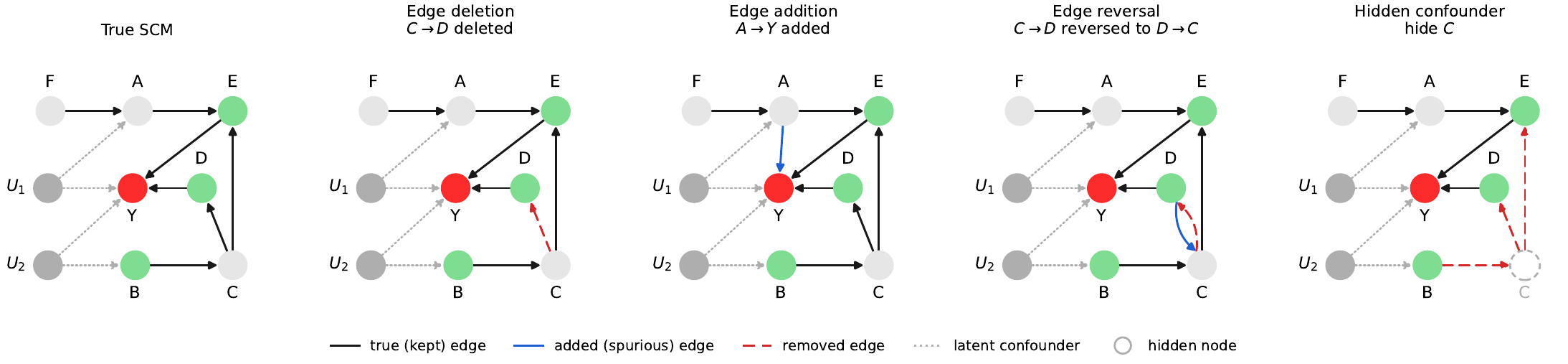}\par\bigskip
{\bfseries Synthetic-2}\par\smallskip
\includegraphics[width=\linewidth]{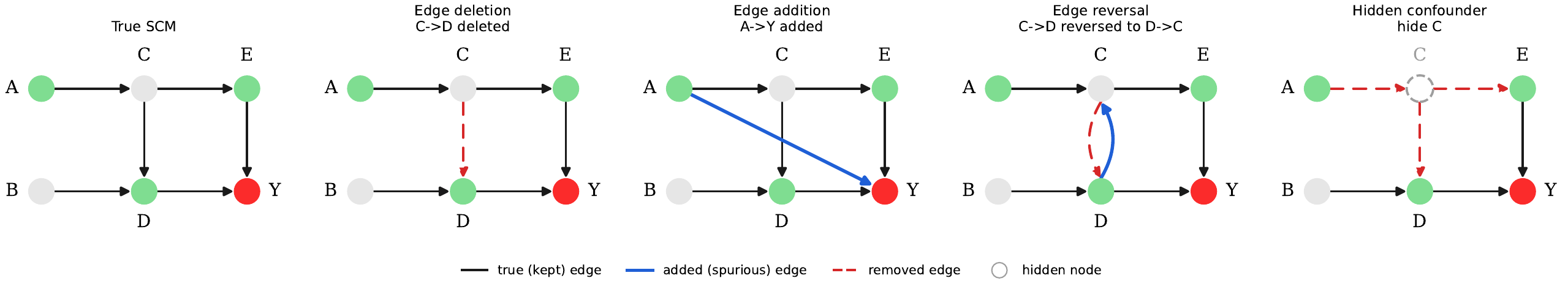}\par\bigskip
{\bfseries Chain-hard}\par\smallskip
\includegraphics[width=\linewidth]{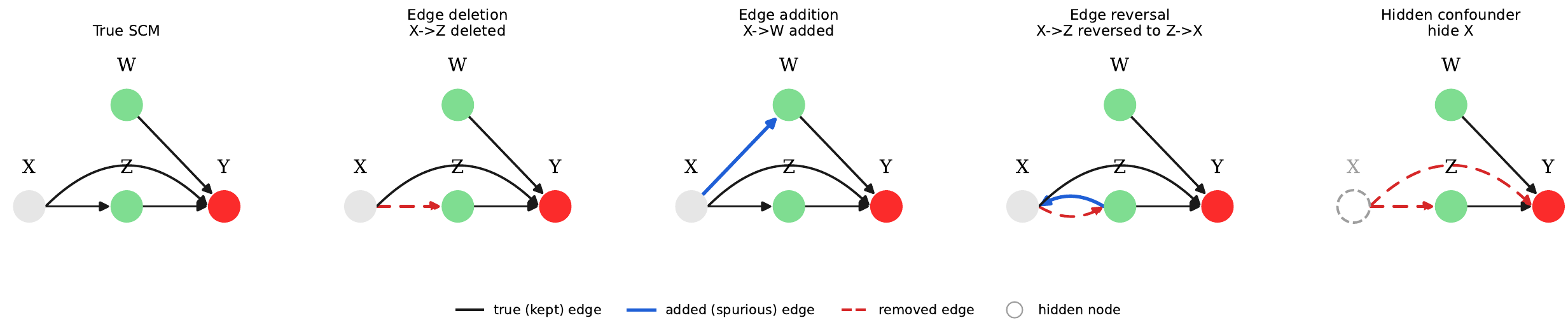}\par\bigskip
\caption{Graph-misspecification schematics for the synthetic SCMs (Synthetic, Synthetic-2, and Chain-hard). Visual encoding follows Figure~\ref{fig:misspec-schematic}: each row shows the true SCM followed by learner-belief graphs under the edge-deletion, edge-addition, and edge-reversal perturbations named in the panel titles. Only the learner's supplied graph is perturbed; the true SCM and $y^\star$ are unchanged.}
\label{fig:misspec-graphs-structured}
\end{figure}

\begin{figure}[h]
\centering
{\bfseries Protein-reconstructed}\par\smallskip
\includegraphics[width=.8\linewidth]{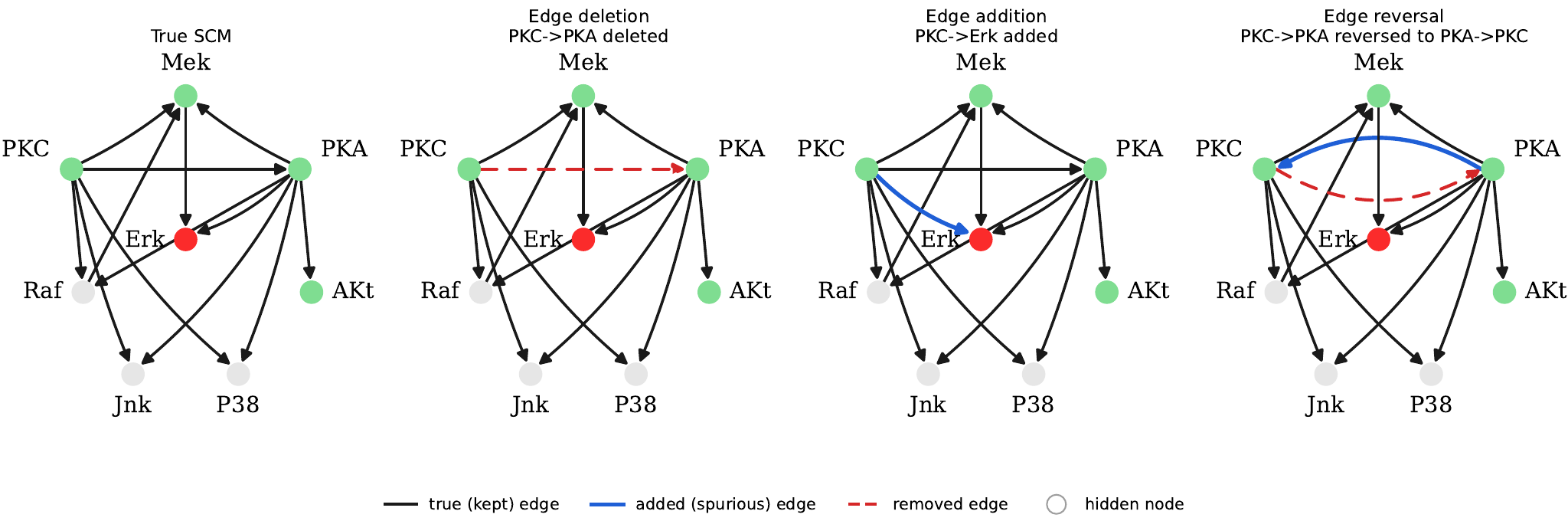}
{\bfseries Ecology}\par\smallskip
\includegraphics[width=\linewidth]{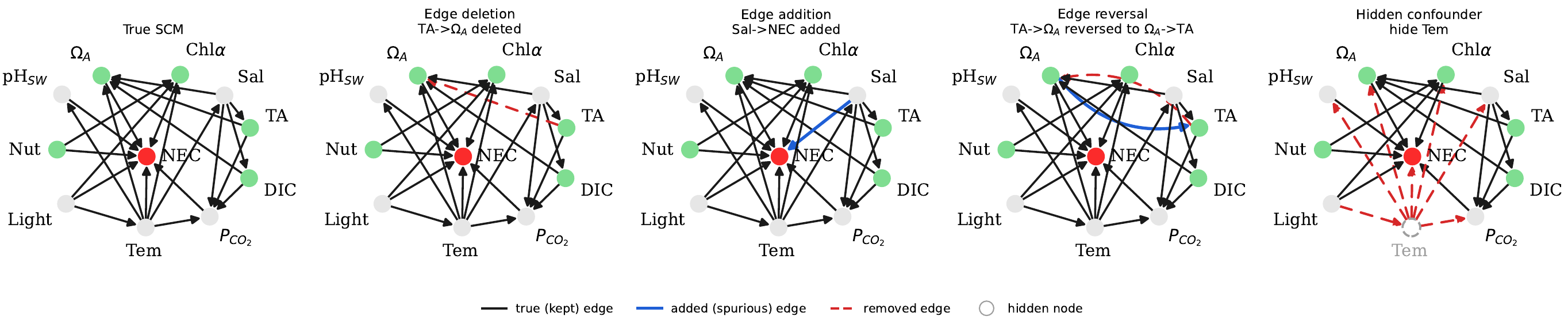}\par\bigskip
{\bfseries Healthcare}\par\smallskip
\includegraphics[width=\linewidth]{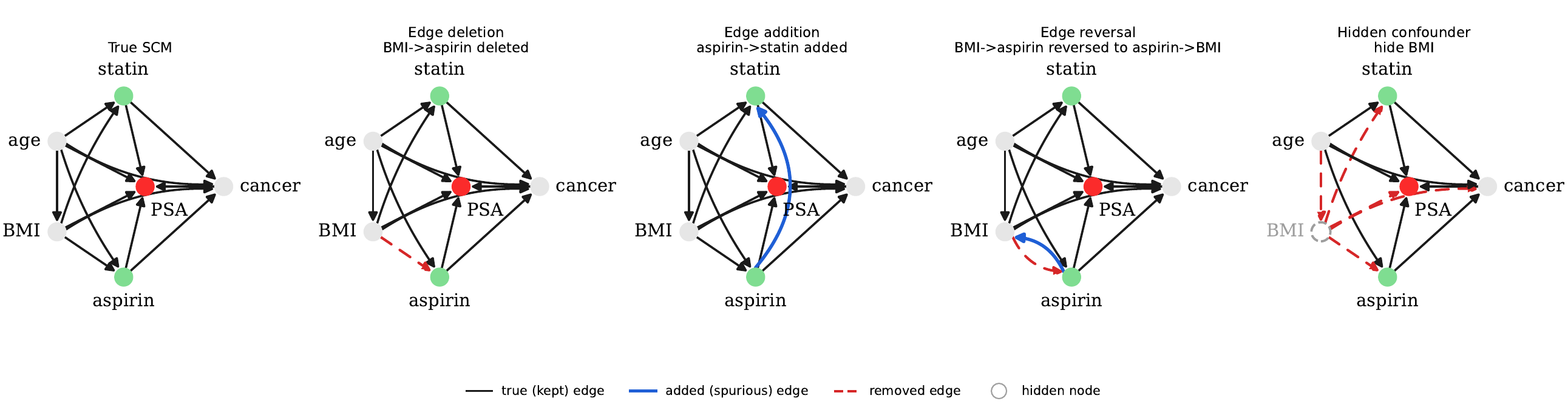}\par\bigskip
{\bfseries Epidemiology}\par\smallskip
\includegraphics[width=.9\linewidth]{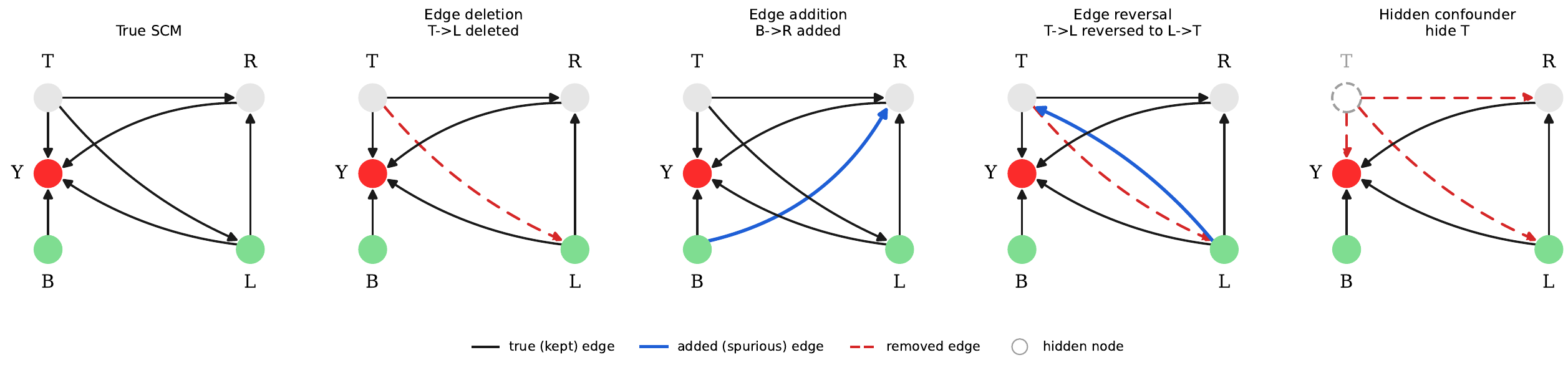}
\caption{Graph-misspecification schematics for Protein-reconstructed and the
domain-inspired benchmark SCMs (Ecology, Healthcare, and Epidemiology). Layout
and visual encoding follow Figure~\ref{fig:misspec-graphs-structured}. Each row
shows the true SCM followed by learner-belief graphs under the misspecification
perturbations shown in the panel titles. Protein-reconstructed uses only the
three edge perturbations: PKC$\to$PKA deletion, PKC$\to$Erk addition, and
PKC$\to$PKA reversal to PKA$\to$PKC. For Ecology, the perturbations are
TA$\to\Omega_A$ deletion, Sal$\to$NEC addition, TA$\to\Omega_A$ reversal to
$\Omega_A\to$TA, and hiding Tem. For Healthcare, the perturbations are
BMI$\to$aspirin deletion, aspirin$\to$statin addition, BMI$\to$aspirin reversal
to aspirin$\to$BMI, and hiding BMI. For Epidemiology, the perturbations are
$T\to L$ deletion, $B\to R$ addition, $T\to L$ reversal to $L\to T$, and hiding
$T$. The Ecology schematic is illustrative and uses representative perturbations
rather than a run-pilot setting. The Healthcare hide-BMI and Epidemiology hide-$T$
panels are the omitted common causes used in the omitted-variable stress test of
Table~\ref{tab:hidden-mainstyle}. Only the learner's supplied graph is perturbed;
the true SCM and $y^\star$ are unchanged.}
\label{fig:misspec-graphs-real}
\end{figure}
\FloatBarrier

\subsubsection{Hidden confounding (omitted-variable stress test).}
Hidden confounding is a distinct assumption violation, and
the graph-perturbation study above does not address it. We therefore added a companion
omitted-variable stress test (Table~\ref{tab:hidden-mainstyle}). Here the learner's graph and
observational view omit a genuine common cause (BMI in Healthcare, $X$ in Chain-hard, and $T$ in
Epidemiology) while the true SCM and reference optimum $y^\star$ are held fixed, so the comparison
isolates the effect of an unobserved confounder rather than changing the task. This omitted-variable
stress test is evaluated over 20 random seeds. We currently include BO as a graph-free control and
CBO/CoCaBO as belief-graph methods.

The effect is method- and dataset-dependent. Omitting the confounder degrades CBO on Chain-hard,
whereas on Healthcare and some Epidemiology settings the omitted-variable runs of CBO and CoCaBO
sometimes match or slightly exceed their full-information counterparts. We do not interpret these
numerical increases as evidence that confounding is beneficial or harmless. Rather, they reflect
finite-budget search effects: a coarser or misspecified observational view can weaken an
overconfident causal prior, change candidate scopes, or accidentally induce a more favorable
acquisition trajectory. The graph-free BO baseline is unchanged by construction, since it does not
use the omitted causal structure. Overall, the omitted-variable stress test reinforces the same
message as the graph-perturbation study: current CBO pipelines are sensitive to the causal
information they are given, and robustness to hidden confounding remains an open evaluation
dimension.

\begin{table}[h]
\centering
\scriptsize
\renewcommand{\arraystretch}{1}
\caption{Hidden-confounding / omitted-variable stress test (mean $\pm$ standard error over 20 seeds; higher is better; best mean per row in bold). The learner's graph and observational view omit a genuine common cause (Healthcare: BMI; Chain-hard: $X$; Epidemiology: $T$); the true SCM and $y^\star$ are unchanged. BO is graph-free and uses the full variable set, so its Full and Hidden rows are identical by construction.}
\label{tab:hidden-mainstyle}
\begin{tabular}{c|c|c|c|ccc}
\toprule
Trial limit & Metric & Dataset & Information & BO & CBO & CoCaBO \\
\midrule
\multirow{12}{*}{100} & \multirow{6}{*}{GAP} & \multirow{2}{*}{Healthcare} & Full information & 0.309$\pm$.057 & \textbf{0.510$\pm$.035} & 0.226$\pm$.118  \\
 &  &  & Hidden variable & 0.309$\pm$.057 & \textbf{0.543$\pm$.012} & 0.251$\pm$.126  \\
\cline{3-7}
 &  & \multirow{2}{*}{Chain-hard} & Full information & 0.780$\pm$.102 & \textbf{0.814$\pm$.128} & 0.761$\pm$.129  \\
 &  &  & Hidden variable & \textbf{0.780$\pm$.102} & 0.695$\pm$.106 & 0.770$\pm$.139  \\
\cline{3-7}
 &  & \multirow{2}{*}{Epidemiology} & Full information & 0.508$\pm$.057 & \textbf{0.563$\pm$.038} & 0.424$\pm$.178  \\
 &  &  & Hidden variable & 0.508$\pm$.057 & \textbf{0.546$\pm$.116} & 0.399$\pm$.094  \\
\cline{2-7}
 & \multirow{6}{*}{PA-GAP} & \multirow{2}{*}{Healthcare} & Full information & 0.122$\pm$.018 & 0.077$\pm$.015 & \textbf{0.125$\pm$.062}  \\
 &  &  & Hidden variable & 0.122$\pm$.018 & 0.080$\pm$.006 & \textbf{0.139$\pm$.070}  \\
\cline{3-7}
 &  & \multirow{2}{*}{Chain-hard} & Full information & \textbf{0.492$\pm$.013} & 0.437$\pm$.005 & 0.486$\pm$.008  \\
 &  &  & Hidden variable & \textbf{0.492$\pm$.013} & 0.420$\pm$.010 & 0.479$\pm$.008  \\
\cline{3-7}
 &  & \multirow{2}{*}{Epidemiology} & Full information & \textbf{0.244$\pm$.023} & 0.185$\pm$.059 & 0.163$\pm$.080  \\
 &  &  & Hidden variable & \textbf{0.244$\pm$.023} & 0.182$\pm$.057 & 0.130$\pm$.091  \\
\midrule
\multirow{12}{*}{50} & \multirow{6}{*}{GAP} & \multirow{2}{*}{Healthcare} & Full information & \textbf{0.438$\pm$.124} & 0.432$\pm$.073 & 0.225$\pm$.128  \\
 &  &  & Hidden variable & 0.438$\pm$.124 & \textbf{0.502$\pm$.020} & 0.378$\pm$.196  \\
\cline{3-7}
 &  & \multirow{2}{*}{Chain-hard} & Full information & 0.680$\pm$.087 & \textbf{0.889$\pm$.016} & 0.815$\pm$.146  \\
 &  &  & Hidden variable & 0.680$\pm$.087 & 0.811$\pm$.054 & \textbf{0.873$\pm$.102}  \\
\cline{3-7}
 &  & \multirow{2}{*}{Epidemiology} & Full information & 0.357$\pm$.053 & 0.355$\pm$.021 & \textbf{0.565$\pm$.076}  \\
 &  &  & Hidden variable & 0.357$\pm$.053 & \textbf{0.590$\pm$.103} & 0.344$\pm$.188  \\
\cline{2-7}
 & \multirow{6}{*}{PA-GAP} & \multirow{2}{*}{Healthcare} & Full information & 0.111$\pm$.017 & 0.069$\pm$.014 & \textbf{0.114$\pm$.057}  \\
 &  &  & Hidden variable & 0.111$\pm$.017 & 0.076$\pm$.005 & \textbf{0.131$\pm$.066}  \\
\cline{3-7}
 &  & \multirow{2}{*}{Chain-hard} & Full information & \textbf{0.484$\pm$.026} & 0.379$\pm$.010 & 0.474$\pm$.016  \\
 &  &  & Hidden variable & \textbf{0.484$\pm$.026} & 0.348$\pm$.018 & 0.468$\pm$.013  \\
\cline{3-7}
 &  & \multirow{2}{*}{Epidemiology} & Full information & \textbf{0.230$\pm$.022} & 0.129$\pm$.061 & 0.141$\pm$.077  \\
 &  &  & Hidden variable & \textbf{0.230$\pm$.022} & 0.162$\pm$.065 & 0.102$\pm$.084  \\
\midrule
\multirow{12}{*}{20} & \multirow{6}{*}{GAP} & \multirow{2}{*}{Healthcare} & Full information & \textbf{0.348$\pm$.120} & 0.268$\pm$.142 & 0.328$\pm$.180  \\
 &  &  & Hidden variable & 0.348$\pm$.120 & \textbf{0.376$\pm$.048} & 0.250$\pm$.159  \\
\cline{3-7}
 &  & \multirow{2}{*}{Chain-hard} & Full information & 0.761$\pm$.124 & 0.718$\pm$.039 & \textbf{0.913$\pm$.047}  \\
 &  &  & Hidden variable & 0.761$\pm$.124 & 0.615$\pm$.090 & \textbf{0.938$\pm$.023}  \\
\cline{3-7}
 &  & \multirow{2}{*}{Epidemiology} & Full information & \textbf{0.694$\pm$.042} & 0.244$\pm$.122 & 0.389$\pm$.071  \\
 &  &  & Hidden variable & \textbf{0.694$\pm$.042} & 0.429$\pm$.181 & 0.156$\pm$.156  \\
\cline{2-7}
 & \multirow{6}{*}{PA-GAP} & \multirow{2}{*}{Healthcare} & Full information & 0.091$\pm$.015 & 0.056$\pm$.011 & \textbf{0.100$\pm$.051}  \\
 &  &  & Hidden variable & 0.091$\pm$.015 & 0.068$\pm$.003 & \textbf{0.112$\pm$.057}  \\
\cline{3-7}
 &  & \multirow{2}{*}{Chain-hard} & Full information & \textbf{0.463$\pm$.062} & 0.228$\pm$.020 & 0.449$\pm$.038  \\
 &  &  & Hidden variable & \textbf{0.463$\pm$.062} & 0.178$\pm$.029 & 0.437$\pm$.033  \\
\cline{3-7}
 &  & \multirow{2}{*}{Epidemiology} & Full information & \textbf{0.228$\pm$.028} & 0.087$\pm$.045 & 0.104$\pm$.069  \\
 &  &  & Hidden variable & \textbf{0.228$\pm$.028} & 0.128$\pm$.080 & 0.068$\pm$.068  \\
\bottomrule
\end{tabular}
\end{table}

\subsection{MCBO \texorpdfstring{$\beta$}{beta} sensitivity}
\label{appendix:mcbo_beta_sensitivity}

The main benchmark keeps the MCBO wrapper setting $\beta=10$, matching the runs reported in Table~\ref{tab:trials_metrics_models}. To assess how much this default matters, we reran MCBO with $\beta\in\{0.5,1,2,5\}$ and compare the resulting GAP and PA-GAP values with the main $\beta=10$ results. The best-performing $\beta$ is not consistent across datasets, budgets, or metrics: smaller values are often better on ToyGraph and Chain-hard, $\beta=5$ is competitive on several PA-GAP rows, and $\beta=10$ remains strongest on Protein-reconstructed. We therefore report the sensitivity sweep as an appendix diagnostic and avoid interpreting the default MCBO performance as an intrinsic property of the method.

\begin{table}[t]
\centering
\scriptsize
\caption{MCBO GAP and PA-GAP vs.\ UCB $\beta$ at $T\in\{100,50,20\}$, $\beta{=}10$ columns are the \texttt{main\_hard} values (the parameter used in the paper); $\beta\in\{0.5,1,2,5\}$ are 20-seed sweep re-runs (100-trial trajectory truncated to $T$). Higher is better; best score per row in bold.}
\label{tab:mcbo-beta-reconciled-full}
\begin{tabular}{ll l ccccc}
\toprule
Metric & $T$ & Dataset & $\beta{=}0.5$ & $\beta{=}1$ & $\beta{=}2$ & $\beta{=}5$ & $\beta{=}10$ \\
\midrule
\multirow{24}{*}{GAP} & \multirow{8}{*}{100} & ToyGraph & \textbf{0.748} & 0.381 & 0.746 & 0.357 & 0.547 \\
 &  & Synthetic & 0.659 & \textbf{0.686} & 0.684 & 0.681 & 0.652 \\
 &  & Synthetic-2 & 0.000 & 0.000 & 0.000 & 0.000 & \textbf{0.000} \\
 &  & Chain-hard & 0.516 & \textbf{0.554} & 0.435 & 0.423 & 0.502 \\
 &  & Ecology & 0.000 & 0.000 & 0.000 & 0.000 & \textbf{0.000} \\
 &  & Protein-reconstructed & 0.419 & 0.337 & 0.109 & 0.286 & \textbf{0.878} \\
 &  & Healthcare & 0.485 & 0.524 & \textbf{0.545} & 0.519 & 0.510 \\
 &  & Epidemiology & 0.000 & 0.000 & 0.000 & 0.000 & \textbf{0.000} \\
\cline{2-8}
 & \multirow{8}{*}{50} & ToyGraph & 0.545 & \textbf{0.576} & 0.492 & 0.253 & 0.513 \\
 &  & Synthetic & 0.605 & 0.606 & \textbf{0.651} & 0.587 & 0.630 \\
 &  & Synthetic-2 & 0.000 & 0.000 & 0.000 & 0.000 & \textbf{0.000} \\
 &  & Chain-hard & 0.392 & \textbf{0.540} & 0.381 & 0.367 & 0.215 \\
 &  & Ecology & 0.000 & 0.000 & 0.000 & 0.000 & \textbf{0.000} \\
 &  & Protein-reconstructed & 0.354 & 0.348 & 0.264 & 0.144 & \textbf{0.752} \\
 &  & Healthcare & 0.303 & 0.292 & 0.275 & \textbf{0.330} & 0.315 \\
 &  & Epidemiology & 0.000 & 0.000 & 0.000 & 0.000 & \textbf{0.000} \\
\cline{2-8}
 & \multirow{8}{*}{20} & ToyGraph & 0.333 & \textbf{0.530} & 0.282 & 0.171 & 0.406 \\
 &  & Synthetic & 0.523 & 0.574 & 0.573 & \textbf{0.578} & 0.558 \\
 &  & Synthetic-2 & 0.000 & 0.000 & 0.000 & 0.000 & \textbf{0.000} \\
 &  & Chain-hard & 0.420 & \textbf{0.432} & 0.422 & 0.207 & 0.360 \\
 &  & Ecology & 0.000 & 0.000 & 0.000 & 0.000 & \textbf{0.000} \\
 &  & Protein-reconstructed & 0.393 & 0.300 & 0.140 & 0.399 & \textbf{0.540} \\
 &  & Healthcare & 0.000 & 0.000 & 0.000 & 0.000 & \textbf{0.000} \\
 &  & Epidemiology & 0.000 & 0.000 & 0.000 & 0.000 & \textbf{0.000} \\
\midrule
\multirow{24}{*}{PA-GAP} & \multirow{8}{*}{100} & ToyGraph & 0.266 & \textbf{0.295} & 0.218 & 0.127 & 0.071 \\
 &  & Synthetic & 0.149 & 0.154 & 0.136 & 0.125 & \textbf{0.162} \\
 &  & Synthetic-2 & 0.000 & 0.000 & 0.000 & 0.000 & \textbf{0.000} \\
 &  & Chain-hard & 0.112 & \textbf{0.174} & 0.142 & 0.067 & 0.111 \\
 &  & Ecology & 0.000 & 0.000 & 0.000 & 0.000 & \textbf{0.000} \\
 &  & Protein-reconstructed & 0.033 & 0.034 & 0.011 & \textbf{0.047} & 0.045 \\
 &  & Healthcare & 0.127 & 0.097 & 0.053 & \textbf{0.136} & 0.100 \\
 &  & Epidemiology & 0.000 & 0.000 & 0.000 & 0.000 & \textbf{0.000} \\
\cline{2-8}
 & \multirow{8}{*}{50} & ToyGraph & 0.186 & \textbf{0.261} & 0.136 & 0.089 & 0.063 \\
 &  & Synthetic & 0.162 & 0.165 & 0.123 & \textbf{0.186} & 0.150 \\
 &  & Synthetic-2 & 0.000 & 0.000 & 0.000 & 0.000 & \textbf{0.000} \\
 &  & Chain-hard & 0.089 & \textbf{0.152} & 0.119 & 0.035 & 0.078 \\
 &  & Ecology & 0.000 & 0.000 & 0.000 & 0.000 & \textbf{0.000} \\
 &  & Protein-reconstructed & 0.033 & 0.034 & 0.011 & \textbf{0.046} & 0.045 \\
 &  & Healthcare & 0.016 & 0.032 & 0.012 & \textbf{0.034} & \textbf{0.034} \\
 &  & Epidemiology & 0.000 & 0.000 & 0.000 & 0.000 & \textbf{0.000} \\
\cline{2-8}
 & \multirow{8}{*}{20} & ToyGraph & \textbf{0.175} & 0.171 & 0.093 & 0.028 & 0.040 \\
 &  & Synthetic & 0.137 & 0.137 & 0.116 & \textbf{0.141} & 0.116 \\
 &  & Synthetic-2 & 0.000 & 0.000 & 0.000 & 0.000 & \textbf{0.000} \\
 &  & Chain-hard & 0.049 & \textbf{0.116} & 0.094 & 0.016 & 0.034 \\
 &  & Ecology & 0.000 & 0.000 & 0.000 & 0.000 & \textbf{0.000} \\
 &  & Protein-reconstructed & 0.034 & 0.034 & 0.011 & 0.043 & \textbf{0.047} \\
 &  & Healthcare & 0.000 & 0.000 & 0.000 & 0.000 & \textbf{0.000} \\
 &  & Epidemiology & 0.000 & 0.000 & 0.000 & 0.000 & \textbf{0.000} \\
\bottomrule
\end{tabular}
\end{table}
\FloatBarrier

\subsection{Ranking Statistical Testing}
\label{appendix:ranking_statistical_testing}

The main text reports no pooled cross-method rank figures, because the compared pipelines do not all solve the same optimization problem (Table~\ref{tab:formulation-labels}). The tables in this appendix are the benchmark's only cross-method aggregation: they are retained as a \emph{descriptive} reliability record, and we assess whether such rank summaries would support statistically separable method differences. For each fixed budget, methods are ranked within each dataset--metric setting, where
rank 1 denotes the best mean score in that setting. We then summarize the average rank across
settings and report bootstrap confidence intervals over settings. We also report a Friedman
omnibus test and the Nemenyi critical difference for post-hoc rank separation. Because GAP and
PA-GAP are computed from the same trajectories, we additionally report a conservative per-dataset
Friedman test in which the two metric ranks are collapsed within each dataset.
An important caveat applies to all of these tests: the pooled settings are not mutually
independent. GAP and PA-GAP share the same underlying trajectories, and the three budgets
$T\in\{100,50,20\}$ are nested truncations of the same runs rather than independent replications.
The reported $p$-values and critical differences should therefore be read as descriptive
reliability checks under these dependencies, not as exact inference over independent samples; the
conservative per-dataset test partially mitigates the metric-sharing dependence but not the budget
nesting.

\subsubsection{Hard-intervention}
\label{appendix:ranking_statistical_testing_hard}

\begin{table}[t]
\centering
\footnotesize
\setlength{\tabcolsep}{4pt}
\caption{Ranking reliability on the hard-intervention benchmark by budget.
Ranks (1=best) are computed within each dataset--metric setting
($8$ datasets $\times$ $2$ metrics, $N{=}16$ per budget). Entries show
average rank $\pm$ s.d.\ and 95\% bootstrap CI; lower is better. Bottom rows
report the Friedman test, Nemenyi critical difference versus rank spread, and
a conservative per-dataset Friedman test.}
\label{tab:avg-rank-hard}
\begin{tabular}{lcccccc}
\toprule
 & \multicolumn{2}{c}{$T{=}100$} & \multicolumn{2}{c}{$T{=}50$} & \multicolumn{2}{c}{$T{=}20$} \\
\cmidrule(lr){2-3}\cmidrule(lr){4-5}\cmidrule(lr){6-7}
Method & Rank $\pm$ s.d. & 95\% CI & Rank $\pm$ s.d. & 95\% CI & Rank $\pm$ s.d. & 95\% CI \\
\midrule
CoCaBO & \textbf{3.38 $\pm$ 2.75} & [2.12, 4.69] & \textbf{3.44 $\pm$ 2.68} & [2.22, 4.75] & 3.88 $\pm$ 2.82 & [2.56, 5.25] \\
cCBO & 3.91 $\pm$ 2.23 & [2.88, 5.03] & 3.69 $\pm$ 2.32 & [2.66, 4.84] & 3.88 $\pm$ 2.24 & [2.84, 4.97] \\
CBO & 3.50 $\pm$ 1.59 & [2.81, 4.31] & 3.94 $\pm$ 1.18 & [3.44, 4.50] & 4.31 $\pm$ 1.35 & [3.62, 4.94] \\
BO & 4.12 $\pm$ 2.84 & [2.78, 5.53] & 3.94 $\pm$ 2.69 & [2.66, 5.22] & \textbf{3.81 $\pm$ 2.48} & [2.62, 5.00] \\
CEO & 4.78 $\pm$ 1.70 & [3.94, 5.56] & 4.88 $\pm$ 1.96 & [3.94, 5.81] & 4.25 $\pm$ 1.98 & [3.31, 5.19] \\
DCBO & 5.03 $\pm$ 2.22 & [3.94, 6.06] & 4.78 $\pm$ 2.27 & [3.75, 5.84] & 4.12 $\pm$ 2.18 & [3.09, 5.22] \\
HCBO & 5.38 $\pm$ 1.89 & [4.44, 6.25] & 5.44 $\pm$ 1.86 & [4.50, 6.25] & 5.84 $\pm$ 2.19 & [4.75, 6.84] \\
MCBO & 5.91 $\pm$ 1.89 & [4.97, 6.72] & 5.91 $\pm$ 2.19 & [4.81, 6.88] & 5.91 $\pm$ 2.01 & [4.91, 6.81] \\
\midrule
Friedman $p$ & \multicolumn{2}{c}{$0.0272$} & \multicolumn{2}{c}{$0.0384$} & \multicolumn{2}{c}{$0.0462$} \\
Nemenyi CD / spread & \multicolumn{2}{c}{2.62 / 2.53} & \multicolumn{2}{c}{2.62 / 2.47} & \multicolumn{2}{c}{2.62 / 2.09} \\
Per-dataset $p$ ($N{=}8$) & \multicolumn{2}{c}{$0.278$} & \multicolumn{2}{c}{$0.342$} & \multicolumn{2}{c}{$0.406$} \\
\bottomrule
\end{tabular}
\end{table}
The hard-intervention results show that average-rank orderings vary with budget. CoCaBO has the
lowest average rank at $T=100$ and $T=50$, while BO has the lowest average rank at $T=20$.
However, these rank differences should be interpreted cautiously. Although the Friedman omnibus
test detects rank heterogeneity at each budget, the rank spread is always smaller than the Nemenyi
critical difference, so no pair of methods is statistically separable by the post-hoc test. The
more conservative per-dataset Friedman test also does not reject at any budget. Thus, the ranking
analysis supports the qualitative conclusion that performance is heterogeneous, but it does not
support declaring a single hard-intervention method as uniformly superior.
\subsubsection{Soft-intervention}
\label{appendix:ranking_statistical_testing_soft}

\begin{table}[t]
\centering
\small
\setlength{\tabcolsep}{4pt}
\caption{Ranking reliability on the soft-intervention benchmark by budget.
Ranks (1=best) are computed within each dataset--metric setting
($5$ datasets $\times$ $2$ metrics, $N{=}10$ per budget). Entries show
average rank $\pm$ s.d.\ and 95\% bootstrap CI; lower is better. Bottom rows
report the Friedman test, Nemenyi critical difference versus rank spread, and
a conservative per-dataset Friedman test.}
\label{tab:avg-rank-soft}
\begin{tabular}{lcccccc}
\toprule
 & \multicolumn{2}{c}{$T{=}100$} & \multicolumn{2}{c}{$T{=}50$} & \multicolumn{2}{c}{$T{=}20$} \\
\cmidrule(lr){2-3}\cmidrule(lr){4-5}\cmidrule(lr){6-7}
Method & Rank $\pm$ s.d. & 95\% CI & Rank $\pm$ s.d. & 95\% CI & Rank $\pm$ s.d. & 95\% CI \\
\midrule
ACBO & \textbf{1.60 $\pm$ 0.84} & [1.10, 2.10] & \textbf{1.60 $\pm$ 0.70} & [1.20, 2.00] & \textbf{1.60 $\pm$ 0.66} & [1.20, 2.00] \\
BO & 1.80 $\pm$ 0.79 & [1.40, 2.30] & 1.90 $\pm$ 0.74 & [1.50, 2.30] & 2.00 $\pm$ 0.78 & [1.55, 2.45] \\
MCBO & 2.60 $\pm$ 0.52 & [2.30, 2.90] & 2.50 $\pm$ 0.85 & [2.00, 3.00] & 2.40 $\pm$ 0.77 & [1.90, 2.80] \\
\midrule
Friedman $p$ & \multicolumn{2}{c}{$0.0608$} & \multicolumn{2}{c}{$0.1225$} & \multicolumn{2}{c}{$0.1690$} \\
Nemenyi CD / spread & \multicolumn{2}{c}{1.05 / 1.00} & \multicolumn{2}{c}{1.05 / 0.90} & \multicolumn{2}{c}{1.05 / 0.80} \\
Per-dataset $p$ ($N{=}5$) & \multicolumn{2}{c}{$0.211$} & \multicolumn{2}{c}{$0.311$} & \multicolumn{2}{c}{$0.390$} \\
\bottomrule
\end{tabular}
\end{table}

The soft-intervention results show a consistent descriptive ordering: ACBO has the lowest average
rank at all three budgets, followed by BO and then MCBO. Nevertheless, the statistical tests do not
support a separable winner. The Friedman omnibus test does not reject at any budget, the rank
spread remains below the Nemenyi critical difference, and the conservative per-dataset test also
does not reject. We therefore treat the soft-intervention ranking as descriptive evidence that ACBO
is often strong in this benchmark, rather than as a statistically conclusive superiority claim.

\subsection{Average Reward Visualization}
\label{appendix:average_reward_visualization}

GAP, PA-GAP, and best-so-far curves summarize incumbent quality, but they do not show the quality of
all evaluations made along the way. We therefore add average-reward trajectories as a
complementary diagnostic. For each run, rewards are normalized so that the reference optimum is $1$ in the task's natural optimization direction. This makes average-reward
curves comparable across datasets while preserving their interpretation as the mean quality of the
sequence of queried interventions. These plots are not used as the primary ranking criterion,
because exploratory acquisitions can intentionally query uncertain points with poor immediate reward,
but they help verify that the benchmark conclusions do not depend only on final incumbent values.

Figures~\ref{fig:average_reward_hard} and~\ref{fig:average_reward_soft} should be read together with
the best-so-far trajectories and scalar metrics in the main text. When a method has strong
best-so-far performance but weaker average reward, the optimizer is finding good incumbents while
spending part of its budget on exploratory or poorly calibrated trials. Conversely, strong average
reward with weaker final GAP indicates reliable early sampling without necessarily discovering the
best final intervention. This reinforces the paper's main evaluation message: CBO methods should be
compared using multiple trajectory summaries rather than a single scalar metric.

\begin{figure}[h]
    \centering
    \includegraphics[width=\linewidth]{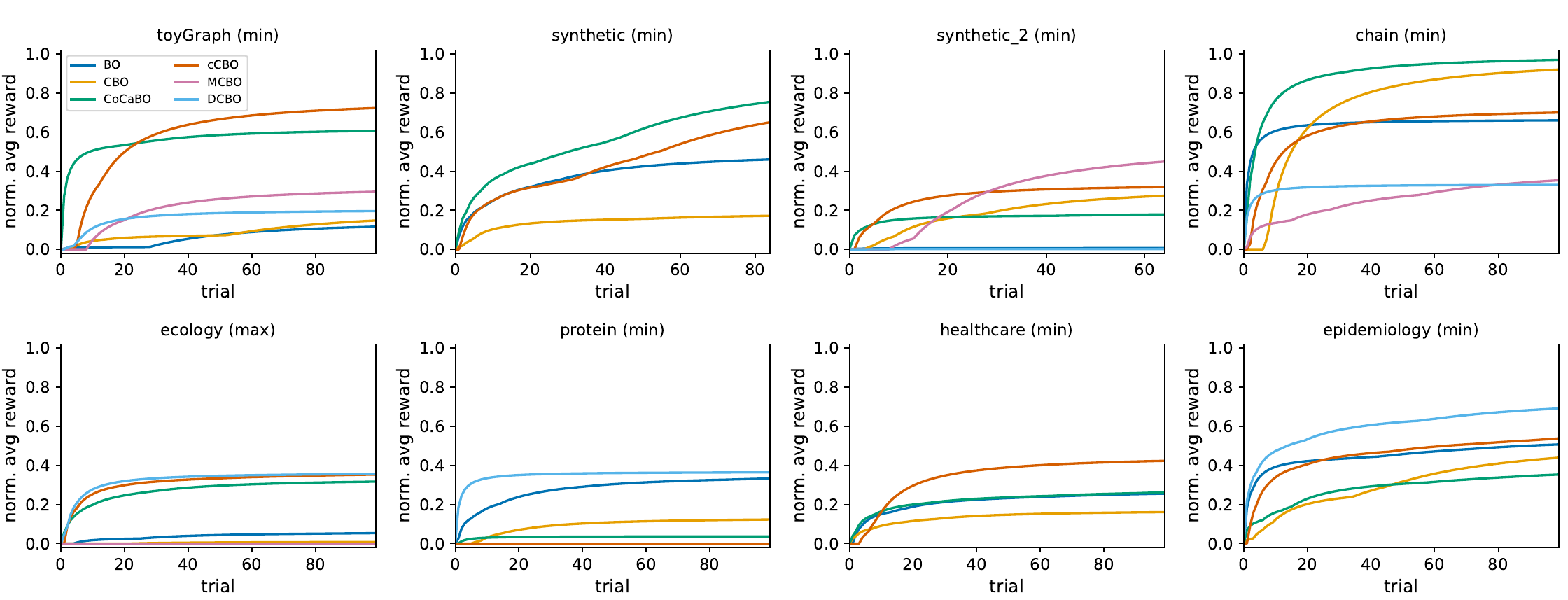}
    \caption{Average-reward trajectories for the hard-intervention benchmark.}
    \label{fig:average_reward_hard}
\end{figure}

\begin{figure}[h]
    \centering
    \includegraphics[width=\linewidth]{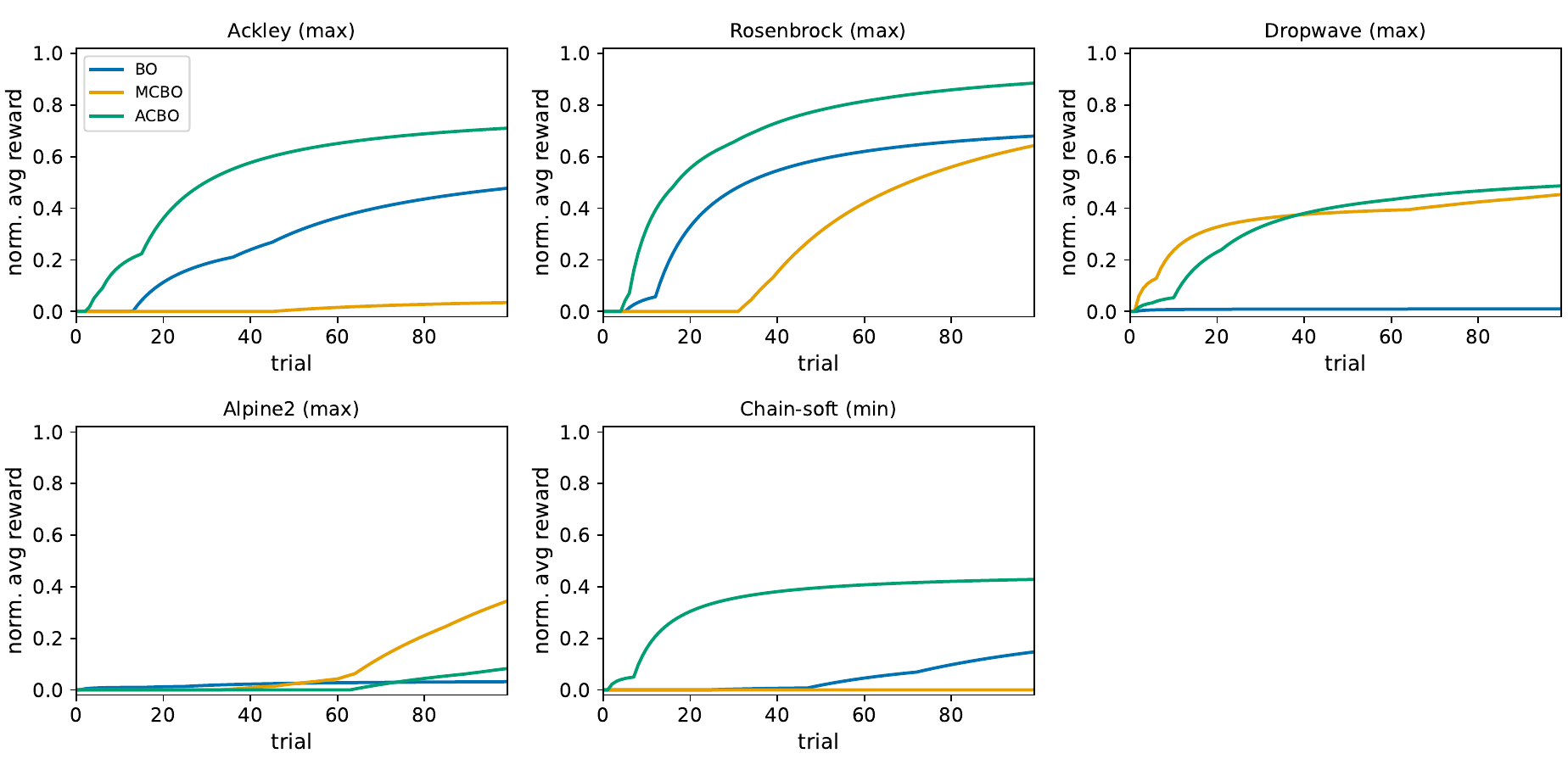}
    \caption{Average-reward trajectories for the soft-intervention/function-network benchmark.}
    \label{fig:average_reward_soft}
\end{figure}
\end{document}